\documentclass{article}

\pdfoutput=1 

     \usepackage[final]{neurips_2021}

\title{The Causal-Neural Connection:\\ Expressiveness, Learnability, and Inference}

\usepackage[square,numbers]{natbib}

\ifdefined\preamble\else\def\preamble{}

\usepackage{tikz}
\usetikzlibrary{shapes,decorations,arrows,calc,arrows.meta,fit,positioning, patterns}
\tikzset{
    -Latex,auto,node distance =1 cm and 1 cm,semithick,
    state/.style ={ellipse, draw, minimum width = 0.7 cm},
    point/.style = {circle, draw, inner sep=0.04cm,fill,node contents={}},
    bidirected/.style={Latex-Latex,dashed},
    el/.style = {inner sep=2pt, align=left, sloped}
}

\usepackage{amsmath,amsthm,amsfonts,amssymb}

\usepackage{thmtools,thm-restate}
\declaretheorem{theorem}
\declaretheorem{lemma}
\declaretheorem{fact}

\theoremstyle{definition}
\declaretheorem{definition}
\declaretheorem{example}

\newcommand{\M}{\mathcal{M}}

\newcommand{\Ll}{\mathcal{L}}

\newcommand{\Layer}{{Layer}}

\DeclareSymbolFont{symbolsC}{U}{txsyc}{m}{n}
\DeclareMathSymbol{\boxright}{\mathrel}{symbolsC}{128}

\def\*#1{\mathbf{#1}}

\newcommand{\pai}[1]{\*{pa}_{#1}}
\newcommand{\Pai}[1]{\*{Pa}_{#1}}

\newcommand{\ui}[1]{\*{u}_{#1}}
\newcommand{\Ui}[1]{\*{U}_{#1}}

\newcommand\numberthis{\addtocounter{equation}{1}\tag{\theequation}}

\theoremstyle{definition}

\DeclareMathOperator{\ctm}{\textsf{CM}}

\DeclareMathOperator{\bin}{bin}

\DeclareMathOperator{\unif}{Unif}

\DeclareMathOperator{\mlp}{MLP}
\DeclareMathOperator{\scand}{\textsc{AND}}
\DeclareMathOperator{\scor}{\textsc{OR}}
\DeclareMathOperator{\scnot}{\textsc{NOT}}
\DeclareMathOperator{\kl}{\textsc{KL}}
\DeclareMathOperator{\se}{SE}

\newcommand{\ncm}{\text{NCM}}
\newcommand{\ncms}{\text{NCMs}}
\newcommand{\genncm}{\text{NCM}}
\newcommand{\werm}{\text{WERM}}

\newcommand{\xdasharrow}[2][->]{
\tikz[baseline=-\the\dimexpr\fontdimen22\textfont2\relax]{
\node[anchor=south,font=\scriptsize, inner ysep=1.5pt,outer xsep=8pt](x){#2};
\draw[shorten <=3.4pt,shorten >=3.4pt,dashed,#1](x.south west)--(x.south east);
}
}

\def\ddefloop#1{\ifx\ddefloop#1\else\ddef{#1}\expandafter\ddefloop\fi}
\def\ddef#1{\expandafter\def\csname bb#1\endcsname{\ensuremath{\mathbb{#1}}}}
\ddefloop ABCDEFGHIJKLMNOPQRSTUVWXYZ\ddefloop
\def\ddef#1{\expandafter\def\csname c#1\endcsname{\ensuremath{\mathcal{#1}}}}
\ddefloop ABCDEFGHIJKLMNOPQRSTUVWXYZ\ddefloop
\def\ddef#1{\expandafter\def\csname v#1\endcsname{\ensuremath{\boldsymbol{#1}}}}
\ddefloop ABCDEFGHIJKLMNOPQRSTUVWXYZabcdefghijklmnopqrstuvwxyz\ddefloop
\def\ddef#1{\expandafter\def\csname v#1\endcsname{\ensuremath{\boldsymbol{\csname #1\endcsname}}}}
\ddefloop {alpha}{beta}{gamma}{delta}{epsilon}{varepsilon}{zeta}{eta}{theta}{vartheta}{iota}{kappa}{lambda}{mu}{nu}{xi}{pi}{varpi}{rho}{varrho}{sigma}{varsigma}{tau}{upsilon}{phi}{varphi}{chi}{psi}{omega}{Gamma}{Delta}{Theta}{Lambda}{Xi}{Pi}{Sigma}{varSigma}{Upsilon}{Phi}{Psi}{Omega}{ell}\ddefloop

\makeatletter
\newcommand*{\indep}{%
  \mathbin{%
    \mathpalette{\@indep}{}%
  }%
}
\newcommand*{\nindep}{%
  \mathbin{
    \mathpalette{\@indep}{\not}
  }%
}
\newcommand*{\@indep}[2]{%
  \sbox0{$#1\perp\m@th$}
  \sbox2{$#1=$}
  \sbox4{$#1\vcenter{}$}
  \rlap{\copy0}
  \dimen@=\dimexpr\ht2-\ht4-.2pt\relax
  \kern\dimen@
  {#2}%
  \kern\dimen@
  \copy0 
} 
\makeatother
\fi

\usepackage[utf8]{inputenc} 
\usepackage[T1]{fontenc}    
\usepackage{url}            
\usepackage{booktabs}       
\usepackage{amsfonts}       
\usepackage{nicefrac}       
\usepackage{microtype}      
\usepackage{lipsum}
\usepackage{amsmath}
\usepackage{amsthm}
\usepackage{bbm}
\usepackage{enumitem}
\usepackage{optidef}
\usepackage{subcaption}
\usepackage{bm}
\usepackage{upgreek}
\usepackage{wrapfig}
\usepackage{enumitem,kantlipsum}
\usepackage{microtype}
\usepackage{setspace}

\tikzset{
    -Latex,auto,node distance =1 cm and 1 cm,semithick,
    state/.style ={ellipse, draw, minimum width = 0.7 cm},
    point/.style = {circle, draw, inner sep=0.04cm,fill,node contents={}},
    bidirected/.style={Latex-Latex,dashed},
    el/.style = {inner sep=2pt, align=left, sloped}
}
\tikzstyle{mybox} = [draw=gray, fill=gray!20, very thick,
rectangle, rounded corners, inner xsep=3pt, inner ysep=7pt]

\usepackage{tikz}
\usetikzlibrary{arrows.meta}

\usepackage[linesnumbered, ruled, vlined]{algorithm2e}
\usepackage{etoolbox}
\makeatletter
\patchcmd{\@algocf@start}
  {-1.5em}
  {0pt}
  {}{}
\makeatother

\usepackage[english]{babel}

\author{%
  Kevin Xia \\
  \textls[-20]{CausalAI Lab } \\
  \textls[-20]{Columbia University} \\
  \textls[-55]{\small{\texttt{kmx2000@columbia.edu}}} \\
  \And
  Kai-Zhan Lee \\
  \textls[-20]{Bloomberg L.P.} \\
  \textls[-20]{Columbia University} \\
    \textls[-55]{\small{\texttt{kl2792@columbia.edu}}}
  \And
  Yoshua Bengio \\
  \textls[-20]{MILA} \\
  \textls[-30]{\small{Universit\'e de Montr\'eal}}\\
   \textls[-55]{\small{\texttt{yoshua.bengio@mila.quebec}}}
  \And
  Elias Bareinboim \\
  \textls[-20]{CausalAI Lab } \\
  \textls[-20]{Columbia University} \\
  \textls[-55]{\small{\texttt{eb@cs.columbia.edu}}}
}

\begin{document}

\maketitle

\begin{abstract}
One of the central elements of any causal inference is an object called structural causal model (SCM), which represents a collection of mechanisms and exogenous sources of random variation of the system under investigation (Pearl, 2000). An important property of many kinds of neural networks is \textit{universal approximability}: the ability to approximate any function to arbitrary precision. Given this property, one may be tempted to surmise that a collection of neural nets is capable of learning any SCM by training on data generated by that SCM. In this paper, we show this is not the case by disentangling the notions of expressivity and learnability. Specifically, we show that the causal hierarchy theorem (Thm.\ 1, Bareinboim et al., 2020), which describes the limits of what can be learned from data, still holds for neural models. For instance, an arbitrarily complex and expressive neural net is unable to predict the effects of interventions given observational data alone. Given this result, we introduce a special type of SCM called a neural causal model (\ncm), and formalize a new type of inductive bias to encode structural constraints necessary for performing causal inferences. Building on this new class of models, we focus on solving two canonical tasks found in the literature known as causal  identification and estimation. Leveraging the neural toolbox, we develop an algorithm that is both sufficient and necessary to determine whether a causal effect can be learned from data (i.e., causal identifiability); it then estimates the effect whenever identifiability holds (causal estimation). Simulations corroborate the proposed approach.
\end{abstract}

\vspace{-0.19in}
\section{Introduction} \label{sec:intro}
\vspace{-0.12in}

One of the most celebrated and relied upon results in the science of intelligence is the universality of neural models. More formally, universality says that neural models can approximate any function (e.g., boolean, classification boundaries, continuous valued) with arbitrary precision given enough capacity in terms of the depth and breadth of the network \citep{Cybenko89,HORNIK1991251,LESHNO1993861,NIPS2017_32cbf687}. This result, combined with the observation that most tasks can be abstracted away and modeled as input/output -- i.e., as functions -- leads to the strongly held belief that under the right conditions, neural networks can solve the most challenging and interesting tasks in AI.
This belief is not without merits, and is corroborated by ample evidence of practical successes, including in compelling tasks in computer vision \citep{NIPS2012_c399862d}, speech recognition \citep{pmlr-v32-graves14}, and game playing \citep{volodymr:etal13}.
Given that the universality of neural nets is such a compelling proposition, we investigate this belief in the context of causal reasoning.

To start understanding the causal-neural connection -- i.e., the non-trivial and somewhat intricate relationship between these modes of reasoning -- 
two standard objects in causal analysis will be instrumental. 
First, we evoke a class of generative models known as the  \textit{Structural Causal Model} (SCM, for short) \citep[Ch.~7]{pearl:2k}. In words, an SCM $\cM^*$ is a representation of a system that includes a collection of mechanisms and a probability distribution over the exogenous conditions
  (to be formally defined later on). 
Second, any fully specified SCM $\cM^*$ induces a collection of distributions known as the \textit{Pearl Causal Hierarchy} (PCH) {\citep[Def.~9]{bareinboim:etal20}}.
The importance of the PCH is that it formally delimits distinct cognitive capabilities (also known as layers; not to be confused with neural nets layers) that can be associated with the human activities of ``seeing'' (layer 1), ``doing'' (2), and ``imagining'' (3)  \citep[Ch.~1]{pearl:mackenzie2018}.
\footnotemark[1]
Each of these layers can be expressed as a distinct formal language and represents queries that can help to classify different types of inferences \citep[Def.~8]{bareinboim:etal20}. Together, these layers form a strict containment hierarchy  \citep[Thm.~1]{bareinboim:etal20}. We illustrate these notions in Fig.~\ref{fig:pch}(a) (left side), where  SCM $\cM^*$ induces layers $L^*_1, L^*_2, L^*_3$ of the PCH.

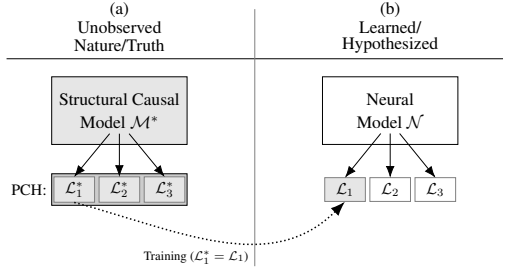
\begin{wrapfigure}{r}{0.5\textwidth}
\vspace{-0.3in}
  \begin{center}
    \begin{tikzpicture}[-, scale=0.6, every node/.append style={transform shape}]
        \draw [line width=0.05mm] (-5.5,1.9) -- (5.5,1.9);
        \draw [line width=0.05mm, draw=gray] (0,3) -- (0,-2.7);
        
        \node [align=center, font=\fontsize{11}{0}\selectfont] (atext1) at (-3, 3) {(a)};
        \node [align=center, font=\fontsize{11}{0}\selectfont] (atext2) at (-3, 2.6) {Unobserved};
        \node [align=center, font=\fontsize{11}{0}\selectfont] (atext3) at (-3, 2.2) {Nature/Truth};
        \node [align=center, font=\fontsize{11}{0}\selectfont] (btext1) at (3, 3) {(b)};
        \node [align=center, font=\fontsize{11}{0}\selectfont] (btext2) at (3, 2.6) {Learned/};
        \node [align=center, font=\fontsize{11}{0}\selectfont] (btext3) at (3, 2.2) {Hypothesized};
        
        \filldraw [fill=gray!40,line width=0.01mm] (-4.5, -0.65) rectangle (-1.5, -1.35);
        
        \node (pch) at (-5, -1) {PCH:};
    
        \filldraw [fill=gray!20, line width=0.01mm] (-4.5, 1.5) rectangle (-1.5, 0);
        \node [align=center, font=\fontsize{11}{0}\selectfont] (Pu1) at (-3,1) {Structural Causal};
        \node [align=center, font=\fontsize{11}{0}\selectfont] (Pu) at (-3,0.5) {Model $\cM^*$};
    
       	\node [fill=gray!20, draw=gray, line width=0.08mm, align=center, minimum width=0.9cm] (L1Pv) at (-4,-1) {$\Ll_1^*$};
       	\node [fill=gray!20, draw=gray, line width=0.08mm, align=center, minimum width=0.9cm] (L2Pv) at (-3,-1) {$\Ll_2^*$};
       	\node [fill=gray!20, draw=gray, line width=0.08mm, align=center, minimum width=0.9cm] (L3Pv) at (-2,-1) {$\Ll_3^*$};
       	
       	\path [-Latex, line width=0.1mm] (Pu) edge (L1Pv.north);
       	\path [-Latex, line width=0.1mm] (Pu) edge (L2Pv.north);
       	\path [-Latex, line width=0.1mm] (Pu) edge (L3Pv.north);

       	\filldraw [fill=gray!0,line width=0.01mm] (4.5, 1.5) rectangle (1.5, 0);
        \node [align=center, font=\fontsize{11}{0}\selectfont] (nPu1) at (3,1) {Neural};
        \node [align=center, font=\fontsize{11}{0}\selectfont] (nPu) at (3,0.5) {Model $\cN$};
    
       	\node [fill=gray!20, draw=gray, line width=0.08mm, align=center, minimum width=0.9cm] (nL1Pv) at (2,-1) {$\Ll_1$};
       	\node [draw=gray, line width=0.08mm, align=center, minimum width=0.9cm] (nL2Pv) at (3,-1) {$\Ll_2$};
       	\node [draw=gray, line width=0.08mm, align=center, minimum width=0.9cm] (nL3Pv) at (4,-1) {$\Ll_3$};
    
       	\path [-Latex, line width=0.1mm] (nPu) edge (nL1Pv.north);
       	\path [-Latex, line width=0.1mm] (nPu) edge (nL2Pv.north);
       	\path [-Latex, line width=0.1mm] (nPu) edge (nL3Pv.north);
       	
       	\path [-Latex, line width=0.2mm] (L1Pv.south) edge[densely dotted, out=-15, in=-135] (nL1Pv.south);
       	\node [align=center, font=\fontsize{8}{0}\selectfont] (trainingtext) at (-1.3, -2.5) {Training ($\cL_1^* = \cL_1$)};
        
    \end{tikzpicture}
    \caption{The l.h.s. contains the unobserved true SCM $\cM^*$ that induces the three layers of the PCH. The r.h.s. contains an NCM that is trained to match in layer 1. The matching shading indicates that the two models agree w.r.t. $L_1$ while not necessarily agreeing w.r.t. layers 2 and 3.}
    \label{fig:pch}
    \end{center}
    \vspace{-0.3in}
\end{wrapfigure}

Even though each possible statement within these capabilities has well-defined semantics given the true SCM $\cM^*$ \citep[Ch.~7]{pearl:2k}, a challenging inferential task arises when one wishes to recover part of the PCH when $\cM^*$ is only partially observed. This situation is typical in the real world aside from some special settings in physics and chemistry where the laws of nature are understood with high precision. 

For concreteness, consider the setting where one needs to make a statement about the effect of a new intervention (i.e., about layer 2),\footnotetext[1]{This 
structure is named after Judea Pearl and is a central topic in his \textit{Book of Why (BoW)},  where it is also called the ``Ladder of Causation'' \citep{pearl:mackenzie2018}. For a more technical discussion on the PCH, we refer readers to \citep{bareinboim:etal20}.  } \stepcounter{footnote}
but only has observational data from layer 1, which is passively collected.\footnote{\label{footnote:id}The full inferential challenge is, in practice, more general since an agent may be able to perform interventions and obtain samples from a subset of the PCH's layers, while its goal is to make inferences about some other parts of the layers  \citep{bareinboim2016causal,lee:etal19,bareinboim:etal20}. This situation is not uncommon in RL settings \cite{sutton:barto98,forney:etal17,lee:bar18a,lee2020omp}. Still, for the sake of space and concreteness, we will focus on two canonical and more basic tasks found in the literature. }
Going back to the causal-neural connection, one could try to learn a neural model $\cN$ using the observational dataset (layer 1) generated by the true SCM $\cM^*$, as illustrated in Fig.~\ref{fig:pch}(b). Naturally, a basic consistency requirement  is that $\cN$ should be capable of generating the same distributions as $\cM^*$; in this case, their layer 1 predictions should match (i.e., $L_1 = L^*_1$). Given the universality of neural models, it is not hard to believe that these constraints can be satisfied in the large sample limit. 
The question arises of whether the learned model $\cN$ can act as a proxy, having the capability of predicting the effect of  interventions that matches the $L_2$ distribution generated by the true (unobserved) SCM $\cM^*$.
\footnote{We defer a more formal discussion on how neural models could be used to assess the effect of interventions to Sec.~\ref{section:ncms}. Still, this is neither attainable in all universal neural architectures nor trivially implementable.} 
 The answer to this question cannot be ascertained in general, as will become evident later on (Corol.~1). The intuitive reason behind this result is that 
there are multiple neural models that are equally consistent w.r.t. the $L_1$ distribution of $\cM^*$ but generate different $L_2$-distributions.
\footnote{Pearl shared a similar observation in the \textit{BoW} \cite[p.~32]{pearl:mackenzie2018}: ``Without the causal model, we could not go from rung (layer) one to rung (layer) two. This is why deep-learning systems (as long as they use only rung-one data and do not have a causal model) will never be able to answer questions about interventions (...)''. 
\label{footnote:pearl} }
Even though $\cN$ may be expressive enough to fully represent $\cM^*$ (as discussed later on), generating one particular parametrization of $\cN$ consistent with $L_1$ is insufficient to provide any guarantee regarding higher-layer inferences, i.e., about predicting the effects of interventions ($L_2$) or counterfactuals ($L_3$).

The discussion above entails two tasks that have been acknowledged in the literature, namely, causal effect identification and estimation. The first -- causal identification  -- has been extensively studied, and general solutions have been developed, such as Pearl's celebrated do-calculus \citep{pearl:95a}. Given the impossibility described above, the ingredient shared across current non-neural  solutions  is to represent assumptions about the unknown $\cM^*$ in the form of causal diagrams \citep{pearl:2k,spirtes:etal00,bareinboim2016causal} or their equivalence classes \citep{jaber2018causal,perkovic:etal18,jaber2019markovcompleteness,Zhang2008}. The task is then to decide whether there is a unique solution for the causal query  based on such assumptions. There are no neural methods today focused on solving this task. 

\vspace{-0.05in}
The second task -- causal estimation -- is triggered when effects are determined to be identifiable by the first task. 
Whenever identifiability is obtained through the backdoor criterion/conditional ignorability  \citep[Sec.~3.3.1]{pearl:2k}, deep learning techniques can be leveraged to estimate such effects with impressive practical performance \citep{shalit2017estimating, louizos2017causal,  NIPS2017_b2eeb736, johansson2018learning, NEURIPS2018_a50abba8, yoon2018ganite,  kallus2018deepmatch, shi2019adapting,  du2020adversarial, Guo_2020, kennedy2021semiparametric, johansson2021generalization}.
For effects that are identifiable through causal functionals that are not necessarily of the backdoor-form (e.g., frontdoor, napkin), other optimization/statistical techniques can be employed that enjoy properties such as double robustness and debiasedness \cite{jung2020estimating,jung2020werm,jung2021dml}. 
Each of these approaches optimizes a particular estimand corresponding to one specific target interventional distribution.

Despite all the great progress achieved so far, it is still largely unknown how to perform the tasks of causal identification and estimation in arbitrary settings using neural networks as a generative model, acting as a proxy for the true SCM $\cM^*$. 
It is our goal here to develop a general causal-neural framework that has the potential to scale to real-world, high-dimensional domains while preserving the validity of its inferences, as in traditional symbolic approaches. In the same way that the causal diagram encodes the assumptions necessary for the do-calculus to decide whether a certain query is identifiable, our method encodes the same invariances as an inductive bias while being amenable to  gradient-based optimization, allowing us to perform both tasks in an integrated fashion (in a way, addressing Pearl's concerns alluded to in Footnote \ref{footnote:pearl}). Specifically, our contributions are as follows:

\vspace{-0.08in}
\begin{enumerate}[wide, labelwidth=0pt, labelindent=0pt]
    \item  $[$Sec.~\ref{section:ncms}$]$ We introduce a special yet simple type of SCM that is  amenable to gradient descent called a \emph{neural causal model} (\ncm{}). We prove basic properties of this class of models, including its universal expressiveness and ability to encode an inductive bias representing certain structural invariances (Thm.~\ref{thm:lrepr}-\ref{thm:nscm-g-uat}). Notably, we show that despite the \ncm's expressivity, it still abides by the Causal Hierarchy Theorem (Corol.~\ref{cor:ncht}). 
    \item $[$Sec.~\ref{sec:neural-id}$]$ We formalize the problem of neural identification  (Def.~\ref{def:nscm-id}) and prove a duality between identification in causal diagrams and in neural causal models  (Thm.~\ref{thm:nscm-id-equivalence}). We introduce an operational way to perform inferences in \ncms{} (Corol.~\ref{thm:dual-graph-id}-\ref{thm:markovid}) and a sound and complete algorithm to jointly train and decide effect identifiability for an NCM (Alg.~\ref{alg:nscm-solve-id},
    Corol.~\ref{thm:nscm-id-correctness}).
    \item $[$Sec.~\ref{sec:ncm-fitting}$]$ Building on these results, we develop a gradient descent algorithm to jointly identify and estimate causal effects  (Alg.~\ref{alg:ncm-learn-pv}). 
\end{enumerate}
\vspace{-0.09in}
There are multiple ways of grounding these theoretical results. In  Sec.~\ref{sec:experiments},  we perform experiments with one possible implementation which support the feasibility of the proposed approach. All appendices including proofs, experimental details, and examples can be found in the full technical report \citep{xia:etal21}.

\vspace{-0.10in}
\subsection{Preliminaries}
\vspace{-0.1in}

In this section, we provide the necessary background to understand this work, following the presentation in \citep{pearl:2k}.  An uppercase letter $X$ indicates a random variable, and a lowercase letter $x$ indicates its corresponding value; bold uppercase  $\mathbf{X}$ denotes a set of random variables, and lowercase letter $\mathbf{x}$ its corresponding values. We use $\cD_{X}$ to denote the domain of $X$ and $\cD_{\mathbf{X}} = \cD_{X_1} \times \dots \times \cD_{X_k}$ for $\mathbf{X} = \{X_1, \dots, X_k\}$. We denote $P(\mathbf{X})$ as a probability distribution over a set of random variables $\mathbf{X}$ and $P(\mathbf{X} = \mathbf{x})$ as the probability of $\mathbf{X}$ being equal to the value of $\mathbf{x}$ under the distribution $P(\mathbf{X})$. For simplicity, we will mostly abbreviate $P(\mathbf{X} = \mathbf{x})$ as simply $P(\mathbf{x})$.
The basic semantic framework of our analysis rests on \textit{structural causal models} (SCMs) \citep[Ch.~7]{pearl:2k}, which are defined below. 

\begin{definition}[Structural Causal Model (SCM)]
    \label{def:scm} 
    An SCM $\mathcal{M}$ is a 4-tuple $\langle \mathbf{U}, \mathbf{V}, \mathcal{F}, P(\mathbf{U})\rangle$, where $\mathbf{U}$ is a set of exogenous variables (or ``latents'') that are determined by factors outside the model; $\mathbf{V}$ is a set $\{V_1, V_2, \ldots, V_n\}$ of (endogenous) variables of interest that are determined by other variables in the model -- that is, in $\mathbf{U} \cup \mathbf{V}$; $\mathcal{F}$ is a set of functions $\{f_{V_1}, f_{V_2}, \ldots, f_{V_n}\}$ such that each $f_i$ is a mapping from (the respective domains of) $\Ui{V_i} \cup \Pai{V_i}$ to $V_{i}$, where $\Ui{V_i} \subseteq \mathbf{U}$, $\Pai{V_i} \subseteq \mathbf{V} \setminus V_{i}$, and the entire set $\mathcal{F}$ forms a mapping from $\mathbf{U}$ to $\mathbf{V}$. That is, for $i=1,\ldots,n$, each $f_i \in \mathcal{F}$ is such that $v_i \leftarrow f_{V_i}(\pai{V_i}, \ui{V_i})$; and $P(\mathbf{u})$ is a probability function defined over the domain of $\mathbf{U}$.
    \hfill $\blacksquare$
\end{definition}

Each SCM $\mathcal{M}$ induces a  causal diagram $G$ where every $V_i \in \*V$ is a vertex, there is a directed arrow $(V_j \rightarrow V_i)$ for every $V_i \in \*V$ and $V_j \in Pa(V_i)$, and there is a dashed-bidirected arrow $(V_j  \dashleftarrow \dashrightarrow V_i)$ for every pair $V_i, V_j \in \*V$ such that  $\*U_{V_i}$ and $\*U_{V_j}$ are not independent. For further details on this construction, see \citep[Def.~13/16,~Thm.~4]{bareinboim:etal20}. The exogenous $\*U_{V_i}$'s are not assumed independent (i.e.\ Markovianity does not hold). We will consider here \emph{recursive} SCMs, which implies acyclic diagrams, and that the endogenous variables ($\*V$) are discrete and have finite domains.

We show next how an SCM $\cM$ gives values to the PCH's layers; for details on the semantics, see \citep[Sec.~1.2]{bareinboim:etal20}.
Superscripts are omitted when unambiguous.

\begin{definition}[Layers 1, 2 Valuations] 
\label{def:l2-semantics}
An SCM $\cM$ induces layer $L_2(\cM)$, a set of distributions over $\*V$, one for each intervention $\*x$.  For each ${\*Y} \subseteq \*V$, 
\begingroup\abovedisplayskip=0.5em\belowdisplayskip=0pt
\begin{align}
    \label{eq:def:l2-semantics}
    P^{\cM}({\*y}_{\*x}) = 
    \sum_{\{\*u \mid {\*Y}_{\*x}(\*u)={\*y}\}}
    P(\*u),
\end{align}
where ${\*Y}_{\*x}(\*u)$ is the solution for $\*Y$ after evaluating 
 $\mathcal{F}_{\*x} := \{f_{V_i} : V_i \in \*V \setminus \*X\} \cup \{f_X \leftarrow x : X \in \*X\}$. 
 \\The specific distribution $P(\*V)$, where $\*X$ is empty, is defined as layer $L_1(\cM)$.
\endgroup
\hfill $\blacksquare$
\end{definition}

In words, an external intervention forcing a set of variables $\*X$ to take values $\*x$ is modeled by replacing the original mechanism $f_X$ for each $X \in \*X$ with its corresponding value in $\*x$. This operation is represented formally by the do-operator, $do(\*X = \*x)$, 
and graphically as the \emph{mutilation} procedure. For the definition of the third layer, $L_3(\cM)$, see Def.~\ref{def:l3-semantics} in Appendix~\ref{app:proofs} or \citep[Def.~7]{bareinboim:etal20}.

\section{Neural Causal Models and the Causal Hierarchy Theorem}
\vspace{-0.08in}
\label{section:ncms}

In this section, we aim to resolve the tension between expressiveness and learnability (Fig.~\ref{fig:pch}). To that end, we define a special class of SCMs based on neural nets that is amenable to optimization and has the potential to act as a proxy for the true, unobserved SCM $\cM^*$.

\begin{definition}[NCM]
    \label{def:nscm}
    A Neural Causal Model (for short, NCM) $\widehat{M}(\bm{\theta})$ over variables $\*V$ with parameters $\bm{\theta} = \{\theta_{V_i} : V_i \in \*V\}$ is an SCM $\langle \widehat{\*U}, \*V, \widehat{\cF}, P(\widehat{\*U}) \rangle$ such that 
    \begin{itemize}
    \setlength\itemsep{-0.5em}
        \item $\widehat{\*U} \subseteq \{\widehat{U}_{\*C} : \*C \subseteq \*V\}$, where each $\widehat{U}$ is associated with some subset of variables $\*C \subseteq \*V$, and $\cD_{\widehat{U}} = [0, 1]$ for all $\widehat{U} \in \widehat{\*U}$.  (Unobserved confounding is present whenever $|\*C|>1$.)
        
        \item $\widehat{\cF} = \{\hat{f}_{V_i} : V_i \in \*V\}$, where each $\hat{f}_{V_i}$ is a feedforward neural network parameterized by $\theta_{V_i} \in \bm{\theta}$ 
        mapping values of $\Ui{V_i} \cup \Pai{V_i}$ to values of $V_i$ for some $\Pai{V_i} \subseteq \*V$ and $\Ui{V_i} = \{\widehat{U}_{\*C} : \widehat{U}_{\*C} \in \widehat{\*U}, V_i \in \*C\}$;
        
        \item $P(\widehat{\*U})$ is defined s.t. $\widehat{U} \sim \unif(0, 1)$ for each $\widehat{U} \in \widehat{\*U}$.
    \end{itemize}
    \vspace{-0.30in}
    \hfill $\blacksquare$
\end{definition}

There is a number of remarks worth making at this point.    

\begin{enumerate}[leftmargin=*,itemindent=1.25em]
\item \textbf{[Relationship NCM $\rightarrow$ SCM}] By definition, all NCMs are SCMs, which means NCMs have the capability of generating any distribution associated with the PCH's layers. 
\item \textbf{[Relationship SCM $\nrightarrow$ NCM}]  On the other hand, not all SCMs are NCMs, since Def. \ref{def:nscm} dictates that $\widehat{\*U}$ follows uniform distributions in the unit interval and $\widehat{\cF}$ are feedforward neural networks. 
\item \textbf{[Non-Markovianity]} For any two endogenous variables $V_i$ and $V_j$, it is the case that $\*U_{V_i}$ and $\*U_{V_j}$ might share an input from $\widehat{\*U}$, which will play a critical role in causality, not ruling out \textit{a priori} the possibility of unobserved confounding  and violations of Markovianity.
\item \textbf{[Universality of Feedforward Nets]} Feedforward networks are universal approximators \cite{Cybenko89, HORNIK1991251} (see also  \cite{Goodfellow-et-al-2016}), and any probability distribution can be generated by the uniform one (e.g., see \textit{probability integral transform} \cite{angus:inttransform}). This  suggests that the pair $\langle \widehat{\cF}, P(\widehat{\*U}) \rangle$ may be expressive enough for modeling  $\cM^*$'s mechanisms $\cF$ and distribution $P(\*U)$ without loss of generality. 
\item \textbf{[Generalizations / Other Model Classes] } The particular modeling choices  within the definition above were made for the sake of explanation, and the results discussed here still hold for other, arbitrary classes of functions and probability distributions, as shown in Appendix~\ref{app:generalizations}. 
\end{enumerate}

To compare the expressiveness of NCMs and SCMs, we introduce the following definition.

\begin{definition}[P$^{(L_i)}$-Consistency]
    \label{def:li-consistency}
    Consider two SCMs, $\cM_1$ and $\cM_2$. $\cM_2$ is said to be P$^{(L_i)}$-consistent (for short, $L_i$-consistent) w.r.t. $\cM_1$ if $L_i(\cM_1) = L_i(\cM_2)$.
    \hfill $\blacksquare$
\end{definition}
This definition applies to NCMs since they are also SCMs.  As shown below, NCMs can not only approximate the collection of functions of the true SCM $\cM^*$, but they can perfectly \textit{represent} all the observational, interventional, and counterfactual distributions. This property is, in fact,  special and not enjoyed by many neural models. (For examples and  discussion, see Appendix \ref{app:examples} and \ref{sec:nn-general}.)

\begin{restatable}[\ncm{} Expressiveness]{theorem}{lrepr} \label{thm:lrepr}
    For any SCM $\cM^* = \langle \*U, \*V, \cF, P(\*U) \rangle$, there exists an \ncm{} $\widehat{M}(\bm{\theta}) = \langle \widehat{\*U}, \*V, \widehat{\cF}, P(\widehat{\*U}) \rangle$ s.t. $\widehat{M}$ is $L_3$-consistent w.r.t. $\cM^*$.
    \hfill $\blacksquare$
\end{restatable}

\vspace{-0.05in}
Thm.~\ref{thm:lrepr} ascertains that there is no loss of expressive  power using \ncms{} despite the constraints imposed over its form, i.e., \ncms{} are as expressive as  SCMs. 
One might be tempted to surmise, therefore, that an \ncm{} can be trained on the observed data and act as a proxy for the true SCM $\cM^*$, and inferences about other quantities of $\cM^*$ can be done through  computation directly in $\widehat{\cM}$. Unfortunately, this is almost never the case: \footnote{Multiple examples of this phenomenon are discussed in Appendix \ref{app:examples-expr} and \citep[Sec.~1.2]{bareinboim:etal20}}

\begin{restatable}[Neural Causal Hierarchy Theorem (N-CHT)]{corollary}{ncht} 
\label{cor:ncht}
    Let $\Omega^*$ and $\Omega$ be the sets of all SCMs and NCMs, respectively. We say that Layer $j$ of the causal hierarchy for NCMs \emph{collapses} to Layer $i$ ($i < j$) relative to $\cM^* \in \Omega^*$ if $L_i(\cM^*) = L_i(\widehat{M})$ implies that $L_j(\cM^*) = L_j(\widehat{M})$ for all $\widehat{M} \in \Omega$. Then, with respect to the Lebesgue measure over (a suitable encoding of $L_3$-equivalence classes of) SCMs, the subset in which Layer $j$ of NCMs collapses to Layer $i$ has measure zero.
    \hfill $\blacksquare$ 
\end{restatable}
  
\vspace{-0.05in}

This corollary  highlights the fundamental challenge of performing inferences across the PCH layers even when the target object (NCM $\widehat{\cM}$) is a suitable surrogate for the underlying SCM $\cM^*$, in terms of expressiveness and capability of generating the same observed distribution. That is, expressiveness does not mean that the learned object has the same empirical content as the generating model. 
For concrete examples of the expressiveness of NCMs and why it is insufficient for causal inference, see Examples~\ref{ex:expressiveness-basic} and \ref{ex:ncht} in Appendix~\ref{app:examples-expr}. Thus, structural assumptions are necessary to perform causal inferences when using NCMs, despite their expressiveness. We discuss next  how to incorporate the necessary assumptions into an NCM to circumvent the limitation highlighted by Corol.~\ref{cor:ncht}.

\vspace{-0.05in}
\subsection{A Family of Neural-Interventional Constraints (Inductive Bias)}
In this section, we investigate constraints about $\cM^*$ that will narrow down the hypothesis space and possibly allow for valid cross-layer inferences.
One well-studied family of structural constraints comes in the form of a pair comprised of a  collection of interventional distributions  $\cP$ and causal diagram $\cG$, known as a \emph{causal bayesian network} (CBN) (Def.~\ref{def:cbn}; see also  \citep[Thm.~4]{bareinboim:etal20})).
The diagram $\cG$ encodes constraints over the space of interventional distributions $\cP$ which are useful to perform cross-layer inferences (for details, see Appendix \ref{app:examples-cg}). For simplicity, we focus on interventional inferences from observational data.  To compare the constraints entailed by distinct SCMs, we define the following notion of consistency: 
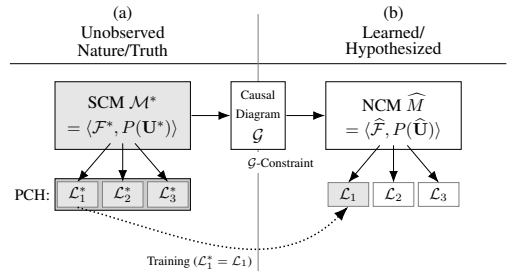
\begin{wrapfigure}{r}{0.45\textwidth}
  \begin{center}
    \begin{tikzpicture}[-, scale=0.6, every node/.append style={transform shape}]
        \draw [line width=0.05mm] (-5.5,1.9) -- (5.5,1.9);
        \draw [line width=0.05mm, draw=gray] (0,3) -- (0,-2.7);
        
        \node [align=center, font=\fontsize{11}{0}\selectfont] (atext1) at (-3, 3) {(a)};
        \node [align=center, font=\fontsize{11}{0}\selectfont] (atext2) at (-3, 2.6) {Unobserved};
        \node [align=center, font=\fontsize{11}{0}\selectfont] (atext3) at (-3, 2.2) {Nature/Truth};
        \node [align=center, font=\fontsize{11}{0}\selectfont] (btext1) at (3, 3) {(b)};
        \node [align=center, font=\fontsize{11}{0}\selectfont] (btext2) at (3, 2.6) {Learned/};
        \node [align=center, font=\fontsize{11}{0}\selectfont] (btext3) at (3, 2.2) {Hypothesized};
        
        \filldraw [fill=gray!40,line width=0.01mm] (-4.5, -0.65) rectangle (-1.5, -1.35);
        
        \node (pch) at (-5, -1) {PCH:};
    
        \filldraw [fill=gray!20,line width=0.01mm] (-4.5, 1.5) rectangle (-1.5, 0);
        \node [align=center, font=\fontsize{11}{0}\selectfont] (Pu1) at (-3,1.05) {SCM $\cM^*$};
        \node [align=center, font=\fontsize{11}{0}\selectfont] (Pu) at (-3,0.45) {$=\langle \cF^*, P(\*U^*) \rangle$};
    
       	\node [fill=gray!20, draw=gray, line width=0.08mm, align=center, minimum width=0.9cm] (L1Pv) at (-4,-1) {$\Ll_1^*$};
       	\node [fill=gray!20, draw=gray, line width=0.08mm, align=center, minimum width=0.9cm] (L2Pv) at (-3,-1) {$\Ll_2^*$};
       	\node [fill=gray!20, draw=gray, line width=0.08mm, align=center, minimum width=0.9cm] (L3Pv) at (-2,-1) {$\Ll_3^*$};
       	
       	\path [-Latex, line width=0.1mm] (Pu) edge (L1Pv.north);
       	\path [-Latex, line width=0.1mm] (Pu) edge (L2Pv.north);
       	\path [-Latex, line width=0.1mm] (Pu) edge (L3Pv.north);

       	\filldraw [fill=gray!0,line width=0.01mm] (4.5, 1.5) rectangle (1.5, 0);
        \node [align=center, font=\fontsize{11}{0}\selectfont] (nPu1) at (3,1.05) {NCM $\widehat{M}$};
        \node [align=center, font=\fontsize{11}{0}\selectfont] (nPu) at (3,0.45) {$= \langle \widehat{\cF}, P(\widehat{\*U})\rangle$};
    
       	\node [fill=gray!20, draw=gray, line width=0.08mm, align=center, minimum width=0.9cm] (nL1Pv) at (2,-1) {$\Ll_1$};
       	\node [draw=gray, line width=0.08mm, align=center, minimum width=0.9cm] (nL2Pv) at (3,-1) {$\Ll_2$};
       	\node [draw=gray, line width=0.08mm, align=center, minimum width=0.9cm] (nL3Pv) at (4,-1) {$\Ll_3$};
    
       	\path [-Latex, line width=0.1mm] (nPu) edge (nL1Pv.north);
       	\path [-Latex, line width=0.1mm] (nPu) edge (nL2Pv.north);
       	\path [-Latex, line width=0.1mm] (nPu) edge (nL3Pv.north);
       	
       	\path [-Latex, line width=0.2mm] (L1Pv.south) edge[densely dotted, out=-15, in=-135] (nL1Pv.south);
       	\node [align=center, font=\fontsize{8}{0}\selectfont, fill=white, inner sep=0.5mm] (trainingtext) at (-1.3, -2.5) {Training ($\cL_1^* = \cL_1$)};
       	
       	\filldraw [fill=gray!0,line width=0.01mm] (0.6, 1.5) rectangle (-0.6, 0);
        \node [align=center, font=\fontsize{8}{0}\selectfont] (nG1) at (0,1.2) {Causal};
        \node [align=center, font=\fontsize{8}{0}\selectfont] (nG2) at (0,0.75) {Diagram};
        \node [align=center, font=\fontsize{11}{0}\selectfont] (nG3) at (0,0.3) {$\cG$};
        
        \node [inner sep=0] (scmright) at (-1.5,0.75) {};
        \node [inner sep=0] (cgleft) at (-0.6,0.75) {};
        \node [inner sep=0] (cgright) at (0.6,0.75) {};
        \node [inner sep=0] (ncmleft) at (1.5,0.75) {};
        \path [-Latex, line width=0.1mm] (scmright) edge (cgleft);
        \path [-Latex, line width=0.1mm] (cgright) edge (ncmleft);
        
        \node [align=center, font=\fontsize{8}{0}\selectfont, fill=white] (gconstrainttext) at (0.5, -0.3) {$\cG$-Constraint};
        
    \end{tikzpicture}
    \caption{The l.h.s. contains the true SCM $\cM^*$ that induces PCH's three layers. The r.h.s. contains an NCM that is trained with layer 1 data. The  matching shading indicates that the two models agree with respect to $L_1$ while not necessarily agreeing in layers 2 and 3. The causal diagram $\cG$ entailed by $\cM^*$ is used as an inductive bias for $\widehat{M}$.}
    \label{fig:cht}
    \end{center}
    \vspace{-0.45in}
\end{wrapfigure}
\begin{definition}[$\cG$-Consistency]
    \label{def:g-consistency}
    Let $\cG$ be the causal diagram induced by SCM $\cM^*$. For any SCM $\cM$, we say that $\cM$ is $\cG$-consistent (w.r.t. $\cM^*$) if $\cG$ is a CBN for $L_2(\cM)$.
    \hfill $\blacksquare$
\end{definition}
In the context of \ncms{}, this means that $\cM$ would impose the same constraints over $\cP$ as the true SCM $\cM^*$ (since $\cG$ is also a CBN for $L_2(\cM^*)$ by \citep[Thm.~4]{bareinboim:etal20}).  Whenever the corresponding  diagram $\cG$ is known, one should only consider \ncms{} that are $\cG$-consistent. \footnote{Otherwise, the causal diagram can be learned through structural learning algorithms from observational data \cite{spirtes:etal00,peters:17} or experimental data \cite{kocaoglu2017experimental,kocaoglu2019characterization,jaber2020cd}. See the next footnote for a neural take on this task. \label{footnote:structural-learning}} We provide below a systematic way of constructing $\cG$-consistent \ncms{}.

\begin{definition}[$C^2$-Component]
    For a causal diagram $\cG$, a subset $\*C \subseteq \*V$ is a complete confounded component (for short, $C^2$-component) if any pair $V_i, V_j \in \*C$ is connected with a bidirected arrow in $\cG$ and is maximal (i.e. there is no $C^2$-component $\*C'$ for which $\*C \subset \*C'$.)
    \hfill $\blacksquare$
\end{definition}

\begin{definition}[$\cG$-Constrained \ncm{} (constructive)]
    \label{def:g-cons-nscm}
    Let $\cG$ be the causal diagram induced by  SCM $\cM^*$.
    Construct NCM $\widehat{M}$ as follows. \textbf{(1)} Choose $\widehat{\*U}$ s.t. $\widehat{U}_{\*C} \in \widehat{\*U}$ if and only if $\*C$ is a $C^2$-component in $\cG$. \textbf{(2)} For each variable $V_i \in \*V$, choose $\Pai{V_i} \subseteq \*V$ s.t. for every $V_j \in \*V$, $V_j \in \Pai{V_i}$ if and only if there is a directed edge from $V_j$ to $V_i$ in $\cG$.
    Any \ncm{} in this family is said to be $\cG$-constrained.
     $\blacksquare$
\end{definition}
\vspace{-0.03in}

Note that this represents a family of \ncms{}, not a unique one, since $\bm{\theta}$ (the parameters of the neural networks)
are not yet specified by the construction, only the scope of the function and independence relations among the sources of randomness ($\widehat{\*U}$). In contrast to SCMs where both $\langle \cF, P(\*u) \rangle$ can freely vary, the degrees of freedom within \ncms{} come from $\bm{\theta}$.
\footnote{There is a growing literature that models SCMs using neural nets as functions, but which differ in nature and scope to our work. Broadly, these works assume Markovianity, which entails strong constraints over $P(\*U)$ and, in the context of identification, implies that all effects are always identifiable; see Corol.~\ref{thm:markovid}. 
For instance, 
\citep{goudet2018fcm} attempts to learn the entire SCM from observational ($L_1$) data, while  \citep{bengio_causal_mechanisms:2020,brouillard:2020} also leverages experimental ($L_2$) data. On the inference side, 
\citep{kocaoglu2017causalgan} focuses on estimating causal effects of labels on images. 
\label{footnote:other-ncms}}

We show next that an \ncm{} constructed following the procedure dictated by Def.~\ref{def:g-cons-nscm} encodes all the constraints of the original causal diagram. 

\begin{restatable}[\ncm{} $\cG$-Consistency]{theorem}{nscmgcons}
    \label{thm:nscm-g-cons}
    Any $\cG$-constrained \ncm{} $\widehat{M}(\bm{\theta})$ is $\cG$-consistent. 
    \hfill $\blacksquare$
\end{restatable}

We show next the implications of imposing the structural constraints embedded in the causal diagram.

\begin{restatable}[$L_2$-$\cG$ Representation]{theorem}{lgrepr}
    \label{thm:nscm-g-uat}
    For any SCM
    $\cM^*$
    that induces causal diagram $\cG$, there exists a $\cG$-constrained \ncm{} $\widehat{M}(\bm{\theta}) = \langle \widehat{\*U}, \*V, \widehat{\cF}, P(\widehat{\*U}) \rangle$ 
    that is $L_2$-consistent w.r.t. $\cM^*$.
    \hfill $\blacksquare$
\end{restatable}
The importance of this result 
stems from the fact that despite constraining the space of NCMs to those compatible with $\cG$, the resultant family is still expressive enough to represent the entire Layer 2 of the original, unobserved SCM $\cM^*$.

Fig.~\ref{fig:cht} provides a mental picture useful to understand the results discussed so far.  The true SCM $\cM^*$ generates the three layers of the causal hierarchy (left side), but in many settings  only observational data (layer 1) is visible. An NCM $\widehat{M}$ trained with this data is capable of perfectly representing $L_1$ (right side). For almost any generating $\cM^*$ sampled from the space $\Omega^*$, there exists an \ncm{} $\widehat{M}$ that exhibits the same behavior with respect to observational data ($\widehat{M}$ is $L_1$-consistent) but exhibits a different behavior with respect to interventional data. In other words, $L_1$ underdetermines $L_2$. (Similarly, $L_1$ and $L_2$ underdetermine $L_3$ \cite[Sec.~1.3]{bareinboim:etal20}.) Still, the true SCM $\cM^*$ also induces a causal diagram $\cG$ that encodes constraints over the interventional distributions. If we use this collection of constraints as an inductive bias, imposing $G$-consistency in the construction of the \ncm{}, $\widehat{M}$ may agree with those of the true $\cM^*$ under some conditions, which we will investigate in the next section.  

\vspace{-0.2in}
\section{The Neural Identification Problem} \label{sec:neural-id}
\vspace{-0.12in}

We now investigate the feasibility of causal inferences 
in the class of $\cG$-constrained \ncms{}.  \footnote{This is akin to what happens with the non-neural CHT \citep[Thm.~1]{bareinboim:etal20} and the subsequent use of causal diagrams to encode the necessary inductive bias, and in which the do-calculus allows for cross-layer inferences directly from the graphical representation  \citep[Sec.~1.4]{bareinboim:etal20}.}   The first step is to refine the notion of identification \citep[pp.~67]{pearl:2k} to inferences within this class of models. 

\begin{definition}[Neural Effect Identification]
    \label{def:nscm-id}
    Consider any arbitrary SCM $\cM^*$ and the corresponding causal diagram $\cG$ and observational distribution $P(\*V)$. 
    The causal effect $P(\*y \mid do(\*x))$ is said to be neural-identifiable from the set of $\cG$-constrained NCMs $\Omega(\cG)$ and observational distribution $P(\*V)$ if and only if  $P^{\widehat{M}_1}(\*y \mid do(\*x)) = P^{\widehat{M}_2}(\*y \mid do(\*x))$ for every pair of models $\widehat{M}_1, \widehat{M}_2 \in \Omega(\cG)$ s.t. $P^{\cM^*}(\*V) = P^{\widehat{M}_1}(\*V) = P^{\widehat{M}_2}(\*V) > 0$.
    \hfill $\blacksquare$
\end{definition}

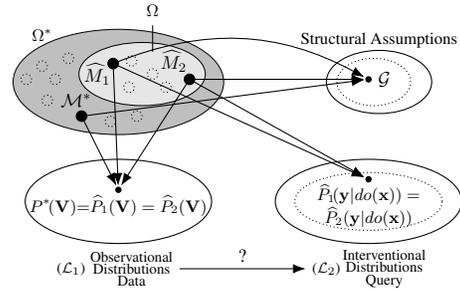
\begin{wrapfigure}{r}{0.42\textwidth}
    \vspace{-0.35in}
    \begin{tikzpicture}[-, scale=0.7, every node/.append style={transform shape}]
        \filldraw [fill=gray!50, line width=0.01mm] (-2.05,-0.05) ellipse (2.0 and 1.0);
        \filldraw [fill=gray!20, line width=0.01mm] (-1.6,0.1) ellipse (1.2 and 0.6);
      	\node [align=center, font=\fontsize{9}{0}\selectfont] at (-3.5,0.85) {$\Omega^*$};
      	\node [align=center, font=\fontsize{9}{0}\selectfont] at (-1.4,1.3) {$\Omega$};
      	\node [inner sep=0] (omg) at (-1.4,1.05) {};
      	\node [inner sep=0] (omgspace) at (-1.4,0.5) {};
      	\path [-] (omg) edge (omgspace);
      	
      	\node at (-2.85,-0.4) {$\cM^*$};
      	\draw [fill=black] (-2.75,-0.7) circle (0.1);
      	\node [inner sep=0] (ms) at (-2.75,-0.7) {\ };

      	\draw [fill=black] (-2.15,0.3) circle (0.1);
      	\node [inner sep=0] (mp) at (-2.15,0.3) {};
      	\node at (-2.45,0.1) {$\widehat{M}_1$};
      	\draw [fill=black] (-0.7,0) circle (0.1);
      	\node [inner sep=0] (mp2) at (-0.7,0) {};
      	\node at (-1,0.35) {$\widehat{M}_2$};
      	
      	\draw [densely dotted] (-1.8,0.35) circle (0.1);
      	\draw [densely dotted] (-1.85,-0.15) circle (0.1);
      	\draw [densely dotted] (-1.15,-0.15) circle (0.1);
      	\draw [densely dotted] (-1.4,0.1) circle (0.1);
      	
      	\draw [densely dotted] (-3.0,0.4) circle (0.1);
        \draw [densely dotted] (-3.5,0.25) circle (0.1);
        \draw [densely dotted] (-3.25,0.0) circle (0.1);
        \draw [densely dotted] (-3.8, 0.0) circle (0.1);
        \draw [densely dotted] (-3.4,-0.4) circle (0.1);
        \draw [densely dotted] (-2.2,-0.8) circle (0.1);
        \draw [densely dotted] (-1.7,-0.7) circle (0.1);

      	\draw (-2.05,-2.32) ellipse (1.77 and 0.8);
      	\node [align=center, font=\fontsize{8}{0}\selectfont] (l1) at (-2.95,-3.6) {($\Ll_1$)};
      	\node [align=center, font=\fontsize{8}{0}\selectfont] (l1) at (-1.8,-3.6) {Observational\\ Distributions\\ Data};
      	\node at (-2.05,-2.4) {\small $P^*\!(\*V){=}\widehat{P}_1(\*V) = \widehat{P}_2(\*V)$};
      	\draw [fill=black] (-2.05,-2.1) circle (0.05);
      	\node [inner sep=0.2em] (pl1) at (-2.05,-2.1) {\ };
    
      	\draw (2.7,-2.32) ellipse (1.77 and 0.8);
      	\node [align=center, font=\fontsize{8}{0}\selectfont] (l2) at (1.9,-3.6) {($\Ll_2$)};
      	\node [align=center, font=\fontsize{8}{0}\selectfont] (l2t) at (3.0,-3.6) {Interventional\\ Distributions\\ Query};
      	\draw [densely dotted] (2.7,-2.32) ellipse (1.4 and 0.55);
      	\node [align=center] at (2.7,-2.4) {{\small $\widehat{P}_1\!(\*y | do(\*x))=$} \\ {\small $\widehat{P}_2(\*y | do(\*x))$}};
    
      	\draw [fill=black] (2.7,-1.9) circle (0.05);
      	\node [inner sep=0.2em] (pl21) at (2.7,-1.9) {\ };

      	\path [-Latex] (ms) edge (pl1);
      	\path [-Latex] (mp) edge (pl1);
      	\path [-Latex] (mp) edge (pl21);
      	\path [-Latex] (l1) edge node[above]  {?} (l2);
      	\path [-Latex] (mp2) edge (pl1);
      	\path [-Latex] (mp2) edge (pl21);
      	
      	\draw (2.9,-0.05) ellipse (1 and 0.6);
      	\node [align=center, font=\fontsize{9}{0}\selectfont] at (2.9,0.75) {Structural Assumptions};
      	\draw [densely dotted] (2.8,-0.05) ellipse (0.7 and 0.45);
      	\node [align=center] at (3.0,0) {$\cG$};
      	\draw [fill=black] (2.7,0) circle (0.05);
      	\node [inner sep=0.2em] (g1) at (2.7,0) {\ };
      	\path [-Latex] (mp) edge [bend left=25] (g1);
      	\path [-Latex] (ms) edge (g1);
      	\path [-Latex] (mp2) edge (g1);
    \end{tikzpicture}
    
    \caption[Neural ID]{
    $P(\*y \mid do(\*x))$ is identifiable from $P(\*V)$ and $\Omega(\cG)$ if for any SCM $\cM^* \in \Omega^*$ and NCMs $\widehat{M}_1, \widehat{M}_2 \in \Omega$ (top left), $\widehat{M}_1, \widehat{M}_2, \cM^*$ match in $P(\*V)$ (bottom left) and $\cG$ (top right), then the NCMs $\widehat{M}_1$, $\widehat{M}_2$ also match in $P(\*y \mid do(\*x))$ (bottom right).
    }
    \vspace{-0.2in}
    \label{fig:omegastar}
\end{wrapfigure}

In the context of graphical identifiability \citep[Def.~3.2.4]{pearl:2k} and do-calculus, an effect is
identifiable if any SCM in $\Omega^*$ compatible with the observed causal diagram and capable of generating the observational distribution matches the interventional query. If we constrain our attention to \ncms{}, identification in the general class would imply identification in \ncms{}, naturally, since it needs to hold for all SCMs. On the other hand, it may be insufficient to constrain identification within the \ncm{} class, like in Def.~\ref{def:nscm-id}, since it is conceivable that the effect could match within the class (perhaps in a not very expressive neural architecture) while there still exists an SCM that generates the same observational distribution and induces the same diagram, but does not agree in the interventional query; see Example \ref{ex:nonexpr} in Appendix \ref{app:examples}.
The next result shows that this is never the case with NCMs, and there is no loss of generality when deciding identification through the NCM class.

\begin{restatable}[Graphical-Neural Equivalence (Dual ID)]{theorem}{idequivalence}
    \label{thm:nscm-id-equivalence}
    Let $\Omega^*$ be the set of all SCMs and $\Omega$ the set of \ncms{}. Consider the true SCM $\cM^*$ and the corresponding causal diagram $\cG$. Let $Q = P(\mathbf{y} \mid do(\mathbf{x}))$ be the query of interest and $P(\mathbf{V})$ the observational distribution. Then, $Q$ is neural identifiable from $\Omega(\cG)$ and $P(\mathbf{V})$ if and only if it is identifiable from $\cG$ and $P(\mathbf{V})$.
    \hfill $\blacksquare$
\end{restatable}

In words,  
Theorem~\ref{thm:nscm-id-equivalence}  relates the solution space of these two classes of models, which means that the identification status of a query is preserved across settings. For instance, if an effect is identifiable from the combination of a causal graph $\cG$ and $P(\*v)$, it will also be identifiable from $\cG$-constrained \ncm{}s (and the other way around). This is encouraging since our goal is to perform inferences directly through neural causal models, within $\Omega(\cG)$, avoiding the symbolic nature of do-calculus computation; the theorem guarantees that this is achievable \textit{in principle}.

\begin{restatable}[Neural Mutilation (Operational ID)]{corollary}{dualid}
    \label{thm:dual-graph-id}
    Consider the true SCM $\cM^* \in \Omega^*$, causal diagram $\cG$,  the observational distribution $P(\mathbf{V})$, and a target query $Q$ equal to $P^{\cM^*}(\*y \mid do(\*x))$. 
    Let $\widehat{M} \in \Omega(\cG)$ be a $\cG$-constrained \ncm{} that is $L_1$-consistent with $\cM^*$. If $Q$ is identifiable from $\cG$ and $P(\mathbf{V})$, then $Q$ is computable through a mutilation process on a proxy \ncm{} $\widehat{M}$, i.e.,  for each $X \in \*X$, replacing the equation $f_x$ with a constant $x$ ($Q = \textsc{proc-mutilation}(\widehat{M}; \*X = \*x, \*Y$)).
    \hfill $\blacksquare$
\end{restatable}

Following the duality stated by Thm.~\ref{thm:nscm-id-equivalence}, 
this result provides a practical, operational way of evaluating queries in \ncms{}: inferences may be carried out through the process of mutilation, which gives semantics to queries in the generating SCM $\cM^*$ (via Def.~\ref{def:l2-semantics}). What is interesting here is that the proposition provides conditions under which this process leads to valid inferences, even when $\cM^*$ is unknown,  or when the mechanisms $\mathcal F$ and exogenous distribution $P(\*U)$ of $\cM^*$ and the corresponding functions and distribution of the proxy \ncm{} $\widehat{M}$ do not match. (For concreteness, refer to example~\ref{ex:id} in Appendix.~\ref{app:examples}.)
In words, inferences using mutilation on $\widehat{M}$ would work as if they were on $\cM^*$ itself, and they would be correct so long as certain stringent properties were satisfied --  $L_1$-consistency, $\cG$-constraint, and identifiability. As shown earlier, if these properties are not satisfied, inferences within a proxy model
will almost never be valid, likely bearing no relationship with the ground truth. 
(For fully worked out instances of this situation,  refer to examples~\ref{ex:ncht}, \ref{ex:cg}, or \ref{ex:non-id} in Appendix~\ref{app:examples}). 

Still, one special class of SCMs in which any interventional distribution is identifiable is called \textit{Markovian}, where all $U_i$ are assumed independent and affect only one endogenous variable $V_i$. 
\begin{restatable}[Markovian Identification]{corollary}{markovid}
    \label{thm:markovid}
    Whenever the $\cG$-constrained \ncm{} $\widehat{\cM}$ is Markovian, $P(\*y \mid do(\*x))$ is always identifiable through the process of mutilation in the proxy \ncm{} (via Corol.~\ref{thm:dual-graph-id}). 
    \hfill $\blacksquare$
\end{restatable}
\begin{wrapfigure}{r}{0.5\textwidth}

\IncMargin{1em}
\vspace{-0.18in}
\begin{algorithm}[H]
    \scriptsize
    \setstretch{0.9}
    \renewcommand{\AlCapSty}[1]{\normalfont\scriptsize{\textbf{#1}}\unskip}
    
    \DontPrintSemicolon
    \SetKwData{ncmdata}{$\widehat{M}$}
    \SetKwData{graphdata}{$\cG$}
    \SetKwData{variabledata}{$\*V$}
    \SetKwData{pvdata}{$P(\*V)$}
    \SetKwData{thetamin}{$\bm{\theta}_{\min}^*$}
    \SetKwData{thetamax}{$\bm{\theta}_{\max}^*$}
    \SetKwFunction{ncmfunc}{NCM}
    \SetKwInOut{Input}{Input}
    \SetKwInOut{Output}{Output}
    
    \Input{ causal query $Q = P(\*y \mid do(\*x))$, $L_1$ data \pvdata, and causal diagram \graphdata}
    \Output{ $P^{\cM^*}(\*y \mid do(\*x))$ if identifiable, \texttt{FAIL} otherwise.}
    \BlankLine
    \textls[-20]{
    $\ncmdata \gets \ncmfunc{\variabledata, \graphdata}$ \tcp*{from Def.\ \ref{def:g-cons-nscm}}
    $\thetamin \! \gets \! \arg \min_{\bm{\theta}} P^{\ncmdata(\bm{\theta})}(\*y \! \mid \! do(\*x))$ s.t. $L_1(\ncmdata(\bm{\theta})) \! = \! \pvdata$\;
    $\thetamax \! \gets \! \arg \max_{\bm{\theta}} P^{\ncmdata(\bm{\theta})}(\*y \! \mid \! do(\*x))$ s.t. $L_1(\ncmdata(\bm{\theta})) \! = \! \pvdata$\;
    \eIf{$P^{\ncmdata(\thetamin)}(\*y \mid do(\*x)) \neq P^{\ncmdata(\thetamax)}(\*y \mid do(\*x))$} {
        \Return \texttt{FAIL}
    }{
        \Return $P^{\ncmdata(\thetamin)}(\*y \mid do(\*x))$ \tcp*{choose min or max arbitrarily} 
    }
    }
    \caption{\scriptsize Identifying/estimating queries with NCMs.}
    \label{alg:nscm-solve-id}
\end{algorithm}
\DecMargin{1em}
\vspace{-0.2in}
\end{wrapfigure}

This is obviously not the case for general non-Markovian models, which leads to the very problem of identification. In these cases, we need to decide whether the mutilation procedure (Corol.~\ref{thm:dual-graph-id}) can, in principle, produce the correct answer. 
We show in Alg.~\ref{alg:nscm-solve-id} 
a learning procedure that decides whether a certain effect is identifiable from observational data. Intuitively, the procedure searches for two models that respectively minimize and maximize the target query while maintaining $L_1$-consistency with the data distribution. If the $L_2$ query values induced by the two models are equal, then the effect is identifiable, and the value is returned; otherwise, the effect is non-identifiable. 
Remarkably, the procedure is both necessary and sufficient, which means that all, and only, identifiable effects are classified as such by our procedure. This implies that, theoretically, deep learning could be as powerful as the do-calculus in deciding identifiability.
(For a more nuanced discussion of symbolic versus optimization-based approaches for identification, see Appendix \ref{app:examples-symbolic-neural}. For non-identifiability examples and further discussion, see \ref{app:examples-id}.)

\begin{restatable}[Soundness and Completeness]{corollary}{completeness}
    \label{thm:nscm-id-correctness}
    Let $\Omega^*$ be the set of all SCMs,  $\cM^* \in \Omega^*$ be the true SCM inducing causal diagram $\cG$, $Q = P(\*y \mid do(\*x))$ be a query of interest, and $\widehat{Q}$ be the result from running Alg.~\ref{alg:nscm-solve-id} 
     with inputs $P^*(\*V) = L_1(\cM^*) > 0$, $\cG$, and $Q$. Then $Q$ is identifiable from $\cG$ and $P^*(\*V)$ if and only if $\widehat{Q}$ is not \texttt{FAIL}. Moreover, if $\widehat{Q}$ is not \texttt{FAIL}, then $\widehat{Q} = P^{\cM^*}(\*y \mid do(\*x))$.
    \hfill $\blacksquare$
\end{restatable}

\vspace{-0.16in}
\section{The Neural Estimation Problem} \label{sec:ncm-fitting}
\vspace{-0.14in}

While identifiability is fully solved by the asymptotic theory discussed so far (i.e., it is both necessary and sufficient), we now consider the problem of estimating causal effects in practice under imperfect optimization and finite samples and computation.
For concreteness, we discuss next the discrete case with binary variables, but our construction extends naturally to categorical and continuous variables (see Appendix \ref{app:experiments}). 
We propose next a construction of a $\cG$-constrained NCM
$\widehat M(\mathcal G; \bm{\theta}) = \langle \widehat{\*U}, \*V, \widehat{\cF}, P(\widehat{\*U}) \rangle$, which is a possible instantiation of  Def.~\ref{def:g-cons-nscm}:
\begin{eqnarray}\label{eq:ncm-differentiable}
    \begin{cases}
        \mathbf V &:= \mathbf V, \hspace{+0.05in}
        \widehat{\*U} := \{ U_{\mathbf C}: \mathbf C \in C^2(\mathcal G) \} \cup \{ G_{V_i} : V_i \in \mathbf V \}, \\
        \widehat{\mathcal F} &:= \left\{ f_{V_i} := \arg \max_{j \in \{0, 1\}} g_{j, V_i} + \begin{cases}
            \log \sigma(\phi_{V_i}(\pai{V_i}, \ui{V_i}^c; \theta_{V_i})) & j = 1 \\
            \log (1 - \sigma(\phi_{V_i}(\pai{V_i}, \ui{V_i}^c; \theta_{V_i}))) & j = 0
        \end{cases} \right\}, \\
        P(\widehat{\mathbf U}) &:= \{ U_{\mathbf C} \sim \mathrm{Unif}(0, 1) : U_{\mathbf C} \in \mathbf U \} \; \cup \\
         & \hspace{0.17in}  \{ G_{j, V_i} \sim \mathrm{Gumbel}(0, 1) : V_i \in \mathbf V, j \in \{0, 1\} \},
    \end{cases}
\end{eqnarray}

where $\mathbf V$ are the nodes of $\mathcal G$; $\sigma: \mathbb R \to (0, 1)$ is the sigmoid activation function; $C^2(\mathcal G)$ is the set of $C^2$-components of $\mathcal G$; each $G_{j, V_i}$ is a standard Gumbel random variable \citep{gumbel1954statistical}; each $\phi_{V_i}(\cdot; \theta_{V_i})$ is a neural net parameterized by $\theta_{V_i} \in \bm{\theta}$; $\pai{V_i}$ are the values of the parents of $V_i$; and $\ui{V_i}^c$ are the values of $\Ui{V_i}^c := \{ U_{\mathbf C} : U_{\mathbf C} \in \mathbf U \text{ s.t. } V_i \in \mathbf C \}$. The parameters $\bm{\theta}$ are not yet specified and must be learned through training to enforce $L_1$-consistency (Def.~\ref{def:li-consistency}). 

Let $\mathbf U^c$ and $\mathbf G$ denote the latent $C^2$-component variables and Gumbel random variables, respectively. To estimate $P^{\widehat{M}}(\*v)$ and  $P^{\widehat{M}}(\*y \mid do(\*x))$ given Eq.~\ref{eq:ncm-differentiable}, we may compute the probability mass of a datapoint $\mathbf v$ with intervention $do(\*X = \*x)$ ($\*X$ is empty when observational) as: 
\begin{align}
    P^{\widehat M(\mathcal G; \bm{\theta})}(\mathbf v \mid do(\*x)) 
    = \mathop{\mathbb E}_{P(\mathbf u^c)} \left[ \prod_{V_i \in \mathbf V \setminus \*X} \tilde{\sigma}_{v_i} \right] \approx \frac{1}{m} \sum_{j = 1}^m \prod_{V_i \in \mathbf V \setminus \*X} \tilde{\sigma}_{v_i}, \label{eq:ncm-mc-sampling}
\end{align}

where $\tilde{\sigma}_{v_i} := \begin{cases} \sigma(\phi_i(\pai{V_i}, \mathbf{u}_{V_i}^c; \theta_{V_i})) & v_i = 1 \\ 1 - \sigma(\phi_i(\pai{V_i}, \mathbf{u}_{V_i}^c; \theta_{V_i})) & v_i = 0 \end{cases}$ and $\{\ui{j}^c\}_{j=1}^m$ are samples from $P(\*U^c)$. Here, we assume $\*v$ is consistent with $\*x$ (the values of $X \in \*X$ in $\*v$ match the corresponding ones of $\*x$). Otherwise, $P^{\widehat M(\mathcal G; \bm{\theta})}(\mathbf v \mid do(\*x)) = 0$.
For numerical stability of each $\phi_i(\cdot)$, we work in log-space and use the log-sum-exp trick. 

\begin{wrapfigure}{r}{0.45\textwidth}

\IncMargin{1em}
\vspace{-0.2in}
\begin{algorithm}[H]
    \scriptsize
    \renewcommand{\AlCapSty}[1]{\normalfont\scriptsize{\textbf{#1}}\unskip}

    \DontPrintSemicolon
    \SetKw{notsymbol}{not}
    \SetKwData{ncmdata}{$\widehat{M}$}
    \SetKwData{paramdata}{$\bm{\theta}$}
    \SetKwData{pdata}{$\hat{p}$}
    \SetKwData{qdata}{$\hat{q}$}
    \SetKwData{lossdata}{$\cL$}
    \SetKwFunction{ncmfunc}{NCM}
    \SetKwFunction{estimate}{Estimate}
    \SetKwFunction{consistent}{Consistent}
    \SetKwFunction{sample}{Sample}
    \SetKwInOut{Input}{Input}
    \SetKwInOut{Output}{Output}
    
    \Input{ Data $\{\*v_k\}_{k=1}^{n}$, variables $\*V$, $\*X \subseteq \*V$, $\*x \in \cD_{\*X}$, $\*Y \subseteq \*V$, $\*y \in \cD_{\*Y}$, causal diagram $\cG$, number of Monte Carlo samples $m$, regularization constant $\lambda$, learning rate $\eta$}
    \BlankLine
    $\ncmdata \gets$ \ncmfunc{$\*V, \cG$} \tcp*{from Def.~\ref{def:g-cons-nscm}}
    Initialize parameters $\bm{\theta}_{\min}$ and $\bm{\theta}_{\max}$\;
    \For{$k \gets 1$ \KwTo $n$}{
        \tcp{\estimate from Eq.~\ref{eq:ncm-mc-sampling}}
        $\pdata_{\min} \gets$ \estimate{$\ncmdata(\paramdata_{\min}), \*V, \*v_k, \emptyset, \emptyset, m$}\;
        $\pdata_{\max} \gets$ \estimate{$\ncmdata(\paramdata_{\max}), \*V, \*v_k, \emptyset, \emptyset, m$}\; 
        $\qdata_{\min} \gets 0$\;
        $\qdata_{\max} \gets 0$\;
        \For{$\*v \in \cD_{\*V}$}{
            \If{\consistent{$\*v, \*y$}}{
                \textls[-20]{
                $\qdata_{\min} \gets \qdata_{\min} +$ \estimate{$\ncmdata(\paramdata_{\min}), \*V, \*v, \*X, \*x, m$}\;
                $\qdata_{\max} \gets \qdata_{\max} +$ \estimate{$\ncmdata(\paramdata_{\max}), \*V, \*v, \*X, \*x, m$}\;
                }
            }
        }
        \tcp{\lossdata from Eq.~\ref{eq:ncm-total-loss}}
        $\lossdata_{\min} \gets -\log \pdata_{\min} - \lambda \log(1 - \qdata_{\min})$\;
        $\lossdata_{\max} \gets -\log \pdata_{\max} - \lambda \log \qdata_{\max}$\;
        $\paramdata_{\min} \gets \paramdata_{\min} + \eta \nabla \lossdata_{\min}$\;
        $\paramdata_{\max} \gets \paramdata_{\max} + \eta \nabla \lossdata_{\max}$\;
    }
    \caption{\scriptsize Training Model}
    \label{alg:ncm-learn-pv}
\end{algorithm}
\DecMargin{1em}
\vspace{-0.9in}

\end{wrapfigure}

Alg.~\ref{alg:nscm-solve-id} (lines 2-3)
 requires non-trivial evaluations of expressions like  $\arg \max_{\bm{\theta}} P^{\widehat{M}}(\*y \mid do(\*x))$  while  enforcing $L_1$-consistency.  Whenever only finite samples are available $\{\mathbf v_k\}_{k=1}^n \sim P^*(\mathbf V)$, the parameters of an $L_1$-consistent NCM may be estimated by minimizing data negative log-likelihood:

\begin{align*}
    \bm{\theta} \in &\arg \min_{\bm{\theta}} \mathbb E_{P^*(\mathbf v)}\left[-\log P^{\widehat M(\mathcal G; \bm{\theta})}(\mathbf v)\right] \\
    \approx &\arg \min_{\bm{\theta}} \frac{1}{n} \sum_{k = 1}^n -\log \widehat{ P}_m^{\widehat M(\mathcal G; \bm{\theta})}(\mathbf v_k). \numberthis \label{eq:ncm-probability-loss}
\end{align*}

To simultaneously maximize $P^{\widehat{M}}(\*y \mid do(\*x))$, we subtract a weighted second term $\log\widehat{P}^{\widehat{M}}_m(\*y \mid do(\*x))$, resulting in the objective $\cL(\{\mathbf v_k\}_{k=1}^n)$ equal to
\begin{align*}
\hspace{-0.1in}
     \frac{1}{n} \sum_{k = 1}^n - \log \widehat{ P}_m^{\widehat M}(\mathbf v_k) - \lambda \log\widehat{P}^{\widehat{M}}_m(\*y \mid do(\*x)), \numberthis \label{eq:ncm-total-loss}
\end{align*}
where $\lambda$ is initially set to a high value and decreases during training. 
 To minimize, we instead subtract $\lambda \log(1 - \widehat{P}^{\widehat{M}}_m(\*y \mid do(\*x)))$ from the log-likelihood.

Alg.~\ref{alg:ncm-learn-pv} 
is one possible way of optimizing the parameters $\bm{\theta}$ required in lines 2,3 of Alg.~\ref{alg:nscm-solve-id}.
Eq.~\ref{eq:ncm-total-loss} is amenable to optimization through standard gradient descent tools, e.g., \citep{KingmaB14, LoshchilovH19, LoshchilovH17}. \footnote{Our approach is flexible and may take advantage of these different methods depending on the context. There are a number of alternatives for minimizing the discrepancy between $P^*$ and $P^{\widehat{M}}$, including minimizing divergences, such as maximum mean discrepancy \citep{NIPS2006_e9fb2eda} or kernelized Stein discrepancy \citep{pmlr-v48-liub16}, 
performing variational inference \citep{Blei_2017}, or generative adversarial optimization \citep{NIPS2014_5ca3e9b1}. \label{foot:alternate-ncm}} \footnote{The NCM can be extended to the continuous case by replacing the Gumbel-max trick on $\sigma(\phi_i(\cdot))$ with a model that directly computes a probability density given a data point, e.g., normalizing flow \citep{pmlr-v37-rezende15} or VAE \citep{KingmaW13}. \label{foot:continuous-ncm}} 

One way of understanding Alg.~\ref{alg:nscm-solve-id} 
is as a search within the $\Omega(\cG)$ space for two \ncm{} parameterizations, $\bm{\theta}^*_{\min}$ and $\bm{\theta}^*_{\max}$, that minimizes/maximizes the interventional distribution, respectively. Whenever the optimization ends, we can compare the corresponding $P(\*y \mid do(\*x))$ and determine whether an effect is identifiable. With perfect optimization and unbounded resources, identifiability entails the equality between these two quantities. In practice, we rely on a  hypothesis testing step such as 
\begin{equation}
    \label{eq:gap-test}
|f(\widehat{M}(\bm{\theta}_{\mathrm{max}})) - f(\widehat{M}(\bm{\theta}_{\mathrm{min}}))| < \tau
\end{equation}
for quantity of interest $f$ and a certain threshold $\tau$. This threshold is somewhat similar to a significance level in statistics and can be used to control certain types of errors.  
In our case, the threshold $\tau$ can be determined empirically. For further discussion, see  Appendix~\ref{app:experiments}. 
\vspace{-0.05in}
\section{Experiments} \label{sec:experiments}
\vspace{-0.17in}

We start by evaluating NCMs (following Eq.~\ref{eq:ncm-differentiable}) in their ability to decide whether an effect is identifiable through Alg.~\ref{alg:ncm-learn-pv}. 
Observational data is generated from 8 different SCMs, and their corresponding causal diagrams are shown in Fig.~\ref{fig:id-exp-results} (top part), and Appendix~\ref{app:experiments} provides further details of the parametrizations. Since the NCM does not have access to the true SCM, the causal diagram and generated datasets are passed to the algorithm to decide whether an effect is identifiable. 
The target  effect is $P(Y \mid do(X))$, and the quantity we optimize is the \textit{average treatment effect} (ATE) of $X$ on $Y$, ${ATE}_{\cM}(X, Y)= \bbE_\cM[Y \mid do(X = 1)] - \bbE_\cM[Y \mid do(X = 0)]$. 
Note that if the outcome $Y$ is binary, as  in our examples, $\bbE[Y \mid do(X = x)] = P(Y = 1 | do(X = x))$. The effect is identifiable through do-calculus in the settings represented by Fig.~\ref{fig:id-exp-results} in the left part, and not identifiable in right.

\begin{figure*}
    \begin{center}
    \includegraphics[width=\textwidth,keepaspectratio]{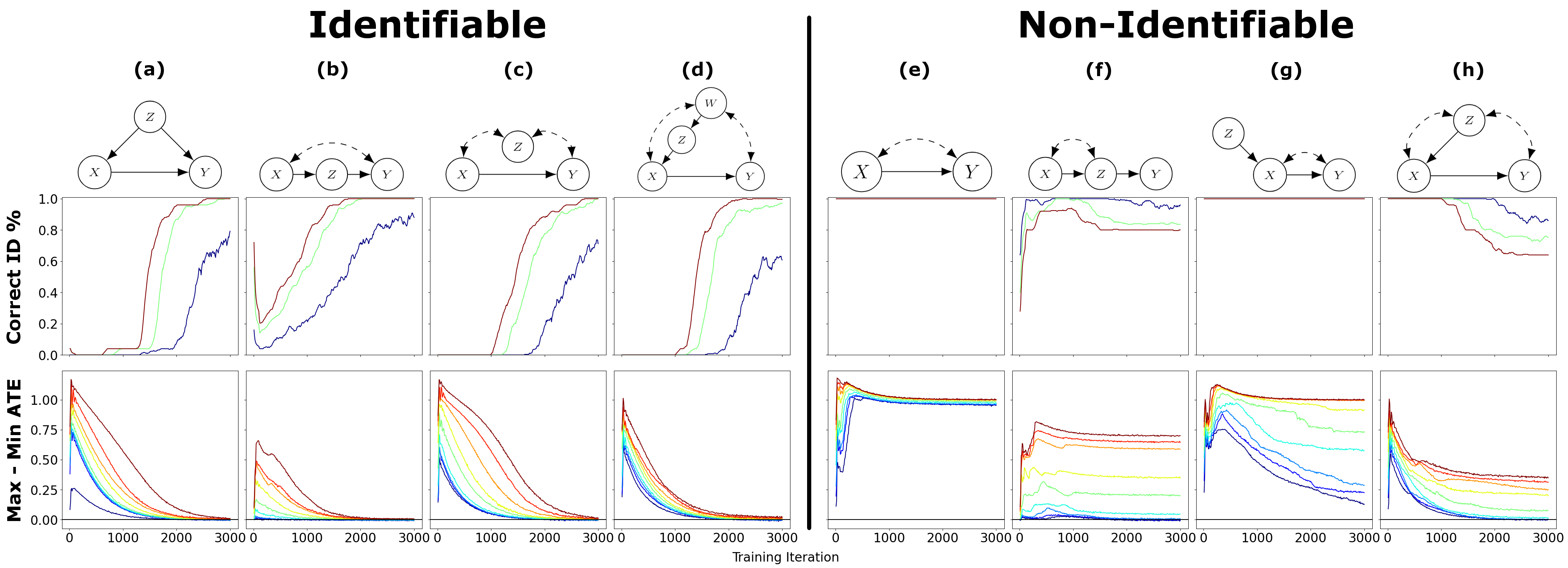}
    \caption{Experimental results on deciding identifiability with NCMs. \textbf{Top}: Graphs from left to right: (ID cases) back-door, front-door, M, napkin; (not ID cases) bow, extended bow, IV, bad M. \textbf{Middle}: Classification accuracy over 3,000 training epochs from running hypothesis test on Eq.~\ref{eq:gap-test} with $\tau = 0.01$ (blue), $0.03$ (green), $0.05$ (red). \textbf{Bottom}: (1, 5, 10, 25, 50, 75, 90, 95, 99)-percentiles for max-min gaps over 3000 training epochs.}
    \label{fig:id-exp-results}
    \end{center}
\end{figure*}

The bottom row of Fig.~\ref{fig:id-exp-results} shows the \emph{max-min gaps}, the l.h.s of Eq.~\ref{eq:gap-test} with $f(\cM) = \textsc{ATE}_{\cM}(X, Y)$, over 3000 training epochs. 
The parameter $\lambda$ is set to $1$ at the beginning, and decreases logarithmically over each epoch until it reaches $0.001$ at the end of training. 
The max-min gaps can be used to classify the quantity as ``ID'' or ``non-ID'' using the hypothesis testing procedure described in Appendix \ref{app:experiments}. The classification accuracies per training epoch are shown in Fig.~\ref{fig:id-exp-results} (middle row). Note that in identifiable settings, the gaps slowly reduce to 0, while the gaps rapidly grow and stay high throughout training in the unidentifiable ones. The classification accuracy for ID cases then gradually increases as training progresses, while accuracy for non-ID cases remain high the entire time (perfect in the bow and IV cases).

\begin{wrapfigure}{r}{0.45\textwidth}
\vspace{-0.2in}
\includegraphics[width=0.45\textwidth,keepaspectratio]{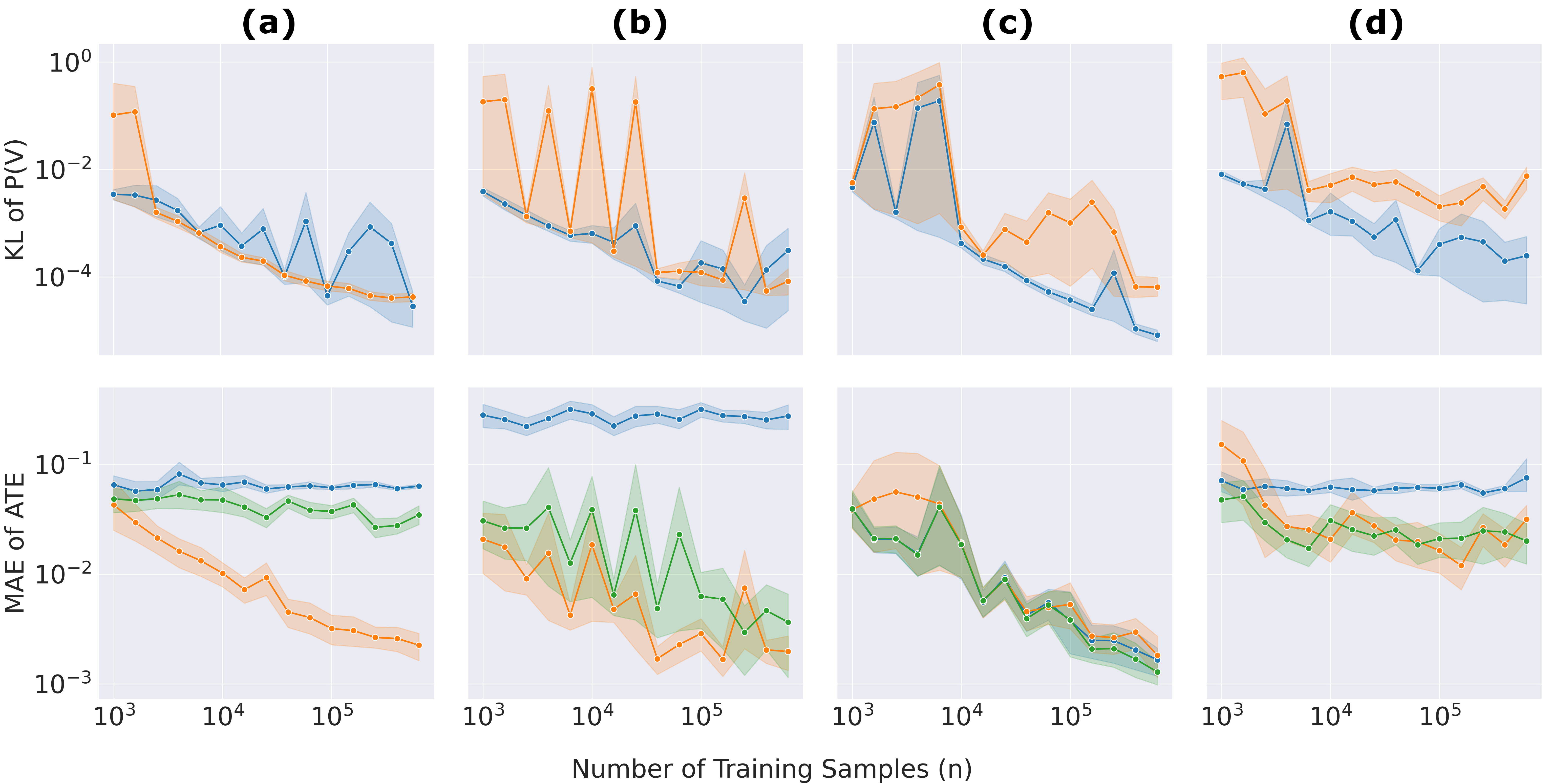}
\vspace{-0.20in}
\caption{NCM estimation results for ID cases. Columns a, b, c, d correspond to the same graphs as a, b, c, d in Fig.~\ref{fig:id-exp-results}. \textbf{Top}: KL divergence of $P(\*V)$ induced by na\"ive model (blue) and NCM (orange) compared to $P^{\cM^*}(\*V)$. \textbf{Bottom}: MAE of ATE of na\"ive model (blue), NCM (orange), and WERM (green). Plots in log-log scale.}
\label{fig:est-samples-results}
\end{wrapfigure}

In the identifiable settings, we also evaluate the performance of the NCM at estimating the correct causal effect, as shown in Fig.~\ref{fig:est-samples-results}. As a generative model, the NCM is capable of generating samples from both $P(\*V)$ and identifiable $L_2$ distributions like $P(Y \mid do(X))$. We compare the NCM to a na\"ive generative model trained via likelihood maximization fitted on $P(\*V)$ without using the inductive bias of the \ncm{}. Since the na\"ive model is not defined to sample from $P(y \mid do(x))$, this shows the implications of arbitrarily choosing $P(y \mid do(x)) = P(y \mid x)$. Both models improve at fitting $P(\*V)$ with more samples, but the na\"ive model fails to learn the correct ATE except in case (c), where $P(y \mid do(x)) = P(y \mid x)$.
 Further, the NCM is competitive with WERM \cite{jung2020werm}, a state-of-the-art estimation method that
  directly targets estimating the causal effect without generating samples.

\vspace{-0.1in}
\section{Conclusions}
\vspace{-0.1in}

In this paper, we introduced neural causal models (NCMs) (Def.~\ref{def:nscm},~\ref{def:gen-ncm}), a special class of SCMs trainable through gradient-based optimization techniques. We showed that despite being as expressive as SCMs (Thm.~\ref{thm:lrepr}), NCMs are unable to perform cross-layer inferences in general (Corol.~\ref{cor:ncht}). 
Disentangling expressivity and learnability, we formalized a new type of inductive bias based on non-parametric, structural properties of the generating SCM, accompanied with a constructive procedure that allows NCMs to represent constraints over the space of interventional distributions akin to causal diagrams (Thm.~\ref{thm:nscm-g-cons}). 
 We showed that NCMs with this bias retain their full expressivity (Thm.~\ref{thm:nscm-g-uat}) but are now empowered to solve canonical tasks in causal inference, including the problems of identification and estimation (Thm.~\ref{thm:nscm-id-equivalence}). 
We grounded these results by providing a  training procedure that is both sound and complete (Alg.~\ref{alg:nscm-solve-id}, \ref{alg:ncm-learn-pv}, Cor.~\ref{thm:nscm-id-correctness}). 
Practically speaking, different neural implementations -- combination of architectures, training algorithms, loss functions -- can leverage the framework results introduced in this work (Appendix \ref{sec:nn-general}). 
We implemented one of such alternatives as a proof of concept, and experimental results support the feasibility of the proposed approach.
After all, we hope the causal-neural framework established in this paper can help develop more principled and robust architectures to empower the next generation of AI systems. 
We expect these systems to combine the best of both worlds by (1) leveraging causal inference capabilities of processing the structural invariances found in nature to construct more explainable and generalizable decision-making procedures, and (2) leveraging deep learning capabilities to scale inferences to handle challenging, high dimensional settings found in practice. 

\section*{Acknowledgements}
We thank Judea Pearl, Richard Zemel, Yotam Alexander, Juan Correa, Sanghack Lee, and Junzhe Zhang for their valuable feedback. Kevin Xia and Elias Bareinboim were supported in part by funding from the NSF, Amazon, JP Morgan, and The Alfred P. Sloan Foundation. Yoshua Bengio was supported in part by funding from CIFAR, NSERC, Samsung, and Microsoft.


\bibliographystyle{apalike}
\bibliography{references}

\clearpage

\appendix
\setcounter{algocf}{2}

\section{Proofs} \label{app:proofs}

In this section, we provide proofs of the statements in the main body of the paper.

\subsection{Proofs of Theorem \ref{thm:lrepr} and Corollary \ref{cor:ncht}}

In addition to Def.~\ref{def:l2-semantics}, defining layers 1 and 2 of the PCH, we also require a definition for layer 3. While Def.~\ref{def:l2-semantics} shows how the SCM valuates observational and interventional distributions, the following definition of layer 3 (\citep[Def.~7]{bareinboim:etal20}) shows how the SCM valuates counterfactual distributions, a family of distributions even more expressive than those from lower layers.
\begin{definition}[\Layer{} 3 Valuation]
\label{def:l3-semantics}
An SCM $\M = \langle \*U, \*V, \mathcal{F}, P(\*U)\rangle$ induces a family of joint distributions over counterfactual events $\*Y_\*x, \ldots, \*Z_{\*w}$, for any $\*Y, \*Z, \dots, \*X, \*W \subseteq \*V$:
\begingroup\abovedisplayskip=0.5em\belowdisplayskip=0pt
{\begin{align}\label{eq:def:l3-semantics}
    P^{\M}(\*{y}_{\*{x}},\dots,\*{z}_{\*{w}}) =
\sum_{\substack{\{\*u\;\mid\;\*{Y}_{\*x}(\*u)=\*{y},\\ \;\;\;\dots,\; \*{Z}_{\*w}(\*u)=\*z\}}}
    P(\*u).
\end{align}}
\hfill $\blacksquare$
\endgroup
\end{definition}

For the expressiveness proofs of this paper, we leverage some of the notation and results from \citep{zhang:bareinboim21}. These results focus on the idea of a canonical form of SCMs, first explored in a special case by \citep{balke:pea94b}. Let $\cM = \langle \*U, \*V, \cF, P(\*U) \rangle$ be any SCM. For each $V \in \*V$, we denote $\cH_V = \{h_V : \cD_{\pai{V}} \rightarrow \cD_V\}$ as the set of all possible functions mapping from the domain of the parents $\pai{V}$ to the domain of $V$. We will order the elements of $\cH_V$ as $h_V^{(1)}, \dots, h_V^{(m_V)}$, where $m_V = |\cH_V|$. Since $\cH_V$ fully exhausts all possible functions, we can partition $\cD_{\Ui{V}}$ into sets $\cD_{\Ui{V}}^{(1)}, \dots, \cD_{\Ui{V}}^{(m_V)}$ such that $\ui{V} \in \cD_{\Ui{V}}^{(r_V)}$ if and only if $f_V(\cdot, \ui{V}) = h_V^{(r_V)}$.

\begin{lemma}[{\citep[Lem.~1]{zhang:bareinboim21}}]
    \label{lem:canon-func}
    For an SCM $\cM = \langle \*U, \*V, \cF, P(\*U) \rangle$, for each $V \in \*V$, function $f_V \in \cF$ can be expressed as
    $$f_V (\pai{V}, \ui{V}) = \sum_{r_V = 1}^{m_V} h_V^{(r_V)}(\pai{V}) \mathbbm{1}\left\{\ui{V} \in \cD_{\Ui{V}}^{(r_V)}\right\}$$
    \hfill $\blacksquare$
\end{lemma}

\begin{definition}[Canonical SCM]
    \label{def:simple-ctm}
    A canonical SCM is an SCM $\cM = \langle \*U, \*V, \cF, P(\*U)\rangle$ such that
    \begin{enumerate}[label=\arabic*.]
        \item $\*U = \{R_{V} : V \in \*V\}$, where $\cD_{R_{V}} = \{1, \dots, m_{V}\}$ (where $m_V = |\{h_V : \cD_{\pai{V}} \rightarrow \cD_V\}|$) for each $V \in \*V$.
        \item For each $V \in \*V$, $f_V \in \cF$ is defined as
        $$f_V(\pai{V}, r_{V}) = h_V^{(r_V)}(\pai{V}).$$
    \end{enumerate}
    \hfill $\blacksquare$
\end{definition}

\begin{lemma}
    \label{lem:scm-to-ctm}
    For any SCM $\cM = \langle \*U, \*V, \cF, P(\*U)\rangle$, there exists a canonical SCM $\cM_{\ctm} = \langle \*U_{\ctm}, \*V, \cF_{\ctm}, P(\*U_{\ctm})\rangle$ such that $\cM_{\ctm}$ is $L_3$-consistent with $\cM$.
    \hfill $\blacksquare$
\end{lemma}

\begin{proof}
    Since $\*U_{\ctm}$ and $\cF_{\ctm}$ are already fixed, we choose $P(\*U_{\ctm})$ to fix our choice of $\cM_{\ctm}$. For $\*r \in \cD_{\*U_{\ctm}}$, we choose
    \begin{align}
        &P^{\cM_{\ctm}}(\*U_{\ctm} = \*r) \\
        &= P^{\cM_{\ctm}}(R_{V_1} = r_{V_1}, \dots, R_{V_n} = r_{V_n}) \\
        &:= P^{\cM} \left(\Ui{V_1} \in \cD_{\Ui{V_1}}^{(r_{V_1})}, \dots, \Ui{V_n} \in \cD_{\Ui{V_n}}^{(r_{V_n})} \right).
        \label{eq:ctm-choice}
    \end{align}
    
    For $\*r \in \cD_{\*U_{\ctm}}$, denote
    \begin{equation*}
        \cD_{\*U}^{(\*r)} = \left\{\*u : \*u \in \cD_{\*U}, \*u_{V} \in \cD_{\*U_V}^{(r_V)} \quad \forall V \in \*V\right\}
    \end{equation*}
    
We now show that $\cM_{\ctm}$ and $\cM$ valuate in the same way any query of the form $P(\bm{\upvarphi})$, where
    $$\bm{\upvarphi} = \bigwedge_{i \in \{1, \dots, k\}} \*Y^i_{\*x_i} = \*y_i$$
    for any $\mathbf{X}_i, \mathbf{Y}_i \subseteq \mathbf{V}$, $\mathbf{Y}_i \neq \emptyset$, $\mathbf{y}_i \in \cD_{\mathbf{Y}_i}$, and positive integer $k$. We say $\cM(\*u) \models \bm{\upvarphi}$ for $\*u \in \cD_{\*U}$ if for all $i \in \{1, \dots, k\}$, $\*Y^i_{\*x_i}(\*u) = \*y_i$. We define this notation similarly for $\cM_{\ctm}$.
    
    Let $\*u^1, \*u^2 \in \cD_{\*U}$ be any two instantiations of $\*U$. If $\*u^1$ and $\*u^2$ come from the same partition $\cD_{\*U}^{(\*r)}$, then we have for all $V \in \*V$,
    \begin{align*}
        &f_V(\pai{V}, \ui{V}^1) \\
        &= \sum_{r'_V = 1}^{m_V} h_V^{(r'_V)}(\pai{V}) \mathbbm{1}\left\{\ui{V}^1 \in \cD_{\Ui{V}}^{(r'_V)}\right\} & \text{ Lem.~\ref{lem:canon-func}}\\
        &= h_V^{(r_V)}(\pai{V}) \\
        &= \sum_{r'_V = 1}^{m_V} h_V^{(r'_V)}(\pai{V}) \mathbbm{1}\left\{\ui{V}^2 \in \cD_{\Ui{V}}^{(r'_V)}\right\} \\
        &= f_V(\pai{V}, \ui{V}^2) & \text{ Lem.~\ref{lem:canon-func}}.
    \end{align*}
    Hence,
    \begin{equation}
        \label{eq:equiv-class-entailment}
        \cM(\*u^1) \models \bm{\upvarphi} \Leftrightarrow \cM(\*u^2) \models \bm{\upvarphi}.
    \end{equation}
    
    Let $\*u \in \cD_{\*U}$ and let $\*r \in \cD_{\*U_{\ctm}}$. Then if $\*u \in \cD_{\*U}^{(\*r)}$, we have for all $V \in \*V$
    \begin{align*}
        &f_V(\pai{V}, \ui{V}) \\
        &= \sum_{r'_V = 1}^{m_V} h_V^{(r'_V)}(\pai{V}) \mathbbm{1}\left\{\ui{V} \in \cD_{\Ui{V}}^{(r'_V)}\right\} & \text{ Lem.~\ref{lem:canon-func}}\\
        &= h_V^{(r_V)}(\pai{V}) \\
        &= f_V^{\ctm}(\pai{V}, \ui{V}) & \text{ Def.~\ref{def:simple-ctm}}.
    \end{align*}
    Hence,
    \begin{equation}
        \label{eq:canonical-entailment}
        \cM(\*u) \models \bm{\upvarphi} \Leftrightarrow \cM_{\ctm}(\*r) \models \bm{\upvarphi}.
    \end{equation}
    
    Then, by the previous statements and $\cM_{\ctm}$'s construction, we have
    \begin{align*}
        P^{\cM}(\bm{\upvarphi}) &= \sum_{\{\*u : \cM(\*u) \models \bm{\upvarphi}\}} P^{\cM}(\*u) \\
        &= \sum_{\{\*r : \cM(\*u) \models \bm{\upvarphi}, \*u \in \cD_{\*U}^{(\*r)}\}} P^{\cM} \left(\*U \in \{\cD_{\*U}^{(\*r)}\} \right) \\
        & \text{ by Eq. \ref{eq:equiv-class-entailment}} \\
        &= \sum_{\{\*r : \cM(\*u) \models \bm{\upvarphi}, \*u \in \cD_{\*U}^{(\*r)}\}} P^{\cM_{\ctm}} \left(\*r\right) \\
        & \text{ by Eq. \ref{eq:ctm-choice}} \\
        &= \sum_{\{\*r : \cM_{\ctm}(\*r) \models \bm{\upvarphi}\}} P^{\cM_{\ctm}} \left(\*r\right) \\
        & \text{ by Eq. \ref{eq:canonical-entailment}} \\
        &= P^{\cM_{\ctm}}(\bm{\upvarphi}).
    \end{align*}
\end{proof}

Lemma 2 shows that the canonical SCM can be used as a representative of equivalence classes of SCMs. In the case where $\cD_{\*V}$ is discrete, the mapping from an SCM to an equivalent canonical model conveniently also remaps $\cD_{\*U}$ to a discrete space. We next show that any canonical SCM can be constructed in the form of an \ncm{}.

\begin{wrapfigure}{r}{0.5\textwidth}
    \centering
    \includegraphics[width=0.4\textwidth,height=0.5\textheight,keepaspectratio]{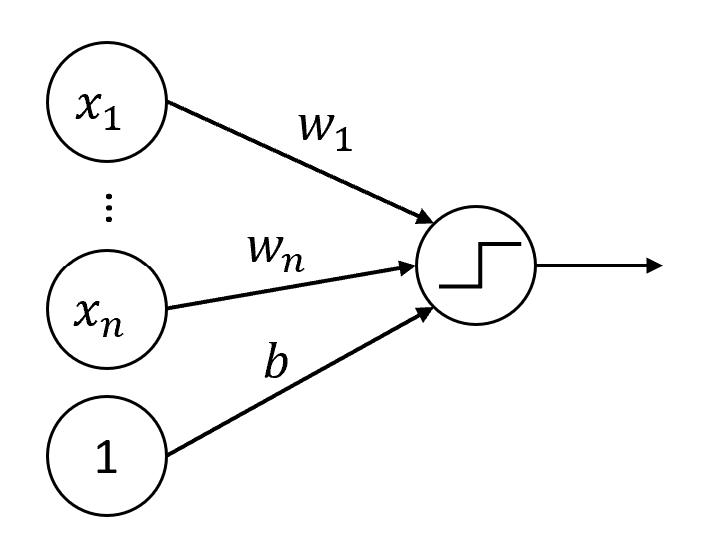}
    \caption{Example diagram of a neural network node from Definition \ref{def:ff-nn-binary-step}. Nodes on the left are inputs, numbers on the edges represent weights, and the weighted sum of the inputs is passed through the binary step activation function.}
    \label{fig:ex-bin-net}
\end{wrapfigure}

We will focus on feedforward neural networks, specifically multi-layer perceptrons (MLPs) with the binary step activation function, even though other types of neural networks could be compatible with the statement proven here (see Appendix \ref{sec:nn-general}).

\begin{definition}[Multi-layer Perceptron]
    \label{def:ff-nn-binary-step}
    A neural network node is a function defined as
    $$\hat{f}(\mathbf{x}; \mathbf{w}, b) = \sigma \left(\sum_i \mathbf{w}_i \mathbf{x}_i + b\right),$$
    where $\mathbf{x}$ is a vector of real-valued inputs, $\mathbf{w}$ and $b$ are the real-valued learned weights and bias respectively, and $\sigma$ is an activation function. For this work, we will often denote $\sigma$ as the binary step function for our activation function:
    \[\sigma(z) = 
    \begin{cases}
        1 & z \geq 0 \\
        0 & z < 0.
    \end{cases}
    \]
    This is simply one choice of activation function which always outputs a binary result. (Figure \ref{fig:ex-bin-net} provides an illustration of such a node.)
    
    A neural network layer of width $k$ is comprised of $k$ neural network nodes with the same input vector, together outputting a $k$-dimensional output:
    $$\hat{f}(\mathbf{x}; \*W, \*b) = \left( \hat{f}_1(\mathbf{x}; \mathbf{w}_1, b_1), \dots, \hat{f}_k(\*x;\mathbf{w}_k, b_k)\right),$$
    where $\*W = \{\*w_1, \dots, \*w_k\}$ and $\*b = \{b_1, \dots, b_k\}$.
    An MLP is defined as a function comprised of several neural network layers $\hat{f}_1, \dots, \hat{f}_{\ell}$, with each layer taking the previous layer's output as its input:
    $$\hat{f}_{\mlp}(\*x) = \hat{f}_{\ell}(\dots \hat{f}_1(\*x; \*W_1, \*b_1) \dots; \*W_{\ell}, \*b_{\ell}).$$
    This means that a neural network is a function that is a composition of the functions of the individual layers, where the input is the input to the first layer, and the output is the output of the last layer.
    \hfill $\blacksquare$
\end{definition}

We will show next three basic lemmas (3-5) that will be used later on to help understand the expressiveness of the networks introduced from Def.~\ref{def:ff-nn-binary-step}.

\begin{lemma}
    \label{lem:binary-uat}
    For any function $f: \mathbf{X} \rightarrow Y$ mapping a set of binary variables to a binary variable, there exists an equivalent MLP $\hat{f}$ using binary step activation functions.
    \hfill $\blacksquare$
    
    \begin{proof}
        We define the following three neural network components:
        \begin{itemize}
            \item Given binary input $x$, with $w = -1$ and $b = 0$, neural network function
            $$\hat{f}_{\textsc{NOT}}(x) = \sigma(-x)$$
            outputs the negation of $x$.
            
            \item Given binary vector input $\mathbf{x}$, with $\mathbf{w} = \mathbf{1}$ and $b = -1$, neural network function
            $$\hat{f}_{\textsc{OR}}(\mathbf{x}) = \sigma \left( \sum_i x_i - 1 \right)$$
            outputs the bitwise-OR of $\mathbf{x}$.
            
            \item Given binary vector input $\mathbf{x}$, with $\mathbf{w} = \mathbf{1}$ and $b = -|\mathbf{x}|$, neural network function
            $$\hat{f}_{\textsc{AND}}(\mathbf{x}) = \sigma \left( \sum_i x_i - |\mathbf{x}| \right)$$
            outputs the bitwise-AND of $\mathbf{x}$.
        \end{itemize}
        Since all functions mapping a set of binary variables to a binary variable can be written in disjunctive normal form (DNF), we can combine these three components to build $\hat{f}$.
    \end{proof}
\end{lemma}

\begin{lemma}
    \label{lem:discrete-to-binary}
    For any function $f: \mathbf{X} \rightarrow Y$ mapping set of variables $\mathbf{X}$ to variable $Y$, all from finite numerical domains, there exists an equivalent MLP $\hat{f}$ using binary step activations.
    \hfill $\blacksquare$
    
    \begin{proof}
        For each value $\mathbf{x} \in \cD_{\mathbf{X}}$, we first aim to assign a unique binary representation $\bin(\mathbf{x})$, which we can use more flexibly due to Lemma \ref{lem:binary-uat}. One simple way to accomplish this is to map the values to a one-hot encoding, a binary vector for which each element corresponds to a unique value in $\cD_{\*X}$.
        
        We will use here a neural function with $w = 1$ and $b = -z$, so we have
        $$\hat{f}_{\geq z}(x) = \sigma(x - z)$$
        which, on input $x$, outputs $1$ if $x \leq z$ or $0$ otherwise. We will also borrow the binary functions from the proof of Lem.~\ref{lem:binary-uat}. 
        
         For each $X_i \in \*X$ and each $x_i \in \cD_{X_i}$, we construct neural network function
        $$\hat{f}_{=x_i}(x) = \hat{f}_{\textsc{AND}} \left(\hat{f}_{\leq x_i}(x), (\forall x'_i < x_i) \hat{f}_{\textsc{NOT}} \left(\hat{f}_{\leq x'_i}(x) \right) \right)$$
        where $x'_i \in \cD_{X_i}$, which, on input $x \in \cD_{X_i}$, outputs $1$ if $x = x_i$ or 0 otherwise.
        
        We can then define for each $\*z \in \cD_{\*X}$
        $$\hat{f}_{=\*z}(\*x) = \hat{f}_{\textsc{AND}} \left( \forall i \hat{f}_{=z_i}(x_i) \right)$$
        which, on input $\*x \in \cD_{\*X}$, outputs 1 if $\*x = \*z$ or 0 otherwise. Here, $x_i$ and $z_i$ denote the $i$th element of $\*x$ and $\*z$ respectively.
        
        We can then define an one-hot binary representation of $\mathbf{x}$, $\bin(\mathbf{x})$, to be a vector of the outputs of $\hat{f}_{=\*z}(\mathbf{x})$
        for all $\*z \in \cD_{\mathbf{X}}$:
        $$\hat{f}_{\textsc{ENC}}(\*x) = \left( \forall (\*z \in \cD_{\*X}) \hat{f}_{=\*z}(\*x) \right)$$
        This representation is a binary vector of length $|\cD_{\mathbf{X}}|$ and is unique for each value of $\mathbf{x} \in \cD_{\mathbf{X}}$ because $\hat{f}_{=\*z}(\mathbf{x}) = 1$
        if and only if $\mathbf{x} = \*z$, so a different bit is 1 for every choice of $\mathbf{x}$.
        
        We can similarly define a binary representation for each $y \in \cD_{Y}$, $\bin(y)$, as a binary vector of length $|\cD_{Y}|$, where each bit corresponds to a value in $\cD_{Y}$. If $y_i \in \cD_{Y}$ is the value that corresponds with the $i$th bit of $\bin(y)$, then $\bin(y)_i = 1$ if and only if $y = y_i$. Now we consider the translation from $\bin(y)$ back into $y$ using neural networks. We can create the neural network function on input $\bin(y)$ with $\mathbf{w} = (y_i: y_i \in \cD_{Y})$ and $b = 0$,
        $$\hat{f}_{\textsc{DEC}}(\bin(y)) = \mathbf{w}^{\intercal} \bin(y),$$
        omitting the binary step activation function $\sigma$. This function simply computes the dot product of $\bin(y)$ with a vector of all of the possible values of $Y$, which results in $y$ since $\bin(y)$ is 0 in every location except for the bit corresponding to $y$.
        
        Combining all of these constructed neural network functions, we can construct a final MLP $\hat{f}$ for mapping $\mathbf{X}$ to $Y$:
        \begin{enumerate}
            \item On input $\mathbf{x} \in \cD_{\mathbf{X}}$, convert $\mathbf{x}$ to $\bin(\mathbf{x})$ using $\hat{f}_{\textsc{ENC}}(\mathbf{x})$.
            
            \item By Lemma \ref{lem:binary-uat}, find some MLP mapping $\bin(\mathbf{x})$ to $\bin(y)$.
            
            \item Finally,  use $\hat{f}_{\textsc{DEC}}$ to convert $\bin(y)$ to $y$.
        \end{enumerate}
        The final MLP $\hat{f}$ is the composition of all of the neural networks used to realize these three steps.        
    \end{proof}
\end{lemma}

Although neural networks as defined in Def.~\ref{def:ff-nn-binary-step} are undefined for non-numerical inputs and outputs, any kind of categorical data can be considered if first converted into a numerical representation.

The above two lemmas show that MLPs can be used to express any function, but we will need another result to incorporate the exogenous sources of randomness. Specifically, we show that MLPs can map $\unif(0, 1)$ noise to any other distribution of variables.

\begin{lemma}[Neural Inverse Probability Integral Transform (Discrete)]
    \label{lem:unif-to-pmf}
    For any probability mass function $P(\mathbf{X})$, there exists an MLP $\hat{f}$ which maps $\unif(0, 1)$ to $P(\mathbf{X})$.
    \hfill $\blacksquare$
    
    \begin{proof}
        Let $\mathbf{x}_1, \mathbf{x}_2, \dots$ be the elements of the support of $P(\mathbf{X})$, ordered arbitrarily. We also define some arbitrary $\mathbf{x}_0$ such that $P(\mathbf{x}_0) = 0$. For each $i \in \{0, 1, 2, \dots\}$, construct neural network function, with $w = 1$ and $b = -\sum_{j=0}^{i} P(\mathbf{x}_j)$
        $$\hat{f}_{\mathbf{x}_i}(u) = \sigma \left(u - \sum_{j=0}^{i} P(\mathbf{x}_j) \right)$$
        which, on input $u$, returns 1 if and only if $u \geq \sum_{j=0}^{i} P(\mathbf{x}_j)$. Note that $\hat{f}_{\mathbf{x}_0}(u)$ is always 1. We then construct a neural network function $\hat{f}_{\textsc{OUT}}$ which, on inputs $(\hat{f}_{\mathbf{x}_0}, \hat{f}_{\*x_1}, \hat{f}_{\*x_2}, \dots)$, outputs one of $\mathbf{x}_1, \*x_2, \dots$. Specifically, it operates as follows:
        \begin{enumerate}
            \item For each $i \in \{1, 2, \dots\}$, if $\hat{f}_{\mathbf{x}_i} = 0$ and $\hat{f}_{\mathbf{x}_{i - 1}} = 1$, then output $\mathbf{x}_i$.
            \item If none hold, output any arbitrary $\mathbf{x}_i$ (this will never happen).
        \end{enumerate}
        By Lemma \ref{lem:discrete-to-binary}, we can construct such a function since all $\hat{f}_{\mathbf{x}_i}$ are binary. Then, let $\hat{g}(u) = \hat{f}_{\textsc{OUT}}(\hat{f}_{\*x_0}(u), \hat{f}_{\*x_1}(u), \hat{f}_{\*x_2}(u), \dots)$. Observe that for $u$ sampled from $\unif(0, 1)$,
        \begin{align*}
            P(\hat{g}(u) = \mathbf{x}_i) &= P \left(\hat{f}_{\textsc{OUT}} \left(\hat{f}_{\*x_0}(u), \hat{f}_{\*x_1}, \hat{f}_{\*x_2}, \dots \right) = \mathbf{x}_i \right) \\
            &= P \left(\hat{f}_{\mathbf{x}_i}(u) = 0 \wedge \hat{f}_{\mathbf{x}_{i-1}}(u) = 1 \right) \\
            &= P \left(u < \sum_{j = 0}^i P(\mathbf{x}_j) \wedge u \geq \sum_{j = 0}^{i - 1} P(\mathbf{x}_j) \right) \\
            &= P \left(\sum_{j = 0}^{i - 1} P(\mathbf{x}_j) \leq u < \sum_{j = 0}^i P(\mathbf{x}_j) \right) \\
            &= \sum_{j = 0}^i P(\mathbf{x}_j) - \sum_{j = 0}^{i - 1} P(\mathbf{x}_j) \\
            &= P(\mathbf{x}_i)
        \end{align*}
        for each $i \in \{1, 2, \dots\}$. Therefore, we see that $\hat{g}$ successfully maps the $\unif(0, 1)$ distribution to $P(\mathbf{X})$.
    \end{proof}
\end{lemma}

We can now combine these neural network results with the canonical SCM results to complete the expressiveness proof for NCMs.

\lrepr*

\begin{proof}
    Lemma \ref{lem:scm-to-ctm} guarantees that there exists a canonical SCM $\cM_{\ctm} = \langle \*U_{\ctm}, \*V, \cF_{\ctm}, P(\*U_{\ctm})\rangle$ that is $L_3$-consistent with $\cM^*$. Hence, to construct $\widehat{M}$, it suffices to show how to construct $\cM_{\ctm}$ using the architecture of an NCM.
    
    Following the structure of Def.~\ref{def:nscm}, we choose $\widehat{\*U} = \{\widehat{U}_{\*V}\}$. For each $V_i \in \*V$, we construct $\hat{f}_{V_i} \in \widehat{\cF}$ using the following components:
    \begin{enumerate}
        \item By Lemma \ref{lem:unif-to-pmf}, construct $\hat{f}_{V_i}^{R}: \cD_{\widehat{U}_{\*V}} \rightarrow \cD_{\*U_{\ctm}}$ such that
        \begin{equation}
            \hat{f}_{V_i}^{R}(\widehat{u}_{\*V}) = \*u_{\ctm},
        \end{equation}
        where
        \begin{equation}
            \label{eq:ncm-match-ctm-pu}
            P^{\widehat{M}} \left(\hat{f}_{V_i}^{R}(\widehat{U}_{\*V}) = \*u_{\ctm}\right) = P^{\cM_{\ctm}}(\*u_{\ctm}).
        \end{equation}
        
        \item By Lemma \ref{lem:discrete-to-binary}, construct $\hat{f}_{V_i}^{H}: \cD_{\Pai{V_i}} \times \cD_{\*U_{\ctm}} \rightarrow \cD_{V_i}$ such that
        \begin{align}
            \hat{f}_{V_i}^{H}(\pai{V_i}, \*u_{\ctm}) &= f_{V_i}^{\ctm}(\pai{V_i}, r_{V_i}) \\
            &= h_{V_i}^{(r_V)}(\pai{V_i})
            \label{eq:ncm-match-ctm-f}
        \end{align}
        where $r_{V_i}$ is the value in $\*u_{\ctm}$ corresponding to $V_i$.
    \end{enumerate}
    Combining these two components leads to MLP 
    \begin{equation}
        \hat{f}_{V_i}(\pai{V_i}, \widehat{u}_{\*V}) = \hat{f}_{V_i}^{H}\left(\pai{V_i}, \hat{f}_{V_i}^{R}(\widehat{u}_{\*V})\right)
    \end{equation}
    Although this does not exactly fit the structure in Def.~\ref{def:ff-nn-binary-step} because $\pai{V}$ is not included as an input into $\hat{f}_{V_i}^R$, this can be altered by simply having $\hat{f}_{V_i}^R$ accepting $\pai{V}$ as an input and outputting it alongside $\ui{\ctm}$ without changing it.
    
    By Eqs.~\ref{eq:ncm-match-ctm-pu} and \ref{eq:ncm-match-ctm-f}, the NCM $\widehat{M}$ is constructed to match $\cM_{\ctm}$ on all outputs. Hence, for any counterfactual query $\bm{\upvarphi}$, we have
    $$\cM_{\ctm} \models \bm{\upvarphi} \Leftrightarrow \widehat{M} \models \bm{\upvarphi}$$
    and therefore
    $$\cM^* \models \bm{\upvarphi} \Leftrightarrow \widehat{M} \models \bm{\upvarphi}.$$
\end{proof}

While Thm.~\ref{thm:lrepr} demonstrates the expressive power of an SCM parameterized by neural networks, we now consider its limitations. Notably, we show in the sequel that \ncms{} suffer from the same consequences implied by the CHT.

\begin{fact}[Causal Hierarchy Theorem (CHT) {\citep[Thm.~1]{bareinboim:etal20}}]
\label{fact:cht}
    Let $\Omega^*$ be the set of all SCMs. We say that Layer $j$ of the causal hierarchy for SCMs \emph{collapses} to Layer $i$ ($i < j$) relative to $\cM^* \in \Omega^*$ if $L_i(\cM^*) = L_i(\cM)$ implies that $L_j(\cM^*) = L_j(\cM)$ for all $\cM \in \Omega^*$. Then, with respect to the Lebesgue measure over (a suitable encoding of $L_3$-equivalence classes of) SCMs, the subset in which Layer $j$ of NCMs collapses to Layer $i$ is measure zero.
    \hfill $\blacksquare$ 
\end{fact}

We prove a similar result for NCMs as a corollary of Fact \ref{fact:cht} and Thm.~{\ref{thm:lrepr}}.

\ncht*

\begin{proof}
    Since all \ncms{} are SCMs, an SCM-collapse with respect to $\cM^*$ also implies an \ncm{}-collapse with respect to $\cM^*$.
    
    If layer $j$ does not SCM-collapse to layer $i$ with respect to $\cM^*$, then there exists an SCM $\cM$ such that $L_i(\cM^*) = L_i(\cM)$ but $L_j(\cM^*) \neq L_j(\cM)$. By Thm.~\ref{thm:lrepr}, this implies that there exists an \ncm{} $\widehat{\cM}$ such that $L_i(\cM^*) = L_i(\widehat{\cM})$ but $L_j(\cM^*) \neq L_j(\widehat{\cM})$, which means that layer $j$ also does not \ncm{}-collapse to layer $i$.
    
    These two statements imply that the set of SCMs that undergo some form of SCM-collapse is equivalent to the set of SCMs that undergo some form of \ncm{}-collapse. Therefore, the result from Fact \ref{fact:cht} must also hold for \ncms{}.
\end{proof}

\subsection{Proof of Theorem \ref{thm:nscm-g-cons}}

The results proven in this section involve the incorporation of structural constraints, as introduced through the graphical treatment provided in \citep{pearl:95a}, and made it explicit and generalized for models with latent variables in \citep[Sec.1.~4]{bareinboim:etal20}. For convenience, we list the basic definitions below, but refer the readers to the references for more detailed explanations and further examples. 

\begin{definition}[Causal Diagram {\citep[Def.~13]{bareinboim:etal20}}]
    \label{def:scm-cg}
    Consider an SCM $\cM = \langle \*U, \*V, \cF, P(\*U)\rangle$. We construct a graph $\cG$ using $\cM$ as follows:
    \begin{enumerate}[label=(\arabic*)]
        \item add a vertex for every variable in $\*V$,
        \item add a directed edge ($V_j \rightarrow V_i$) for every $V_i, V_j \in \*V$ if $V_j$ appears as an argument of $f_{V_i} \in \cF$,
        \item add a bidirected edge ($V_j \xdasharrow[<->]{} V_i$) for every $V_i, V_j \in \*V$ if the corresponding $\Ui{V_i}, \Ui{V_j} \subseteq \*U$ are not independent or if $f_{V_i}$ and $f_{V_j}$ share some $U \in \*U$ as an argument.
    \end{enumerate}
    We refer to $\cG$ as the causal diagram induced by $\cM$ (or ``causal diagram of $\cM$'' for short).
    \hfill $\blacksquare$
\end{definition}

\begin{definition}[Confounded Component {\citep[Def.~14]{bareinboim:etal20}}]
    Let $G$ be a causal diagram. Let $\{\*C_1, \*C_2, \dots, \*C_k\}$ be a partition over the set of variables $\*V$, where $\*C_i$ is said to be a confounded component (C-component for short) of $G$ if for every $V_a, V_b \in \*C_i$, there exists a path made entirely of bidirected edges between $V_a$ and $V_b$ in $G$, and $\*C_i$ is maximal. We denote $\*C(V_a)$ as the C-component containing $V_a$.
    \hfill $\blacksquare$
\end{definition}

\begin{definition}[Semi-Markov Relative {\citep[Def.~15]{bareinboim:etal20}}]
    A distribution $P$ is said to be \emph{semi-Markov relative} to a graph $G$ if for any topological order $<$ of $G$ through its directed edges, $P$ factorizes as
    \begin{equation}
        P(\*v) = \prod_{V_i \in \*V} P(v_i \mid \pai{V_i}^+),
    \end{equation}
    where $\Pai{V_i}^+ = \Pai{}(\{V \in \*C(V_i) : V \leq V_i\})$, with $\leq$ referring to the topological ordering.
    \hfill $\blacksquare$
\end{definition}

\begin{definition}[Causal Bayesian Network (CBN) {\citep[Def.~16]{bareinboim:etal20}}]
    \label{def:cbn}
    Given observed variables $\mathbf{V}$, let $\mathbf{P}^*$ be the collection of all interventional distributions $P(\mathbf{V} \mid do(\mathbf{x}))$, $\mathbf{X} \subseteq \mathbf{V}$, $\mathbf{x} \in \cD_{\mathbf{X}}$. A causal diagram $\cG$ is a Causal Bayesian Network for $\mathbf{P}^*$ if for every intervention $do(\mathbf{X} = \mathbf{x})$ and every topological ordering $<$ of $\cG_{\overline{X}}$ through its directed edges,
    
    \begin{enumerate}[label=(\roman*)]
        \item $P(\mathbf{V} \mid do(\mathbf{X} = \mathbf{x}))$ is semi-Markov relative to $\cG_{\overline{\mathbf{X}}}$.
        
        \item For every $V_i \in \mathbf{V} \setminus \mathbf{X}$, $\mathbf{W} \subseteq \mathbf{V} \setminus (\Pai{i}^{\mathbf{X}+} \cup \mathbf{X} \cup \{V_i\})$:
        $$P(v_i \mid do(\mathbf{x}), \pai{i}^{\mathbf{x}+}, do(\mathbf{w})) = P(v_i \mid do(\mathbf{x}), \pai{i}^{\mathbf{x}+})$$,
        
        \item For every $V_i \in \mathbf{V} \setminus \mathbf{X}$, let $\Pai{i}^{\mathbf{X}+}$ be partitioned into two sets of confounded and unconfounded parents, $\Pai{i}^c$ and $\Pai{i}^u$ in $\cG_{\overline{\mathbf{X}}}$. Then
        \begin{align*}
            & P(v_i \mid do(\mathbf{x}), \pai{i}^c, do(\pai{i}^u)) \\
            &= P(v_i \mid do(\mathbf{x}), \pai{i}^c, \pai{i}^u)
        \end{align*}
    \end{enumerate}
    
    Here, $\Pai{V_i}^{\*x+} = \Pai{}(\{V \in \*C_{\overline{\*X}}(V_i) : V \leq V_i\})$, with $\*C_{\overline{\*X}}$ referring to the corresponding C-component in $\cG_{\overline{\*X}}$ and $\leq$ referring to the topological ordering.
    \hfill $\blacksquare$
\end{definition}

In fact, for any SCM $\cM$, its induced causal diagram and interventional distributions form a CBN. This means that the diagram encodes the qualitative constraints induced over the space of interventional distributions, despite the specific values that these distributions attain and the $\cF$ and $P(\*U)$ of $\cM$.

\begin{fact}[SCM-CBN $L_2$ connection {\citep[Thm.~4]{bareinboim:etal20}}]
    \label{fact:scm-cbn-l2}
    The causal diagram $\cG$ induced by SCM $\cM$ is a CBN for $L_2(\cM)$.
    \hfill $\blacksquare$
\end{fact}

We can now show that, indeed, all of the CBN constraints implied by a causal diagram $\cG$ are encoded in a $\cG$-constrained \ncm{} constructed via Def.~\ref{def:g-cons-nscm}.

\begin{lemma}
    \label{lem:nscm-g-match}
    Let $\widehat{M}(\bm{\theta}) = \langle \widehat{\*U}, \*V, \widehat{\cF}, \widehat{P}(\widehat{\*U})\rangle$ be a $\cG$-constrained \ncm{}. Let $\widehat{\cG}$ be the causal diagram induced by $\widehat{M}$. Then $\widehat{\cG} = \cG$.
    \hfill $\blacksquare$
    \begin{proof}
        Considering Def. \ref{def:scm-cg} in the context of $\widehat{\cG}$'s construction, note that by step 1 all of the vertices match, simply having one for each variable in $\*V$.
        
        Step 2 adds a directed edge from $V_i$ to $V_j$ if $\hat{f}_{V_i}$ has $V_j$ as an argument. By step 2 of Def. \ref{def:g-cons-nscm}, $\pai{V_j}$ will contain $V_i$ if and only if there was a directed edge from $V_i$ to $V_j$ in $\cG$. This implies that $\hat{f}_{V_j}$ will contain $V_i$ as an argument if and only if there was a directed edge from $V_i$ to $V_j$ in $\cG$, so the directed edges must also match in $\widehat{\cG}$.
        
        Finally, step 3 of Def. \ref{def:scm-cg} states that a bidirected edge between $V_i$ and $V_j$ is added to $\widehat{\cG}$ when $\hat{f}_{V_i}$ and $\hat{f}_{V_j}$ share some $\widehat{U} \in \widehat{\*U}$ as an argument or have arguments from $\widehat{\*U}$ that are not independent. Def. \ref{def:nscm} ensures that all variables in $\widehat{\*U}$ are independent, so a shared $\widehat{U} \in \widehat{\*U}$ between functions in $\widehat{\cF}$ is the only way a bidirected edge would be generated in $\widehat{\cG}$. Step 1 of Def. \ref{def:g-cons-nscm} constructs $\widehat{U}$ such that it contains some $\widehat{U}_{\*C}$ if and only if $\*C$ is a $C^2$-component in $\cG$. If $V_i, V_j \in \*V$ are connected by a bidirected edge in $\cG$, then there must exist some $C^2$-component $\*C$ in $\cG$ such that $V_i, V_j \in \*C^*$, so there must exist $\widehat{U}_{\*C} \in \widehat{\*U}$. Hence, since $\widehat{U}_{\*C}$ is shared by both $\hat{f}_{V_i}$ and $\hat{f}_{V_j}$, the corresponding bidirected edge in $\cG$ must also match in $\widehat{\cG}$.
        
        Therefore, since all vertices and edges match between $\cG$ and $\widehat{\cG}$, we can conclude that $\cG = \widehat{\cG}$.
    \end{proof}
\end{lemma}

\nscmgcons*

\begin{proof}
    This follows directly from Lemma \ref{lem:nscm-g-match} and Fact \ref{fact:scm-cbn-l2}.
\end{proof}

\subsection{Proof of Theorem \ref{thm:nscm-g-uat}}

For this proof, we expand on the technical results provided in \citet{zhang:bareinboim21b}. The paper works with a subclass of SCMs known as \emph{discrete SCMs} and proves strong results about its expressiveness. Similar to this paper, we will assume that in any SCM, the variables in $\mathbf{U}$ are independent, and unobserved confounding is modeled by sharing the same variable from $\mathbf{U}$ in the functions of multiple variables in $\mathbf{V}$.

\begin{definition}[Discrete SCM {\citep[Def.~2]{zhang:bareinboim21b}}]
    An SCM $\cM = \langle \*U, \*V, \cF, P(\*U) \rangle$ is said to be a discrete SCM if $\cD_U$ is discrete for all $U \in \*U$ and $\cD_V$ is both discrete and finite for all $V \in \*V$.
    \hfill $\blacksquare$
\end{definition}

\begin{fact}[{\citep[Thm.~1]{zhang:bareinboim21b}}]
    \label{fact:scm-to-discrete}
    Let $\Omega^*$ be the set of all SCMs and $\Omega'$ the set of discrete SCMs. For any SCM $\cM^* \in \Omega^*$ inducing causal diagram $\cG$, there exists a discrete SCM $\cM' \in \Omega'$ with finite $|\cD_{U}|$ for all $U \in \*U$ such that $\cM'$ is $\cG$-consistent and is $L_3$-consistent with $\cM^*$.
    \hfill $\blacksquare$
\end{fact}

We can now combine these results to achieve an expressiveness result on the counterfactual level for NCMs.

\lgrepr*

\begin{proof}
    By Fact \ref{fact:scm-to-discrete}, there must exist a discrete SCM $\cM' = \langle \*U' , \*V, \cF', P(\*U') \rangle$ such that $\cM'$ is $\cG$-consistent, is $L_3$-consistent with $\cM^*$ (implying $L_2$-consistency), and $|\cD_U|$ is finite for all $U \in \*U$. Let $\widehat{M} = \langle \widehat{\*U}, \*V, \widehat{\cF}, P(\widehat{\*U})\rangle$ be a $\cG$-NCM such that each $\hat{f}_{V} \in \widehat{\cF}$ is an MLP composed of smaller MLPs as defined next.
    
    For $U \in \*U'$, let $\*V_{U} \subseteq \*V$ denote the set of endogenous variables such that for all $V \in \*V_{U}$, $\hat{f}_V$ takes $U$ as an argument. Let $\bbC = \{\*C_1, \dots, \*C_k\}$, with $\*C_1, \dots \*C_k \subseteq \*V$ denote the set of $C^2$-components of $\cG$. Let $\*U'_{\*C_1}, \dots, \*U'_{\*C_k}$ be a partition over $\*U'$ such that if $U \in \*U'_{\*C_i}$, then $\*V_{U} \subseteq \*C_i$. If for any $U \in \*U'$, there exist multiple components such that this applies, one can be chosen arbitrarily. For each $V$, let $\bbC(V) \subseteq \bbC$ denote the set of components that contain $V$. Then we note that
    $$\*U'_{V} \subseteq \bigcup_{\*C \in \bbC(V)}\*U'_{\*C},$$
    where $\*U'_{V} \subseteq \*U'$ denotes the exogenous parents of $V$ in $\cM'$.
    
    By construction of the $\cG$-NCM in Def.~\ref{def:g-cons-nscm}, $\widehat{\*U}$ consists of a $\unif(0, 1)$ random variable $\widehat{U}_{\*C_i}$ for each $C^2$-component $\*C_i$. By Lemma \ref{lem:unif-to-pmf}, we can construct MLP $\hat{f}^{(\*U')}_{\*C_i}$ mapping $\widehat{U}_{\*C_i}$ to $\*U'_{\*C_i}$ for each $\*C_i$. Then by Lemma \ref{lem:discrete-to-binary}, we can construct MLP $\hat{f}_{V}^{(\cF')}$ to map $\*U'_{V}$ and $\Pai{V}$ to $V$, matching $f'_V$. Combining these two results, $\hat{f}_V(\pai{V}, \widehat{\*u}_V)$ is defined as $\hat{f}_V^{(\cF')}(\pai{V}, \hat{f}^{(\*U')}_{{\*C_{i_1}}}(\widehat{u}_{\*C_{i_1}}), \dots , \hat{f}^{(\*U')}_{\*C_{i_{\ell}}}(\widehat{u}_{\*C_{i_{\ell}}}))$, where $\bbC(V) = \{\*C_{i_1}, \dots, \*C_{i_{\ell}}\}$. \footnote{There are a few subtleties here to align with Def.~\ref{def:ff-nn-binary-step}. First, although $\bigcup_{\*C \in \bbC(V)} \*U'_{\*C}$ may contain elements of $\*U'$ not in $\*U'_V$, $\hat{f}_{V}^{(\cF')}$ can be constructed to accept them as input and not use them (weight of 0, identity activation function). Secondly, although Def.~\ref{def:ff-nn-binary-step} does not allow additional inputs inbetween layers, $\pai{V}$ can simply be provided as an input to $\hat{f}^{(\*U')}_{\*C}$ and passed forward to the next layer without modification (weight of 1, identity activation function). Thirdly, while Def.~\ref{def:ff-nn-binary-step} is not defined to have multiple nested MLPs in the same layer, the same result can be achieved by nesting them iteratively and passing the relevant outputs through the nested layers without modification. The presentation in the proof is made for the sake of clarity.}
    
    Altogether, $\widehat{\cF}^{(\*U')} = \{\hat{f}^{(\*U')}_{\*C} : \*C \in \bbC\}$ collectively forms a neural mapping from $\widehat{\*U}$ to $\*U'$, and $\widehat{\cF}^{(\cF')} = \{\hat{f}^{(\cF')}_V : V \in \*V\}$ collectively forms a neural mapping from $\*U'$ to $\*V$. We use the notation $\cF_{\*x}(\*U) \models \*y$ to be equivalent to the statement $\*Y_{\*x}(\*U) = \*y$, which is the random event (w.r.t. $\*U$) that the variables of $\*Y$ under intervention $\*X = \*x$ takes value $\*y$.
    
    Then, for any interventional query $Q = P(\*y \mid do(\*x)) = P \left(\*Y_{\*x} = \*y \right)$, we have
    \begin{align*}
        P^{\widehat{M}}(\*y \mid do(\*x)) &= P\left( \widehat{\*Y}_{\*x}(\widehat{\*U}) = \*y \right) \\
        &= P\left(\widehat{\cF}_{\*x}(\widehat{\*U}) \models \*y \right) \\
        &= P\left(\widehat{\cF}_{\*x}^{(\cF')} \left( \widehat{\cF}^{(\*U')}(\widehat{\*U}) \right) \models \*y \right) \\
        &= P\left(\widehat{\cF}_{\*x}^{(\cF')} \left( \*U' \right) \models \*y \right) \\
        &= P\left(\cF'_{\*x} \left( \*U' \right) \models \*y \right) \\
        &= P\left(\*Y'_{\*x} (\*U') = \*y \right) \\
        &= P^{\cM'}(\*y \mid do(\*x)) \\
        &= P^{\cM^*}(\*y \mid do(\*x)).
    \end{align*}
    
    Hence, we have shown that for any SCM $\cM^*$ inducing causal diagram $\cG$, we can construct a $\cG$-NCM that matches $\cM^*$ on the second layer.
\end{proof}

\subsection{Proofs of Theorem \ref{thm:nscm-id-equivalence} and Corollaries \ref{thm:dual-graph-id} and \ref{thm:nscm-id-correctness}}

\idequivalence*

\begin{proof}
    Since $\Omega \subset \Omega^*$, it must be the case that identifiability in $\cG$ and $P(\*V)$ implies identifiability in $\Omega(\cG)$ and $P(\*V)$. Identifiability in $\cG$ and $P(\*V)$ implies that all pairs of SCMs that match in $\cG$ and $P(\*V)$ will also match in $Q$. This means that for any SCM $\cM^*$ that induces $P(\*V)$ and $\cG$, we must have $P^{\cM_1}(\*y \mid do(\*x)) = P^{\cM_2}(\*y \mid do(\*x))$ for all pairs of SCMs $\cM_1, \cM_2 \in \Omega^*$ that induce $\cG$ such that $L_1(\cM_1) = L_1(\cM_2) = P(\*V)$. Given that all NCMs are SCMs, and $\cG$-constrained NCMs are $\cG$-consistent by Thm.~\ref{thm:nscm-g-cons}, the set of all $\cG$-consistent SCMs includes the set of all $\cG$-constrained NCMs. Hence, this match in $Q$ should also hold if $\cM_1$ and $\cM_2$ are NCMs.
    
    If $Q$ is not identifiable from $\cG$ and $P(\*V)$, there must exist $\cM_1, \cM_2 \in \Omega^*$ such that $\cM_1$ and $\cM_2$ both induce $\cG$, $L_1(\cM_1) = L_1(\cM_2) = P(\*V) > 0$, but $P^{\cM_1}(\*y \mid do(\*x)) \neq P^{\cM_2}(\*y \mid do(\*x))$. Theorem \ref{thm:nscm-g-uat} states that there must exist NCMs $\widehat{M}_1, \widehat{M}_2 \in \Omega(\cG)$ such that $L_2(\widehat{\cM_1}) = L_2(\cM_1)$ and $L_2(\widehat{\cM_2}) = L_2(\cM_2)$. This implies that, for any SCM $\cM^*$ inducing $\cG$ and observational distribution $P(\*V)$, there exist these two $\cG$-constrained NCMs that are $L_1$-consistent with $\cM^*$ (since they are $L_2$-consistent with $\cM_1$ and $\cM_2$ respectively), yet they do not match in $Q$ (i.e. $P^{\widehat{M}_1}(\*y \mid do(\*x)) \neq P^{\widehat{M}_2}(\*y \mid do(\*x))$). In other words, if $Q$ is not identifiable from $\cG$ and $P(\*V)$, then it is also not identifiable from $\Omega(\cG)$ and $P(\*V)$.
\end{proof}

\dualid*

\begin{proof}
    By Thm.~\ref{thm:nscm-g-uat}, there must exist a $\cG$-constrained NCM $\widehat{M}' \in \Omega(\cG)$ such that $\widehat{M}'$ is $L_2$-consistent with $\cM^*$, implying that $\widehat{M}'$ is $L_1$-consistent with $\cM^*$ and $P^{\widehat{M}'}(\*y \mid do(\*x)) = P^{\cM^*}(\*y \mid do(\*x))$. By Thm.~\ref{thm:nscm-id-equivalence}, since $Q$ is identifiable from $\cG$ and $P(\*V)$, it must be also be identifiable from $\Omega(\cG)$ and $P(\*V)$. This means that all $\cG$-constrained NCMs that agree on $P(\*V)$ must also agree on $Q$ with each other, so it must be the case that $P^{\widehat{M}}(\*y \mid do(\*x)) = P^{\widehat{M}'}(\*y \mid do(\*x))$ for any other $\cG$-constrained NCM $\widehat{M} \in \Omega(\cG)$ such that $L_1(\widehat{M}) = P(\*V)$. This means that $P^{\widehat{M}}(\*y \mid do(\*x)) = P^{\cM^*}(\*y \mid do(\*x))$ for arbitrary $\cG$-constrained NCMs $\widehat{M}$ that match $P(\*V)$. Finally, $P^{\widehat{M}}(\*y \mid do(\*x))$ can be computed from $\widehat{M}$ using Def.~\ref{def:l2-semantics}.
\end{proof}

\completeness*

\begin{proof}
    Theorem \ref{thm:nscm-id-equivalence} states that $Q$ must be identifiable from $\cG$ and $P^*(\*V)$ if and only if for all pairs of $\cG$-consistent NCMs, $\widehat{M}_1, \widehat{M}_2 \in \Omega(\cG)$ with $L_1(\widehat{M}_1) = L_1(\widehat{M}_2) = P^*(\*V) > 0$, $P^{\widehat{M}_1}(\*y \mid do(\*x)) = P^{\widehat{M}_2}(\*y \mid do(\*x))$. This holds if and only if $P^{\widehat{M}(\bm{\theta}^*_{\min})}(\*y \mid do(\*x)) = P^{\widehat{M}(\bm{\theta}^*_{\max})}(\*y \mid do(\*x))$. If they are not equal, then $\widehat{M}(\bm{\theta}^*_{\min})$ and $\widehat{M}(\bm{\theta}^*_{\max})$ are a counterexample of two NCMs that do not match for $Q$. Otherwise, if they are equal, then all other NCMs must also induce the same answer for $Q$. A result for $Q$ that is less than $P^{\widehat{M}(\bm{\theta}^*_{\min})}(\*y \mid do(\*x))$ or greater than $P^{\widehat{M}(\bm{\theta}^*_{\max})}(\*y \mid do(\*x))$ would contradict the optimality of $\bm{\theta}^*_{\min}$ and $\bm{\theta}^*_{\max}$.
    
    If $Q$ is identifiable, then Corollary \ref{thm:dual-graph-id} guarantees that any NCM that induces $P^*(\*V)$ and $\cG$ will induce the correct $P^{\cM^*}(\*y \mid do(\*x))$.
\end{proof}

\subsection{Proof of Corollary \ref{thm:markovid}}

In our discussion of the identification problem in Sec.~\ref{sec:neural-id}, we stated Corol.~\ref{thm:markovid} as the solution to a special class of models known as Markovian. In SCMs, Markovianity implies that all variables in $\*U$ are independent and not shared between functions. Correspondingly, this means that no variable in $\widehat{\*U}$ of an NCM can be associated with more than one function. In the causal diagram, this implies that there are no bidirected edges.

We emphasize that Markovianity is a strong constraint in many settings, and the following corollary from \citep{bareinboim:etal20} illustrates that identification in the Markovian setting is quite trivial.

\begin{fact}[{\citep[Corol.~2]{bareinboim:etal20}}]
    \label{fact:adjustment-id}
    In Markovian models (i.e., models without unobserved confounding), for any treatment $\*X$ and outcome $\*Y$, the interventional distribution $P(\*Y \mid do(\*x))$ is always identifiable and given by the expression
    \begin{equation}
        \label{eq:adjustment}
        P(\*Y \mid do(\*x)) = \sum_{\*z} P(\*Y \mid \*x, \*z) P(\*z),
    \end{equation}
    where $\*Z$ is the set of all variables not affected by the action $\*X$ (non-descendants of $\*X$).
    \hfill $\blacksquare$
\end{fact}

In other words, every query in a Markovian setting can be identified via Eq.~\ref{eq:adjustment}, also known as the backdoor adjustment formula. Naturally, this result extends to identifiability via mutilation in \ncm{}s due to the connection between neural identification and graph identification.

\markovid*

\begin{proof}
    Fact \ref{fact:adjustment-id} shows that any query $P(\*y \mid do(\*x))$ is identifiable from $\cG$ and $P(\*v)$ in Markovian settings. Hence, by Thm.~\ref{thm:nscm-id-equivalence}, it must also be identifiable from $\Omega(\cG)$ and $P(\*v)$. Moreover, Corol.~\ref{thm:dual-graph-id} states that the effect can be computed via the mutilation procedure on a proxy $\cG$-constrained NCM $\widehat{M}$ with $L_1(\widehat{M}) = P(\*v)$.
\end{proof}

Without the Markovianity assumption, the identification problem becomes significantly more challenging. As we show in Example \ref{ex:ncht} from Appendix \ref{app:examples-expr}, a query can be non-ID even in a two variable case whenever unobserved confounding cannot be ruled out.

\subsection{Derivation of Results in Section~\ref{sec:ncm-fitting}}

Recall the description of our choice of NCM architecture from Eq.~\ref{eq:ncm-differentiable}:

\begin{equation}
    \begin{cases}
        \mathbf V &:= \mathbf V, \hspace{+0.05in}
        \widehat{\*U} := \{ U_{\mathbf C}: \mathbf C \in C^2(\mathcal G) \} \cup \{ G_{V_i} : V_i \in \mathbf V \}, \\
        \widehat{\mathcal F} &:= \left\{ f_{V_i} := \arg \max_{j \in \{0, 1\}} g_{j, V_i} + \begin{cases}
            \log \sigma(\phi_{V_i}(\pai{V_i}, \ui{V_i}^c; \theta_{V_i})) & j = 1 \\
            \log (1 - \sigma(\phi_{V_i}(\pai{V_i}, \ui{V_i}^c; \theta_{V_i}))) & j = 0
        \end{cases} \right\}, \\
        P(\widehat{\mathbf U}) &:= \{ U_{\mathbf C} \sim \mathrm{Unif}(0, 1) : U_{\mathbf C} \in \mathbf U \} \; \cup \\
         & \hspace{0.17in}  \{ G_{j, V_i} \sim \mathrm{Gumbel}(0, 1) : V_i \in \mathbf V, j \in \{0, 1\} \}, 
    \end{cases} \tag{\ref{eq:ncm-differentiable}}
\end{equation}

Using this architecture, we show the derivation of Eq.~\ref{eq:ncm-mc-sampling}, starting with sampling from $P(\*v)$. Let $\mathbf U^c$ and $\mathbf G$ denote the latent $C^2$-component variables and Gumbel random variables \cite{gumbel1954statistical}, respectively. The formulation in Eq.~\ref{eq:ncm-differentiable} allows us to compute the probability mass of a datapoint $\mathbf v$  as:
\begin{align*}
    &P^{\widehat M(\mathcal G; \bm{\theta})}(\mathbf v) \\
    &= \mathop{\mathbb E}_{P(\mathbf u)} [ \mathbbm 1[\mathcal F(\mathbf u) = \mathbf v]] \\
    &= \mathop{\mathbb E}_{P(\mathbf u^c) P(\mathbf g)} \left[ \prod_{V_i \in \mathbf V} \mathbbm 1 \left[f_{V_i} (\pai{V_i}, \mathbf u_{V_i}) = v_i \right] \right] \\
    &= \mathop{\mathbb E}_{P(\mathbf u^c)} \left[ \prod_{V_i \in \mathbf V} \tilde{\sigma}_{v_i}(\phi_i(\pai{V_i}, \mathbf{u}_{V_i}^c; \theta_{V_i})) \right], \numberthis
\end{align*}
where $\tilde{\sigma}_{v_i}(x) := \begin{cases} \sigma(x) & v_i = 1 \\ 1 - \sigma(x) & v_i = 0 \end{cases}$. We can then derive a Monte Carlo estimator given samples $\{\ui{j}^c\}_{j=1}^m \sim P(\mathbf U^c)$:
\begin{align}
    \hat P_m^{\widehat M(\mathcal G; \bm{\theta})}(\mathbf v) = \frac{1}{m} \sum_{j = 1}^m \prod_{V_i \in \mathbf V} \tilde{\sigma}_{v_i}(\phi_{V_i}(\pai{V_i}, \ui{j, V_i}^c; \theta_{V_i})).
\end{align}

One may similarly estimate the interventional query, $P(\mathbf y | do(\mathbf x))$, where $\mathbf y, \mathbf x$ are the values of the variable sets $\mathbf Y, \mathbf X \subseteq \mathbf V$. We first compute an estimable expression for $P^{\widehat M}(\mathbf y | do(\mathbf x))$:

\begin{align*}
    &P^{\widehat M(\cG, \bm{\theta})}(\mathbf y | do(\mathbf x)) \\
    &= \sum_{\mathbf V \setminus (\mathbf Y \cup \mathbf X)} P^{\widehat M}(\mathbf v | do(\mathbf x)) \\
    &= \mathop{\mathbb E}_{P(\mathbf u)}\left[\sum_{\mathbf V \setminus (\mathbf Y \cup \mathbf X)} \mathbbm 1\left[ \mathcal F_{\mathbf x}(\mathbf u) = \mathbf v \right] \right] \\
    &= \mathop{\mathbb E}_{P(\mathbf u^c) P(\mathbf g)}\left[\sum_{\mathbf V \setminus (\mathbf Y \cup \mathbf X)} \prod_{V_i \in \mathbf V \setminus \mathbf X} \mathbbm 1 \left[ f_{V_i}(\pai{V_i}, \ui{V_i}) = v_i \right] \right] \\
    &= \mathop{\mathbb E}_{P(\mathbf u^c)}\left[\sum_{\mathbf V \setminus (\mathbf Y \cup \mathbf X)}  \prod_{V_i \in \mathbf V \setminus \mathbf X} \tilde{\sigma}_{v_i}(\phi_{V_i}(\pai{V_i}, \ui{V_i}^c; \theta_{V_i})) \right]. \numberthis
\end{align*}

The interventional distribution may be similarly estimated. We derive a Monte Carlo estimator for the above expression using a submodel of the \ncm{} under $do(\mathbf X=\mathbf x)$ given samples $\{\ui{j}^c\}_{j=1}^m \sim P(\mathbf U^c)$:

\begin{align*}
    & P_m^{\widehat M(\cG, \bm{\theta})}(\mathbf y | do(\mathbf x)) \\
    &= \frac{1}{m}\sum_{j=1}^m \sum_{\mathbf V \setminus (\mathbf Y \cup \mathbf X)} \prod_{V_i \in \mathbf V \setminus \mathbf X} \tilde{\sigma}_{v_i}(\phi_{V_i}(\pai{V_i}, \ui{j, V_i}^c; \theta_{V_i})). \numberthis
\end{align*}

The implementation of Eq.~\ref{eq:ncm-mc-sampling} can be summarized in Alg.~\ref{alg:ncm-pv}, which defines the ``Estimate'' function used in Alg.~\ref{alg:ncm-learn-pv}.

\begin{algorithm}[H]
    \DontPrintSemicolon
    \SetKw{notsymbol}{not}
    \SetKwData{ncmdata}{$\widehat{M}$}
    \SetKwData{pdata}{$\hat{p}$}
    \SetKwFunction{consistent}{Consistent}
    \SetKwFunction{sample}{Sample}
    \SetKwInOut{Input}{Input}
    \SetKwInOut{Output}{Output}
    
    \Input{ NCM \ncmdata, variables $\*V$ in topological order, $\*v \in \cD_{\*V}$, intervention set $\*X \subset \*V$, $\*x \in \cD_{\*X}$, number of Monte Carlo samples $m$}
    \Output{ estimate of $P^{\ncmdata}(\*v \mid do(\*x))$}
    \BlankLine
    \lIf{\notsymbol \consistent{$\*v$, $\*x$}}{\Return 0}
    $\widehat{\*u}^c_{1:m} \gets $ \sample{$P(\widehat{\*U}^c)$}\;
    $\pdata \gets 0$\;
    \For{$j \gets 1$ \KwTo $m$}{
        $\pdata_j \gets 1$\;
        \For{$i \gets 1$ \KwTo $|\*V|$}{
            \If{$V_i \not \in \*X$}{
                $\pdata_j \gets \pdata_j \sigma'_{v_i}(\phi_{V_i}(\pai{V_i}, \ui{j, V_i}^c; \theta_{V_i}))$ \tcp*{From Eq.~\ref{eq:ncm-mc-sampling}}
            }
        }
        $\pdata \gets \pdata + \pdata_j$\;
    }
    \Return $\frac{\pdata}{m}$\;
    \caption{Estimate $P^{\widehat{M}}(\*v \mid do(\*x))$ (Eq. \ref{eq:ncm-mc-sampling})}
    \label{alg:ncm-pv}
\end{algorithm}
\DecMargin{1em}

\section{Experimental Details} \label{app:experiments}

This section provides the detailed information about our experiments. Code for the experiments can be found at \url{https://github.com/CausalAILab/NeuralCausalModels}.
Our models are primarily written in PyTorch \citep{paszke2017automatic}, and training is facilitated using PyTorch Lightning \citep{falcon2020framework}.

\subsection{Data Generation Process}

The NCM architecture follows the description in Eq.~\ref{eq:ncm-differentiable}. For each function, we use a MADE module \citep{pmlr-v37-germain15} following the implementation in \citep{karpathy-made} (MIT license). Each MADE module uses 4 hidden layers of size 32 with ReLU activations. For each complete confounded component's unobserved confounder, we use uniform noise with dimension equal to the sum of the complete confounded component's variables' dimensions.

To generate data for our experiments, we use a version of the canonical SCM similar to Def.~\ref{def:simple-ctm}. Given data from $P(\*V)$ and a graph $\cG$, we construct a canonical SCM $\cM_{\ctm} = \langle \*U_{\ctm}, \*V, \cF_{\ctm}, P(\*U_{\ctm}) \rangle$ such that for each $V \in \*V$, $\pai{V}$ contains the set of variables with a directed edge into $V$ in $\cG$. For any bidirected edge between $V_1$ and $V_2$, we introduce a dependency between the noise generated from $R_{V_1}$ and $R_{V_1}$. Otherwise, variables in $\*U_{\ctm}$ are independent.

The canonical SCM is implemented in PyTorch, with trainable parameters that determine $P(\*U_{\ctm})$. In an experiment trial, we create true model $\cM^*$, which is an instantiation of a canonical SCM with random parameters. Unlike cases where $\cM^*$ comes from a set family of SCMs, which may produce biased results, the expressiveness of the canonical SCM allows us to robustly sample data generating models that can take any behavior consistent with the constraints of $\cG$.

In the estimation experiments, to make things more challenging, we take one additional step after constructing $\cM^*$ in cases where $P(Y \mid do(X))$ is not necessarily equal to $P(Y \mid X)$. To ensure that this inequality holds, we widen the difference between ATE and total variation (TV), computed in this case as $P(Y = 1 \mid X = 1) - P(Y = 1 \mid X = 0)$. We accomplish this by performing gradient descent on the parameters of the canonical SCM until this difference is at least 0.05.

From the perspective of simulating the ground truth, sampling from the canonical SCM can be accomplished using the same procedure from Def.~\ref{def:l2-semantics}.

\subsection{Identification Running Procedure}

For the identification experiments, we test on the 8 graphs shown in the top of Fig.~\ref{fig:id-exp-results}. For each graph, we run 20 trials, each repeated 4 times on the same dataset for the sake of hypothesis testing.

For each trial, we create a canonical SCM from the graph $\cG$ for the data generating model $\cM^*$. We sample 10,000 data points from $P(\*V)$, where $\*V$ is determined by the nodes in $\cG$, each being a binary variable.

With $\cG$ as input, we instantiate two NCMs, $\widehat{M}_{\min}$ and $\widehat{M}_{\max}$, and train them on the given data using a PyTorch Lightning trainer. Both models are trained for 3000 epochs using the minimization and maximization version of the objective in Eq.~\ref{eq:ncm-total-loss}, respectively. $\lambda$ is initialized to 1 at the beginning of training and decreases logarithmically throughout training until it reaches 0.001 at the end.

We use the AdamW optimizer \citep{LoshchilovH19} with the Cosine Annealing Warm Restarts learning rate scheduler \citep{LoshchilovH17}. When using Eq.~\ref{eq:ncm-mc-sampling} to compute probabilities, we choose $m = 20000$.

We log the ATE of both models every 10 iterations and use it to compute the max-min gaps from the l.h.s. of Eq.~\ref{eq:gap-test}. The gaps from each trial are averaged across the 4 runs. The percentiles of these gaps over the 20 trials are plotted across the 3000 epochs in the bottom row of Fig.~\ref{fig:id-exp-results}. We use these gaps in the hypothesis testing procedure described later in Sec.~\ref{sec:id-hyp-test} with $\tau = 0.01, 0.03, 0.05$ and plot the accuracy over the 20 trials in the middle row of Fig.~\ref{fig:id-exp-results}. For smoother plots, we use the running average over 50 iterations.

To ensure reproducibility, each run's random seed is generated through a SHA512 encoding of a key. The key is determined by the parameters of the trial such as the graph, the number of samples, and the trial index.

The NCMs are trained on NVIDIA Tesla V100 GPUs provided by Amazon Web Services. In total, we used about 150 GPU hours for the identification experiments.

\subsection{Estimation Running Procedure}

For the estimation experiments, we test on the 4 identifiable graphs at the top of Fig.~\ref{fig:id-exp-results}. For each graph, we test 15 different amounts of sample data, ranging from $10^3$ to $10^6$ on a logarithmic scale. Each setting is run for 25 trials.

For each trial, we create a canonical SCM from the graph $\cG$ for the data generating model $\cM^*$. We sample the data from $P(\*V)$, where $\*V$ is determined by the nodes in $\cG$, each being a binary variable.

We instantiate an NCM, $\widehat{M}$, given $\cG$, and we train it on the given data using a PyTorch Lightning trainer. The model is trained using the objective in Eq.~\ref{eq:ncm-probability-loss}. We incorporate early stopping, where the model ends training if loss does not decrease by a minimum of $10^{-6}$ within 100 epochs. After training, the parameters of the best model encountered during training are reloaded.

We use the AdamW optimizer \citep{LoshchilovH19} with the Cosine Annealing Warm Restarts learning rate scheduler \citep{LoshchilovH17}. When using Eq.~\ref{eq:ncm-mc-sampling} to compute probabilities, we choose $m = 1.6 \times 10^{6}$.

For each trial, we save the training data and perform the same estimation experiment with the na\"ive likelihood maximization model and WERM \cite{jung2020werm}. The na\"ive model is also optimized with the objective in Eq.~\ref{eq:ncm-probability-loss} to fit $P(\*V)$.  Since code for WERM is not  available, we implement our own version in Python trying to match the code provided in the paper. Notably, we use XGBoost \citep{Chen:2016:XST:2939672.2939785} regressors trained on the data to learn the WERM weights. These regressors are chosen with 20 estimators and a max depth of 10. They are trained with the binary logistic objective for binary outputs, otherwise they are trained using the squared error objective. The regularization parameter $\lambda_{\werm}$ (not to be confused with the $\lambda$ in Eq.~\ref{eq:ncm-total-loss}) is learned through hyperparameter tuning, checking every value from 0 to 10 in increments of 0.2. The regularization parameter $\alpha_{\werm}$ is set to $\lambda_{\werm} / 2$.

We record the KL divergence of the best models of the NCM and na\"ive models after training, computed as
\begin{equation*}
    D_{\kl}\left(P^{\cM^*}(\*V) || P^{\widehat{M}}(\*V)\right) = \sum_{\*v \in \cD_{\*V}} P^{\cM^*}(\*v) \log \left(\frac{P^{\cM^*}(\*v)}{P^{\widehat{M}}(\*v)}\right).
\end{equation*}
The mean and 95\%-confidence intervals of the KL divergence across the 25 trials are plotted over each setting of sample size in the top row of Fig.~\ref{fig:est-samples-results}.

We also record the ATE computed from all three methods after training. In the case of the na\"ive model, we use TV as the ATE calculation. For the NCM and na\"ive model, we use the sampling procedure described in Eq.~\ref{eq:ncm-mc-sampling} with $m = 10^6$ to compute these results. The mean and 95\%-confidence intervals of the ATE computations across the 25 trials are plotted over each setting of sample size in the bottom row of Fig.~\ref{fig:est-samples-results}.

Randomization is performed similarly to the identification experiments.

The NCMs and na\"ive models are trained on NVIDIA Tesla V100 GPUs provided by Amazon Web Services. In total, we used about 300 GPU hours for the estimation experiments.

\subsection{Identification Hypothesis Testing} \label{sec:id-hyp-test}

We incorporate a basic hypothesis testing procedure to test the inequality in Eq.~\ref{eq:gap-test}. In each trial of the identification experiments, we rerun the procedure on the same dataset $r$ times ($r = 4$ in our experiments). Let $\Delta \widehat{Q}$ denote the random variable representing the max-min gap from a run in a specific trial, and let $\{\Delta \hat{q}_i\}_{i=1}^{r}$ denote the empirical set of max-min gaps from all $r$ runs. The randomness in $\Delta \widehat{Q}$ arises from the randomness of parameter initialization in the NCM. $\Delta \widehat{Q}$ is not necessarily normally distributed, but we assume that the mean of $\{\Delta \hat{q}_i\}_{i=1}^{r}$ will be normally distributed given a large enough $r$ due to the central limit theorem.

For a given $\tau$, we test the hypothesis that $\bbE[\Delta \widehat{Q}] < \tau$ with the null hypothesis that $\bbE[\Delta \widehat{Q}] \geq \tau$. We estimate 
$$\bbE[\Delta \widehat{Q}] \approx \overline{\Delta \hat{q}} := \frac{1}{r}\sum_{i=1}^r \Delta \hat{q}_i,$$
the mean of the empirical set of gaps. Then we compute the standard error
$$\se(\Delta \hat{q}) := \frac{1}{r} \sqrt{\sum_{i=1}^r \left(\Delta \hat{q}_i - \overline{\Delta \hat{q}}\right)}.$$

\begin{figure}[t]
\centering
\vspace{-0.1in}
\includegraphics[width=0.9\textwidth,keepaspectratio]{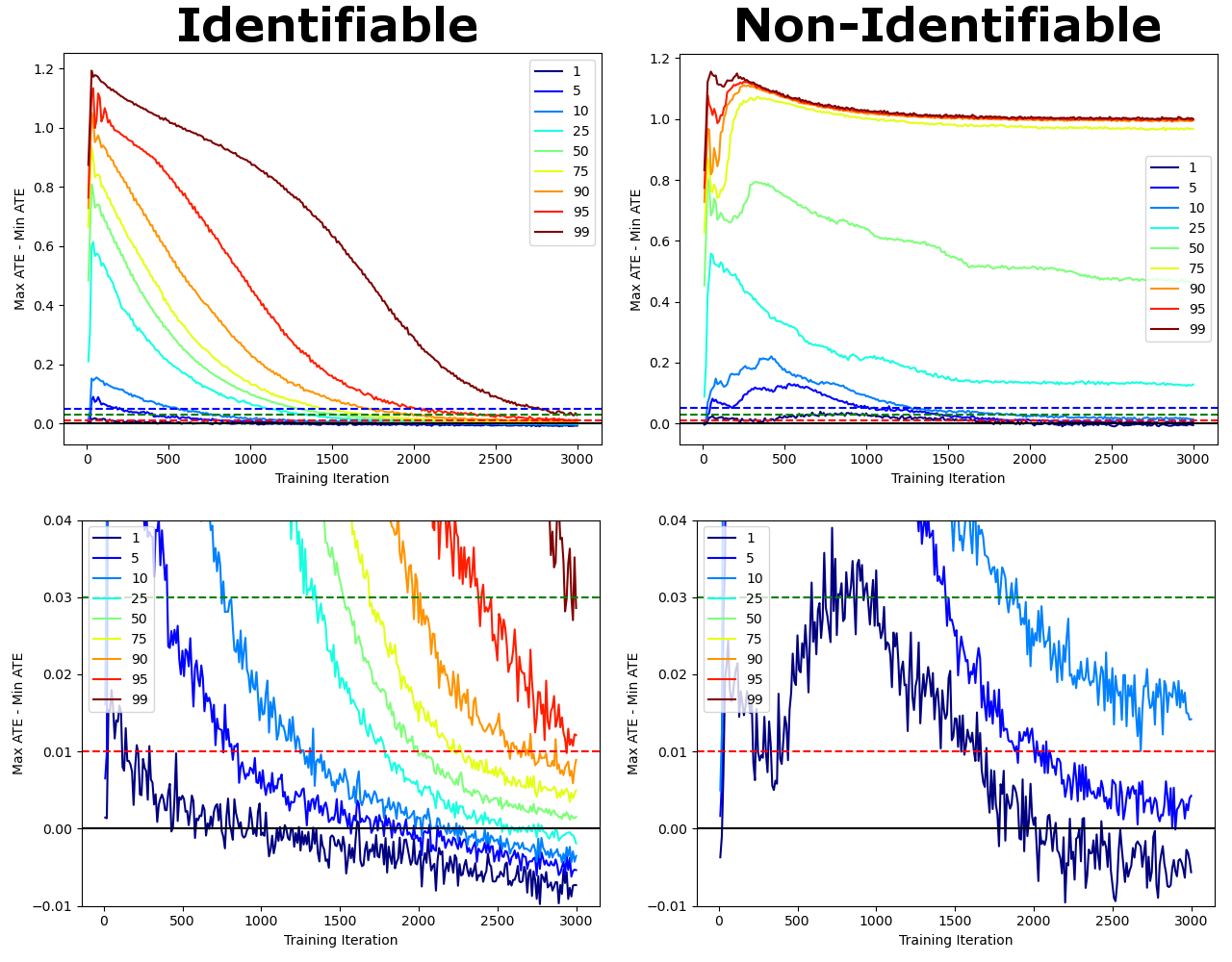}
\caption{Example plots aggregating max-min gaps across multiple settings. Left is ID, right is nonID. Bottom plots are zoomed versions of top plots. Dashed lines represent choices of $\tau$ at 0.01 (red), 0.03 (green) and 0.05 (blue).}
\label{fig:choose-tau}
\end{figure}

Finally, we check
$$\overline{\Delta \hat{q}} + 1.65 \se(\Delta \hat{q}) < \tau.$$
If this quantity holds, we reject the null hypothesis with 95\% confidence, and we return the result that the query is identifiable. If not, we fail to reject the null hypothesis and we return the result that the query is not identifiable.

\begin{wrapfigure}{r}{0.45\textwidth}
 \vspace{-0.4in}
\includegraphics[width=0.45\textwidth,keepaspectratio]{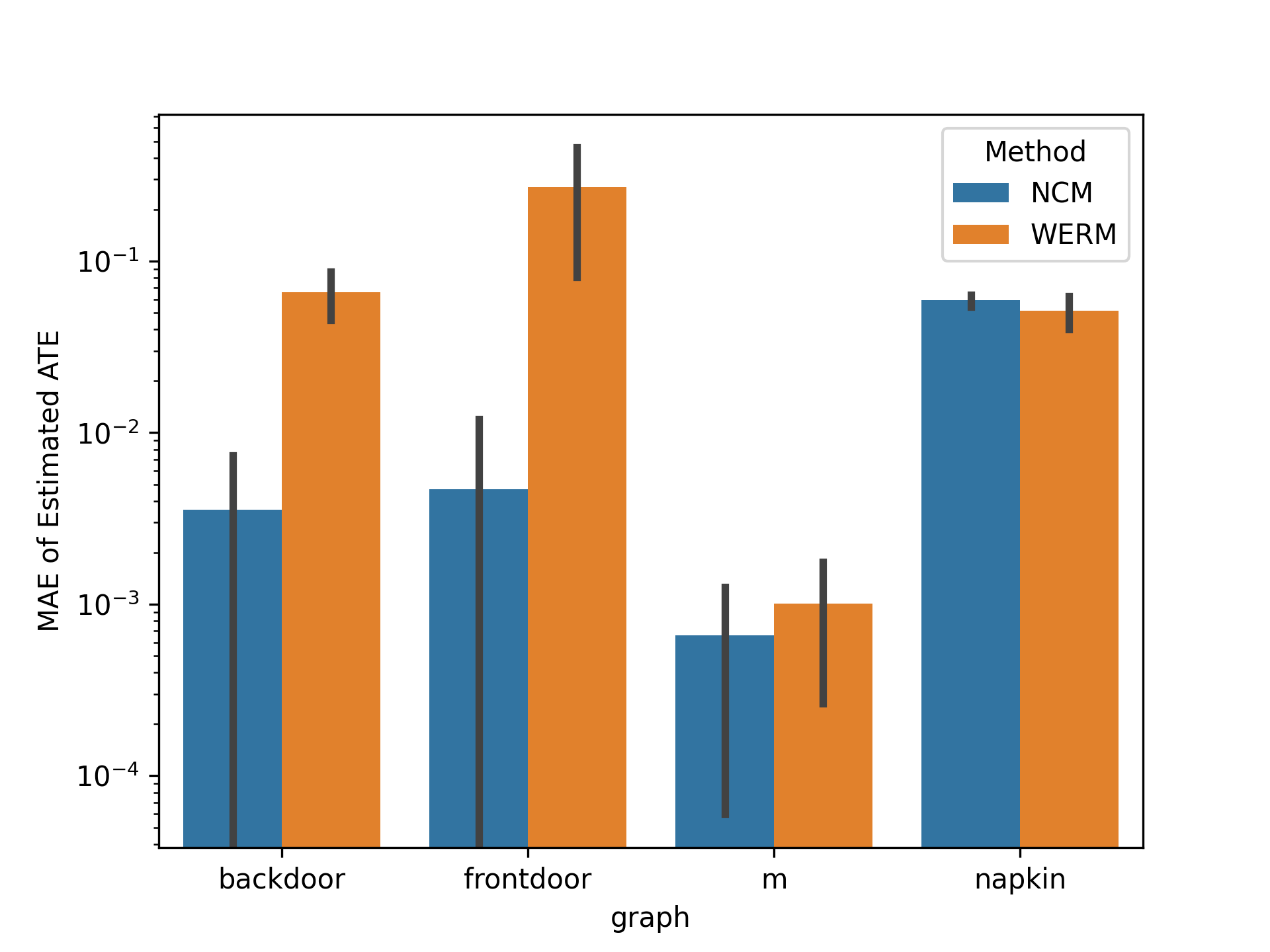}
\caption{NCM estimation results for ID graphs when setting the number of observed samples to 1000000 and the number of dimensions for each multi-dimensional binary covariate (any variable which is not $X$ or $Y$) to 20. The graphs displayed in the plot correspond to the same graphs as a, b, c, d in Fig.~\ref{fig:id-exp-results}. Plot in log scale.}
\vspace{-0.6in}
\label{fig:high-dim-results}
\end{wrapfigure}

We showed the results of three different values of $\tau$ ($0.01, 0.03, 0.05$) in Fig.~\ref{fig:id-exp-results} (main paper). Here, we try to understand how $\tau$ can be determined more systematically by running a preliminary round of identification experiments and observing the percentiles of the resulting gaps; see  Fig.~\ref{fig:choose-tau}. Note that most gaps in the ID case fall below the $0.03$ line by the end of training process, while most gaps in the nonID case stay above it. This can be used to motivate a particular choice of $\tau$.

\subsection{Additional Results}

To demonstrate the potential of NCMs to scale to larger settings, we provide additional estimation results on higher dimensional data in Fig. \ref{fig:high-dim-results}. We use the same canonical SCM architecture but probabilistically map the value of each binary variable which is not $X$ or $Y$ to a 20-dimensional binary vector, such that there exists a deterministic mapping from each covariate's 20-dimensional binary vectors to the original binary value (that is, the original binary value is recoverable from the 20-dimensional mapped vector). This preserves ATE and TV while increasing the difficulty of the NCM's optimization problem by testing the NCM's ability to learn in high-dimensional spaces. In this more challenging high-dimensional setting, the NCM maintains superior or equal performance when compared to WERM.

\subsection{Other Possible Architectures}

In Sec.~\ref{sec:ncm-fitting}, we provided one possible architecture of the NCM (Eq.~\ref{eq:ncm-differentiable}), which produced the empirical results discussed in Sec.~\ref{sec:experiments}. These results show that, in principle, training a proxy model can be used to solve the identification and estimation problems in practice.

Still, other  architectures 
may be more suitable and perform more efficiently than the one shown in  Eq.~\ref{eq:ncm-differentiable} in  larger, more involved domains.  In such settings, it may be challenging to match the distributions well enough to reach the correct identification result,  or to lower the error on estimating the query of interest. While Corol.~\ref{thm:nscm-id-correctness} guarantees the correct result with perfect optimization, this is rarely achieved in practice. It's certainly not fully understood how optimization errors on the learning of $P(\*V)$ will affect the performance of identification/estimation of Layer 2 quantities. 

Naturally, the framework developed here to support proxy SCMs is not limited to the architecture in Eq.~\ref{eq:ncm-differentiable}. For instance, as discussed in Footnote \ref{foot:alternate-ncm}, there are alternative methods of learning $P(\*V)$ such as by minimizing divergences, performing variational inference, generative adversarial optimization, to cite a few prominent choices. Some of these methods may be more scalable and robust to optimization errors. See Appendix \ref{sec:nn-general} for a discussion on how the theory generalizes for other architectures.

Further, we note that if $\*V$ comes from continuous domains, there are many ways the architecture of the NCM could be augmented accordingly. As discussed in Footnote \ref{foot:continuous-ncm}, one could replace the Gumbel-max trick in the architecture in Eq.~\ref{eq:ncm-differentiable} with a model that directly computes a probability density given a data point. Some of the alternative architectures mentioned above may work as well.

We believe this leads to an exciting research agenda that is to understand the tradeoffs involved in the choice of the architectures and how to properly optimize NCMs for specific applications. 

\section{Examples and Further Discussion} \label{app:examples}

In this section, we complement the discussion provided in the main paper through examples. Specifically, in Sec.~\ref{app:examples-expr}, we discuss what expressiveness means and provide examples to show why it's insufficient to perform cross-layer inferences. In Sec.~\ref{app:examples-cg}, we provide an intuitive discussion on why causal diagrams are able to perform cross-layer inferences and basic examples illustrating the structural constraints encoded in such models, which can be seen as an inductive bias as discussed in the body of the paper. In Sec.~\ref{app:examples-id}, we exemplify how this inductive bias can be added to \ncms{} such that they are capable of solving ID instances. In Sec.~\ref{app:examples-symbolic-neural}, we discuss possible reasons why one would prefer to perform ID through neural computation instead the machinery of the do-calculus.

\subsection{Expressiveness Examples} \label{app:examples-expr}

In this section, we try to clarify through examples the notion of expressiveness, as shown in Thm.~\ref{thm:lrepr}, and discuss the reasons it alone is not sufficient for performing cross-layer inferences. 

\begin{example}
    \label{ex:expressiveness-basic}
    Consider a study on the effects of diet ($D$) on blood pressure ($B$), inspired by studies such as \citep{appel_1997}. $D$ and $B$ are binary variables such that $D=1$ represents a high-vegetable diet, and $B=1$ represents high blood pressure. Suppose the true relationship between $B$ and $D$ is modeled by the SCM $\cM^* = \{\*U, \*V, \cF, P(\*U)\}$, where
    \begin{equation}
        \label{eq:ex-expr-basic-scm}
        \cM^* := 
        \begin{cases}
            \*U &:= \{U_D, U_{DB}, U_{B1}, U_{B2}\}, \text{ all binary} \\
            \*V &:= \{D, B\} \\
            \cF &:=
            \begin{cases}
                f_D(U_D, U_{DB}) &= \neg U_D \wedge \neg U_{DB} \\
                f_B(D, U_{B1}, U_{B2}) &= ((\neg D \oplus U_{B1}) \vee U_{DB}) \oplus U_{B2}
            \end{cases} \\
            P(\*U) &:=
            \begin{cases}
                P(U_D = 1) = P(U_{DB} = 1) = P(U_{B1} = 1) = P(U_{B2} = 1) = \frac{1}{4}\\
                \text{ all } U \in \*U \text{ are independent}
            \end{cases}
        \end{cases}
    \end{equation}
The set of exogenous variables $\*U$ affects the set of endogenous variables $\*V$. For example, $U_{DB}$ is an unobserved confounding variable that affects both diet and blood pressure (ethnicity, for example), while the remaining variables in $\*U$ represent other factors unaccounted for in the study. $\cF$ describes the relationship between the variables in $\*V$ and $\*U$, and $P(\*U)$ describes the distribution of $\*U$.
    
    \begin{table}[h]
        \centering
        \begin{tabular}{llll|ll|l}
        \hline \hline
        $U_D$ & $U_{DB}$ & $U_{B1}$ & $U_{B2}$ & $D$ & $B$      & $P(\*U)$        \\ \hline \hline
        0     & 0        & 0        & 0        & 1   & $\neg D$ & $p_0 = 81/256$  \\ \hline
        0     & 0        & 0        & 1        & 1   & $D$      & $p_1 = 27/256$  \\ \hline
        0     & 0        & 1        & 0        & 1   & $D$      & $p_2 = 27/256$  \\ \hline
        0     & 0        & 1        & 1        & 1   & $\neg D$ & $p_3 = 9/256$   \\ \hline
        0     & 1        & 0        & 0        & 0   & 1        & $p_4 = 27/256$  \\ \hline
        0     & 1        & 0        & 1        & 0   & 0        & $p_5 = 9/256$   \\ \hline
        0     & 1        & 1        & 0        & 0   & 1        & $p_6 = 9/256$   \\ \hline
        0     & 1        & 1        & 1        & 0   & 0        & $p_7 = 3/256$   \\ \hline
        1     & 0        & 0        & 0        & 0   & $\neg D$ & $p_8 = 27/256$  \\ \hline
        1     & 0        & 0        & 1        & 0   & $D$      & $p_9 = 9/256$   \\ \hline
        1     & 0        & 1        & 0        & 0   & $D$      & $p_{10} = 9/256$  \\ \hline
        1     & 0        & 1        & 1        & 0   & $\neg D$ & $p_{11} = 3/256$  \\ \hline
        1     & 1        & 0        & 0        & 0   & 1        & $p_{12} = 9/256$ \\ \hline
        1     & 1        & 0        & 1        & 0   & 0        & $p_{13} = 3/256$  \\ \hline
        1     & 1        & 1        & 0        & 0   & 1        & $p_{14} = 3/256$  \\ \hline
        1     & 1        & 1        & 1        & 0   & 0        & $p_{15} = 1/256$  \\ \hline
        \end{tabular}
        \caption{Truth table showing induced values of $\cM^*$ for Example \ref{ex:expressiveness-basic}. Probabilites in $P(\*U)$ are labeled from $p_0$ to $p_{15}$ for convenience.}
        \label{tab:ex-expr-basic-mstar}
    \end{table}
    
    This SCM $\cM^*$ induces three collections of distributions corresponding to the layers of PCH with respect to $\*V$. These can be summarized in Table \ref{tab:ex-expr-basic-mstar}. For example, a layer 1 quantity such as $P(B = 1 \mid D = 1)$, the probability of someone having high blood pressure given a high-vegetable diet, can be computed as
    \begin{eqnarray}
        \label{eq:ex-expr-basic-l1-query}
        P(B = 1 \mid D = 1) &=& \frac{P(D = 1, B = 1)}{P(D = 1)}\\ &=& \frac{p_1 + p_2}{p_0 + p_1 + p_2 + p_3} =  \frac{54}{144} = 0.375.
    \end{eqnarray}
    
    A layer 2 quantity such as $P(B = 1 \mid do(D = 1))$, the probability from $\cM^*$ of someone having high blood pressure when intervened on a high-vegetable diet, can be computed as
    \begin{eqnarray}
        \label{eq:ex-expr-basic-l2-query}
        P(B = 1 \mid do(D = 1)) &=& p_1 + p_2 + p_4 + p_6 + p_9 + p_{10} + p_{12} + p_{14}\\&=& \frac{120}{256} = 0.46875.
    \end{eqnarray}
    
    Layer 3 quantities can also be computed, for example, $P(B_{D = 1} = 1 \mid D = 0, B = 1)$. Given that an individual has high blood pressure and a low-vegetable diet, this quantity represents the probability that they would have high-blood pressure had they instead eaten a high-vegetable diet. This can be computed as
    \begin{align*}
        &P(B_{D = 1} = 1 \mid D = 0, B = 1) \\
        &= \frac{P(B_{D = 1} = 1, D = 0, B = 1)}{P(D = 0, B = 1)} = \frac{p_4 + p_6 + p_{12} + p_{14}}{p_4 + p_6 + p_8 + p_{11} + p_{12} + p_{14}} = \frac{48}{78} \approx 0.6154. \numberthis \label{eq:ex-expr-basic-l3-query}
    \end{align*}

    Following the construction in the proof for Thm.~\ref{thm:lrepr}, we will show the (somewhat tedious) construction of the corresponding NCM that induces the same distributions. First, we build a simple canonical SCM from Def.~\ref{def:simple-ctm} to match the functionality in $\cM^*$. We build $\cM_{\ctm} = \langle \*U_{\ctm}, \*V, \cF_{\ctm}, P(\*U_{\ctm})$ such that
    \begin{equation}
        \label{eq:ex-expr-basic-ctm}
        \cM_{\ctm} := 
        \begin{cases}
            \*U_{\ctm} &:= \{R_D, R_B\}, \cD_{R_D} = \{0, 1\}, \cD_{R_B} = \{0, 1, 2, 3\} \\
            \*V &:= \{D, B\} \\
            \cF_{\ctm} &:=
            \begin{cases}
                f^{\ctm}_D(R_D) &= R_D \\
                f^{\ctm}_B(D, R_B) &=
                \begin{cases}
                    0 & \text{ if } R_B = 0 \\
                    D & \text{ if } R_B = 1 \\
                    \neg D & \text{ if } R_B = 2 \\
                    1 & \text{ if } R_B = 3
                \end{cases}
            \end{cases}
        \end{cases}
    \end{equation}
    Looking at Table \ref{tab:ex-expr-basic-mstar}, we mimic the same distributions of $\cM^*$ in $\cM_{\ctm}$ by choosing
    \begin{equation}
        \label{eq:ex-expr-basic-ctm-pu}
        P(\*U_{\ctm}) = P(R_D, R_B) = 
        \begin{cases}
            p_5 + p_7 + p_{13} + p_{15} = 16/256 & R_D = 0, R_B = 0 \\
            p_9 + p_{10} = 18/256 & R_D = 0, R_B = 1 \\
            p_8 + p_{11} = 30/256 & R_D = 0, R_B = 2 \\
            p_4 + p_6 + p_{12} + p_{14} = 48/256 & R_D = 0, R_B = 3 \\
            p_1 + p_2 = 54/256 & R_D = 1, R_B = 1 \\
            p_0 + p_3 = 90/256 & R_D = 1, R_B = 2 \\
            0 & \text{otherwise.}
        \end{cases}
    \end{equation}
    We state the following useful neural network functions defined in the proof of Thm.~\ref{thm:lrepr}:
    \begin{itemize}
        \item $\hat{f}_{\scand}(x_1, \dots, x_n) = \sigma \left( \sum_{i=1}^n x_i - n \right)$
        \item $\hat{f}_{\scor}(x_1, \dots, x_n) = \sigma \left( \sum_{i=1}^n x_i - 1 \right)$
        \item $\hat{f}_{\scnot}(x) = \sigma \left( -x \right)$
        \item $\hat{f}_{\geq z}(x) = \sigma \left(x - z \right)$
    \end{itemize}
    For this appendix, $\sigma$ will be the binary step activation function. $\hat{f}_{\scand}$, $\hat{f}_{\scor}$, and $\hat{f}_{\scnot}$ compute the bitwise AND, OR, and NOT of the inputs (if binary) respectively. $\hat{f}_{\geq z}$ will output 1 if its input is greater than or equal to $z$ and will output 0 otherwise.
    
    We begin the construction of our NCM $\widehat{M} = \langle \widehat{\*U}, \*V, \widehat{\cF}, P(\widehat{\*U})$ as follows:
    \[\widehat{M} := 
    \begin{cases}
        \widehat{\*U} &:= \{\widehat{U}\}, \cD_{\widehat{U}} = [0, 1] \\
        \*V &:= \{D, B\} \\
        \widehat{\cF} &:= 
        \begin{cases}
            \hat{f}_D(\widehat{U}) &= \text{ ?} \\
            \hat{f}_B(D, \widehat{U}) &= \text{ ?}
        \end{cases}\\
        P(\widehat{\*U}) &:= P(\widehat{U}) \sim \unif(0, 1)
    \end{cases}
    \]
    We must then construct the neural networks in $\widehat{\cF}$. With the goal of reproducing the behavior of $\cM_{\ctm}$, we first start by converting a uniform distribution into $P(R_D, R_B)$. We can accomplish this by dividing the $[0, 1]$ interval into chunks for each individual value of $(R_D = r_D, R_B = r_B)$ with the size of $P(R_D = r_D, R_B = r_B)$. For instance, considering values of $P(R_D, R_B)$ in the order in Eq.~\ref{eq:ex-expr-basic-ctm-pu}, the interval $[0, 16/256)$ would correspond to $(R_D = 0, R_B = 0)$, the interval $[16/256, 34/256)$ would correspond to $(R_D = 0, R_B = 1)$, and so on. In this case, $\widehat{U}$ would be broken into six intervals corresponding to the 6 cases of Eq.~\ref{eq:ex-expr-basic-ctm-pu}: $[0, 16/256)$, $[16/256, 34/256)$, $[34/256, 64/256)$, $[64/256, 112/256)$, $[112/256, 166/256)$, and $[166/256, 1)$. For the sake of convenience, we will label these $I_1$ to $I_6$ respectively
    
    We can check if $U$ is contained within an interval like $[16/256, 34/256)$ with the following neural network function:
    \begin{align}
    \hat{f}_{I_2} = \hat{f}_{[16/256, 34/256)}(u) = \hat{f}_{\scand}\left(\hat{f}_{\geq 16/256}(u), \hat{f}_{\scnot}\left(\hat{f}_{\geq 34/256}(u) \right)\right)
    \end{align}
    Observe that $f_{I_1}$ to $f_{I_6}$ provide a one-hot encoding of the location of $\widehat{U}$. Mapping this encoding to $R_D$ and $R_B$ is simply multiplying the corresponding values of $r_D$ and $r_B$ as weights to the encoding:
    \begin{align}
        \hat{f}_{R_D}(u) &= 0 \cdot \hat{f}_{I_1}(u) + 0 \cdot \hat{f}_{I_2}(u) + 0 \cdot \hat{f}_{I_3}(u) + 0 \cdot \hat{f}_{I_4}(u) + 1 \cdot \hat{f}_{I_5}(u) + 1 \cdot \hat{f}_{I_6}(u) \\
        \hat{f}_{R_B}(u) &= 0 \cdot \hat{f}_{I_1}(u) + 1 \cdot \hat{f}_{I_2}(u) + 2 \cdot \hat{f}_{I_3}(u) + 3 \cdot \hat{f}_{I_4}(u) + 1 \cdot \hat{f}_{I_5}(u) + 2 \cdot \hat{f}_{I_6}(u)  
    \end{align}
    Note that $\hat{f}_{R_D}$ and $\hat{f}_{R_B}$ are indeed neural networks, and this final layer does not contain an activation function.
    
    Since $D = R_D$ in $\cM_{\ctm}$, we can already construct
    \begin{align}
     \hat{f}_D(u) = \hat{f}_{R_D}(u). 
    \end{align}
    The function $\hat{f}_B$ is somewhat more involved since it must take $D$ as an input. Since $R_B$ is not a binary variable, we first convert it to binary form to simplify the mapping. We can accomplish this by using the same strategy of obtaining the one-hot encoding of $R_B$. For instance, to check if $R_B = 1$, we can build the neural network function
    \begin{align}
     \hat{f}_{=1}(r_B) = \hat{f}_{\scand}\left(\hat{f}_{\geq 1} (r_B), \hat{f}_{\scnot}\left(\hat{f}_{\geq 2} (r_B) \right)\right).
    \end{align}

    The encoding is formed with $\hat{f}_{=0}$, $\hat{f}_{=1}$, $\hat{f}_{=2}$, and $\hat{f}_{=3}$.
    
    Using this, we can construct $\hat{f}_B(u)$ following the desired properties in Eq.~\ref{eq:ex-expr-basic-ctm}:
    \begin{align}
     \hat{f}_{B}(d, u) = \hat{f}_{\scor}\left( \hat{f}_{\scand}\left(\hat{f}_{=1}(r_B), d\right), \hat{f}_{\scand}\left(\hat{f}_{=2}(r_B), \hat{f}_{\scnot}(d)\right), \hat{f}_{=3}(r_B)\right)
    \end{align}
     where $r_B = \hat{f}_{R_B}(u)$.
    \begin{table}[h]
        \centering
        \begin{tabular}{l|ll|l}
        \hline
        \hline
        $\widehat{U}$        & $D$ & $B$      & $P(\*U)$       \\ \hline \hline
        $[0, 16/256)$        & 0   & 0        & $q_0 = 16/256$ \\ \hline
        $[16/256, 34/256)$   & 0   & $D$      & $q_1 = 18/256$ \\ \hline
        $[34/256, 64/256)$   & 0   & $\neg D$ & $q_2 = 30/256$ \\ \hline
        $[64/256, 112/256)$  & 0   & 1        & $q_3 = 48/256$ \\ \hline
        $[112/256, 166/256)$ & 1   & $D$      & $q_4 = 54/256$ \\ \hline
        $[166/256, 1)$       & 1   & $\neg D$ & $q_5 = 90/256$ \\ \hline
        \end{tabular}
        \caption{Truth table showing induced values of $\widehat{M}$ for Example \ref{ex:expressiveness-basic}. Probabilites in $P(\*U)$ are labeled from $q_0$ to $q_5$ for convenience.}
        \label{tab:ex-expr-basic-mhat}
    \end{table}
    With $\hat{f}_D$ and $\hat{f}_B$ defined, our construction of $\widehat{M}$ is complete. 
    
    We can now verify that the distributions induced by $\widehat{M}$, as represented in Table \ref{tab:ex-expr-basic-mhat}, matches the distributions induced by $\cM^*$. Consider the same three queries from Eqs.~\ref{eq:ex-expr-basic-l1-query}, \ref{eq:ex-expr-basic-l2-query}, and \ref{eq:ex-expr-basic-l3-query}:
    $$P^{\widehat{M}}(B = 1 \mid D = 1) = \frac{P^{\widehat{M}}(D = 1, B = 1)}{P^{\widehat{M}}(D = 1)} = \frac{q_4}{q_4 + q_5} = \frac{54}{144} = 0.375$$
    $$P^{\widehat{M}}(B = 1 \mid do(D = 1)) = q_1 + q_3 + q_4 = \frac{120}{256} = 0.46875$$
    \begin{align*}
        &P^{\widehat{M}}(B_{D = 1} = 1 \mid D = 0, B = 1) \\
        &= \frac{P^{\widehat{M}}(B_{D = 1} = 1, D = 0, B = 1)}{P^{\widehat{M}}(D = 0, B = 1)} = \frac{q_3}{q_2 + q_3} = \frac{48}{78} \approx 0.6154
    \end{align*}
    
    This demonstrates that the NCM is indeed expressive enough to model this setting on all three layers of PCH. While a more basic model might be fitted to answer $L_1$ queries (such as $P(B = 1 \mid D = 1)$), a well-parameterized NCM is expressive enough to represent distributions on higher layers.
    \hfill $\blacksquare$
\end{example}

Example \ref{ex:expressiveness-basic} shows a scenario where the expressiveness of the NCM class allowed us to fit an NCM instance that could completely match a complex SCM and answer any question of interest. However, we also highlight that expressiveness alone does not necessarily mean that the NCM can solve any causal problem. Experimental data, data from layer 2, can be difficult to obtain in practice, so it would be convenient if we could fit a model $\widehat{M}$ only on observational data from layer 1, then deduce the same causal results using $\widehat{M}$ as if it were the true SCM, $\cM^*$. Unfortunately, Corol.~\ref{cor:ncht} states that this deduction typically cannot be made no matter the expressiveness of the model. The next example will illustrate this point more concretely.

\begin{example}
    \label{ex:ncht}
    
    \begin{table}[h]
        \centering
        \begin{tabular}{ll|l}
        \hline \hline
        $D$ & $B$ & $P(D, B)$ \\ \hline \hline
        0   & 0   & $34/256$  \\ \hline
        0   & 1   & $78/256$  \\ \hline
        1   & 0   & $90/256$  \\ \hline
        1   & 1   & $54/256$  \\ \hline
        \end{tabular}
        \caption{Observational distribution $P(\*V)$ induced by $\cM^*$ from Example \ref{ex:expressiveness-basic}.}
        \label{tab:ex-ncht-pv}
    \end{table}

    Consider the same study of the effects of diet on blood pressure in Example \ref{ex:expressiveness-basic} in which the two variables are described by the SCM, $\cM^*$, in Eq.~\ref{eq:ex-expr-basic-scm}.
    
    While in the previous example we were able to construct NCM $\widehat{M}$ which matched $\cM^*$'s behavior on all three layers of the PCH, this was contingent on having access to the true SCM. This is unlikely to happen in practice. Instead, we typically only have partial information from the SCM. In many cases, we may only have observational data from layer 1. The data in  Table \ref{tab:ex-ncht-pv} exemplifies the aspect of $\cM^*$ we may have access to. 
    
Naturally, we should construct an NCM such that matches on the given observed dataset. The NCM $\widehat{M}$ from Example \ref{ex:expressiveness-basic} will indeed induce the same observational data $P(\*V)$, but there are several other ways we could construct a model that fits such criterion.
    For illustration purposes, consider four NCMs $\widehat{M}_1$, $\widehat{M}_2$, $\widehat{M}_3$, and $\widehat{M}_4$ such that $\widehat{M}_i = \langle \widehat{\*U}, \*V, \widehat{\cF}_i, P(\widehat{\*U}) \rangle$ is defined as follows:
    
    \begin{equation*}
        \begin{cases}
            \widehat{\*U} &:= \{\widehat{U}_D, \widehat{U}_B\}, \cD_{\widehat{U}_D} = \cD_{\widehat{U}_B} = [0, 1] \\
            \*V &:= \{D, B\} \\
            \widehat{\cF}_i &:= \{\hat{f}_D^i, \hat{f}_B^i\} \\
            P(\widehat{\*U}) &:= \widehat{U}_D, \widehat{U}_B \sim \unif(0, 1), \widehat{U}_D \indep \widehat{U}_B
        \end{cases}
    \end{equation*}
    
    We use two variables in $\widehat{\*U}$ with the intent of simplifying some of the expressions, but it is feasible to construct NCMs with the same behavior using one uniform random variable, similar to Example \ref{ex:expressiveness-basic}.
    
    Concretely, we construct $\widehat{\cF}_1$ such that  $\widehat{M}_1$ has no unobserved confounding, 
    \begin{equation}
        \widehat{\cF}_1 := 
        \begin{cases}
            \hat{f}^1_D(u_d) &= 
            \begin{cases}
                1 & u_d \geq \frac{112}{256} \\
                0 & \text{otherwise}
            \end{cases}
            \\
            \hat{f}^1_B(d, u_b) &= 
            \begin{cases}
                1 & (d = 0) \wedge (u_b \geq \frac{34}{112}) \\
                1 & (d = 1) \wedge (u_b \geq \frac{90}{144}) \\
                0 & \text{otherwise}
            \end{cases}
        \end{cases}
    \end{equation}
    The corresponding neural network for $\widehat{\cF}_1$ can be seen as 
    \begin{align}
        \hat{f}^1_D(u_d) &= \hat{f}_{\geq 112/256}(u_d) \\
        \hat{f}^1_B(d, u_b) &= \hat{f}_{\scor}\left( \hat{f}_{\scand} \left( \hat{f}_{\scnot}(d), \hat{f}_{\geq 34/112}(u_b)\right), \hat{f}_{\scand} \left(d, \hat{f}_{\geq 90/144}(u_b)\right) \right).
    \end{align}
    
    $\widehat{\cF}_2$ is constructed such that the causal direction of $D$ and $B$ in $\widehat{M}_2$ is reversed, i.e., 
    \begin{equation}
        \widehat{\cF}_2 :=
        \begin{cases}
            \hat{f}^2_D(b, u_d) &= 
            \begin{cases}
                1 & (b = 0) \wedge (u_d \geq \frac{34}{124}) \\
                1 & (b = 1) \wedge (u_d \geq \frac{78}{132}) \\
                0 & \text{otherwise}
            \end{cases}
            \\
            \hat{f}^2_B(u_b) &= 
            \begin{cases}
                1 & u_b \geq \frac{124}{256} \\
                0 & \text{otherwise}
            \end{cases}
        \end{cases}
    \end{equation}
    The corresponding neural network for $\widehat{\cF}_2$ can be written as
    \begin{align}
        \hat{f}^2_D(b, u_d) &= \hat{f}_{\scor}\left( \hat{f}_{\scand} \left( \hat{f}_{\scnot}(b), \hat{f}_{\geq 34/124}(u_d)\right), \hat{f}_{\scand} \left(b, \hat{f}_{\geq 78/132}(u_d)\right) \right) \\
        \hat{f}^2_B(u_b) &= \hat{f}_{\geq 124/256}(u_b).
    \end{align}
    
    $\widehat{\cF}_3$ is constructed such that $P^{\widehat{M}_3}(B = 1 \mid do(D = 1))$ is maximized, namely, 
    \begin{equation}
        \widehat{\cF}_3 := 
        \begin{cases}
            \hat{f}^3_D(u_d) &= 
            \begin{cases}
                1 & u_d \geq \frac{112}{256} \\
                0 & \text{otherwise}
            \end{cases}
            \\
            \hat{f}^3_B(d, u_b, u_d) &= 
            \begin{cases}
                1 & (d = 1) \wedge (u_b < \frac{34}{112}) \wedge (u_d < \frac{112}{256}) \\
                1 & (u_b \geq \frac{34}{112}) \wedge (u_d < \frac{112}{256}) \\
                1 & (d = 1) \wedge (u_b < \frac{54}{144}) \wedge (u_d \geq \frac{112}{256}) \\
                0 & \text{otherwise}
            \end{cases}
        \end{cases}
    \end{equation}
        The corresponding neural network for $\widehat{\cF}_3$ can be written as
    \begin{align}
        \hat{f}^3_D(u_d) &= \hat{f}_{\geq 112/256}(u_d) \\
        \hat{f}^3_B(d, u_b, u_d) &= 
        \hat{f}_{\scor}
        \begin{cases}
            \hat{f}_{\scand} \left( d, \hat{f}_{\scnot} \left(\hat{f}_{\geq 34/112}(u_b)\right), \hat{f}_{\scnot} \left(\hat{f}_{\geq 112/256}(u_d)\right) \right) \\
            \hat{f}_{\scand} \left(\hat{f}_{\geq 34/112}(u_b), \hat{f}_{\scnot} \left(\hat{f}_{\geq 112/256}(u_d)\right) \right) \\
            \hat{f}_{\scand} \left( d, \hat{f}_{\scnot} \left(\hat{f}_{\geq 54/144}(u_b)\right), \hat{f}_{\geq 112/256}(u_d)\right).
        \end{cases}
    \end{align}
    
    Finally, $\widehat{\cF}_4$ is constructed such that $P^{\widehat{M}_4}(B = 1 \mid do(D = 1))$ is minimized, 
    \begin{equation}
        \widehat{\cF}_4 := 
        \begin{cases}
            \hat{f}^4_D(u_d) &= 
            \begin{cases}
                1 & u_d \geq \frac{112}{256} \\
                0 & \text{otherwise}
            \end{cases}
            \\
            \hat{f}^4_B(d, u_b, u_d) &= 
            \begin{cases}
                1 & (d = 0) \wedge (u_b \geq \frac{34}{112}) \wedge (u_d < \frac{112}{256}) \\
                1 & (d = 1) \wedge (u_b < \frac{54}{144}) \wedge (u_d \geq \frac{112}{256}) \\
                0 & \text{otherwise}
            \end{cases}
        \end{cases}
    \end{equation}
    The corresponding neural network for $\widehat{\cF}_4$ can be written as
    \begin{align}
        \hat{f}^4_D(u_d) &= \hat{f}_{\geq 112/256}(u_d) \\
        \hat{f}^4_B(d, u_b, u_d) &= 
        \hat{f}_{\scor}
        \begin{cases}
            \hat{f}_{\scand} \left( \hat{f}_{\scnot} (d), \hat{f}_{\geq 34/112}(u_b), \hat{f}_{\scnot} \left(\hat{f}_{\geq 112/256}(u_d)\right) \right) \\
            \hat{f}_{\scand} \left( d, \hat{f}_{\scnot} \left(\hat{f}_{\geq 54/144}(u_b)\right), \hat{f}_{\geq 112/256}(u_d)\right).
        \end{cases}
    \end{align}
    
    Even though the functions of these NCMs are constructed in very different ways, one can verify that they indeed induce the same layer 1 as shown in Table \ref{tab:ex-ncht-pv}.
    
    Despite this match, they all disagree on a simple layer 2 quantity such as $P(B = 1\mid do(D = 1))$; to witness note that  
    \begin{align}
        P^{\widehat{M}_1}(B = 1 \mid do(D = 1)) &= P(B = 1 \mid D = 1) = \frac{54}{144} = 0.375 \\
        P^{\widehat{M}_2}(B = 1 \mid do(D = 1)) &= P(B = 1) = \frac{132}{256} \approx 0.5156 \\
        P^{\widehat{M}_3}(B = 1 \mid do(D = 1)) &= \frac{166}{256} \approx 0.6484 \\
        P^{\widehat{M}_4}(B = 1 \mid do(D = 1)) &= \frac{54}{256} \approx 0.2109
    \end{align}
    
    Without further information, it's not possible to distinguish which model (if any) is the correct one. Na\"ievely, choosing an arbitrary model is likely to result in misleading results, even if the model can reproduce the given $L_1$ data with high fidelity.
    
    In practice, depending on the chosen model, the conclusion of the study could be dramatically different. For instance, someone who fits models $\widehat{M}_1$ or $\widehat{M}_4$ may conclude that a high-vegetable diet is beneficial for lowering blood pressure, while someone who fits model $\widehat{M}_2$ may conclude that it has no effect at all. Someone who fits model $\widehat{M}_3$ may even conclude that eating more vegetables is harmful for regulating blood pressure.
    
    \hfill $\blacksquare$
\end{example}

\subsection{Structural Constraints Embedded in Causal Bayesian Networks} \label{app:examples-cg} 

In this section, we provide an example to illustrate the inductive bias embedded in causal diagrams, building on the more comprehensive treatment provided in \cite[Sec.~1.4]{bareinboim:etal20}. 

\begin{example}
    \label{ex:cg}
    Consider the following two SCMs:
    \begin{equation*}
        \cM_1 := 
        \begin{cases}
            \*U &:= \{U_X, U_Y\}, \text{ all binary} \\
            \*V &:= \{X, Y\}, \text{ all binary} \\
            \cF_1 &:=
            \begin{cases}
                f^1_X(U_X) &= U_X \\
                f^1_Y(X, U_Y) &= X \oplus U_Y
            \end{cases} \\
            P(\*U) &:= P(U_X = 1) = \frac{1}{2}, P(U_Y = 1) = \frac{1}{4}, U_X \indep U_Y
        \end{cases}
    \end{equation*}
    \begin{equation*}
        \cM_2 := 
        \begin{cases}
            \*U &:= \{U_X, U_Y\}, \text{ all binary} \\
            \*V &:= \{X, Y\}, \text{ all binary} \\
            \cF_2 &:=
            \begin{cases}
                f^2_X(Y, U_X) &= Y \oplus U_X \\
                f^2_Y(U_Y) &= U_Y
            \end{cases} \\
            P(\*U) &:= P(U_X = 1) = \frac{1}{4}, P(U_Y = 1) = \frac{1}{2}, U_X \indep U_Y
        \end{cases}
    \end{equation*}
    
    Note that both $\cM_1$ and $\cM_2$ induce the same $L_1$ distributions. 
    
    For example, in both models, $P(Y = 1) = \frac{1}{2}$, and $P(Y = 1 \mid X = 1) = \frac{3}{4}$. However, even without making any computation, it seems intuitive from the structure of the functions and noise that a causal query like $P(Y = 1 \mid do(X = 1))$ would have different answers in both models. First, note that $f^2_Y$ in $\cM_2$ does not  have $X$ as an argument, and there are no other variables in $\*V$, so intervening on $X$ would have no causal effect on $Y$. In other words, the constraint $P^{\cM_2}(Y = 1 \mid do(X = 1)) = P(Y = 1)$ is implied. More subtly in the case of $\cM_1$,  we note that $X$ directly affects $Y$ and there is no unobserved confounding between $X$ and $Y$. This means that the observed association between $X$ and $Y$ must be due to the causal effect (i.e.\ $P^{\cM_1}(Y = 1 \mid do(X = 1)) = P(Y = 1 \mid X = 1)$). The relationship between $X$ and $Y$ in these two cases can be seen graphically, as shown in Fig.~\ref{fig:simple-graphs}.
    
    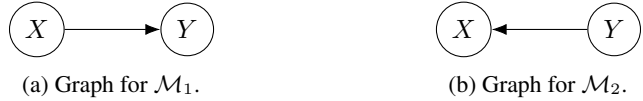
\begin{figure}[h]
    	\centering
    	\begin{subfigure}{0.4\textwidth}
    	    \centering
    	    \begin{tikzpicture}[xscale=1, yscale=1.5]
        		\node[draw, circle] (X) at (-1, -1) {$X$};
        		\node[draw, circle] (Y) at (1, -1) {$Y$};
        
        		\path [-{Latex[length=2mm]}] (X) edge (Y);
        	\end{tikzpicture}
        	\caption{Graph for $\cM_1$.}
        	\label{fig:simple-graphs-a}
    	\end{subfigure}
    	\begin{subfigure}{0.4\textwidth}
    	    \centering
    	    \begin{tikzpicture}[xscale=1, yscale=1.5]
        		\node[draw, circle] (X) at (-1, -1) {$X$};
        		\node[draw, circle] (Y) at (1, -1) {$Y$};
        
        		\path [-{Latex[length=2mm]}] (Y) edge (X);
        	\end{tikzpicture}
        	\caption{Graph for $\cM_2$.}
        	\label{fig:simple-graphs-b}
    	\end{subfigure}
    	\caption{Causal diagrams for SCMs in Example \ref{ex:cg}. A directed edge from $A$ to $B$ indicates that $A$ is an argument of the function for $B$.}
    	\label{fig:simple-graphs}
    \end{figure}
    
	\hfill $\blacksquare$
\end{example}

As illustrated in the previous example, and discussed more formally in \citep{bareinboim:etal20}, these equality constraints across distributions arise from the qualitative structural properties of the SCM such as which variable is an argument of which function, and which variables share exogenous influence. In fact, these constraints are  invariant to the details of the functions and the distribution of the exogenous variables. 

Moreover, there are exponentially many constraints implied by an SCM in the collection of  $L_1$ and $L_2$ distributions, which can be parsimoniously represented in the form of a \textit{causal bayesian network} (Def.~\ref{def:cbn})\footnote{This is a generalization of the notion of markov compatibility used in  $L_1$ models, such as bayesian networks \cite{pearl:88a}. In our case, there are multiple distributions of different nature, observational and experimental, and the constraints are among them (i.e., not conditional independences). } In fact, the  graphical component of this object provides an intuitive interpretation for these constraints. It can be shown that the closure of such constraints entails the do-calculus \citep[Thm.~5]{bareinboim:etal20}, which means that all identifiable effects can be computed from them (due to the completeness shown in \cite{lee:etal19}). 

\begin{figure}[h]
    \centering
    \begin{tikzpicture}[xscale=1, yscale=1.5]
		\node[draw, circle] (D) at (-1, 0) {$D$};
		\node[draw, circle] (B) at (1, 0) {$B$};

		\path [-{Latex[length=2mm]}] (D) edge (B);
		\path [dashed, {Latex[length=2mm]}-{Latex[length=2mm]}, bend left] (D) edge (B);
	\end{tikzpicture}
	\caption{Causal diagram $\cG$ of $\cM^*$ from Example \ref{ex:expressiveness-basic}}
	\label{fig:ex-expr-basic-cg}
\end{figure}

Revisiting Examples \ref{ex:expressiveness-basic} and \ref{ex:ncht} and applying the definition of causal diagram (Def.~\ref{def:scm-cg}), we can see from the structure of $\cM^*$ that the causal diagram $\cG$ in Fig.~\ref{fig:ex-expr-basic-cg} is induced. If we had this information, we could immediately eliminate $\widehat{M}_2$ from Example \ref{ex:ncht} as a possible model. The reason is that we can read off from diagram $\cG$ that $B$ cannot be an argument of $f_D$. 

With an SCM-like structure, the NCM naturally induces its own structural constraints. Def.~\ref{def:g-cons-nscm} in the main paper shows how to construct an NCM to fit the same constraints entailed by a given causal diagram from another SCM. In some way, the original model induces mark in the collection of observational and experimental distributions. In the case of the NCM, these constraints are used as input, inductive bias to constrain its functions.  This is powerful because any inferences performed on a $\cG$-constrained NCM will automatically satisfy the constraints of $\cG$.

\subsection{Solving Identification through NCMs} \label{app:examples-id}

\begin{figure}[h]
    \centering
	\begin{subfigure}[t]{0.4\textwidth}
	    \centering
	    \includegraphics[width=\textwidth,height=0.5\textheight,keepaspectratio]{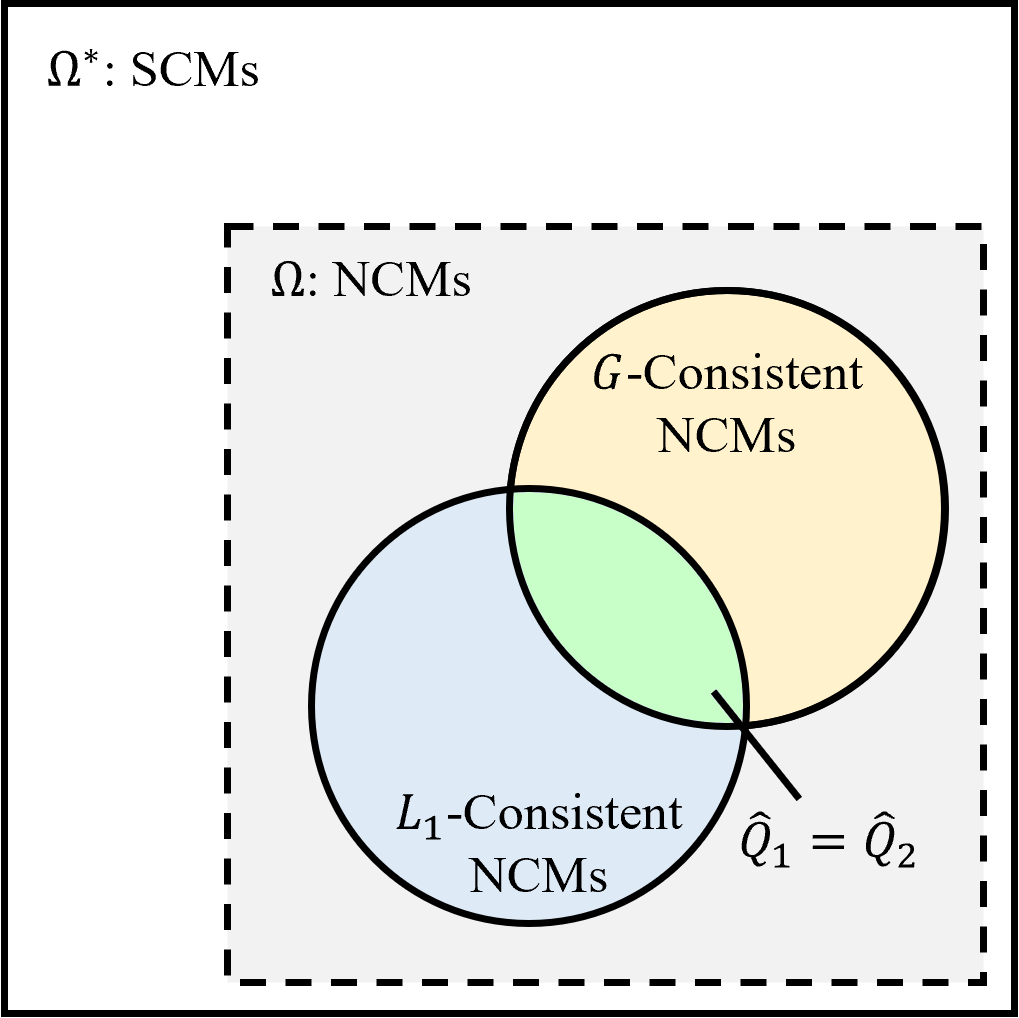}
        \caption{In the identifiable case, all NCMs that are $\cG$-consistent and $L_1$-consistent with $\cM^*$ will also match in $Q$.}
        \label{fig:id-venn-diagram-idcase}
	\end{subfigure}
	\quad
	\begin{subfigure}[t]{0.4\textwidth}
	    \centering
	    \includegraphics[width=\textwidth,height=0.5\textheight,keepaspectratio]{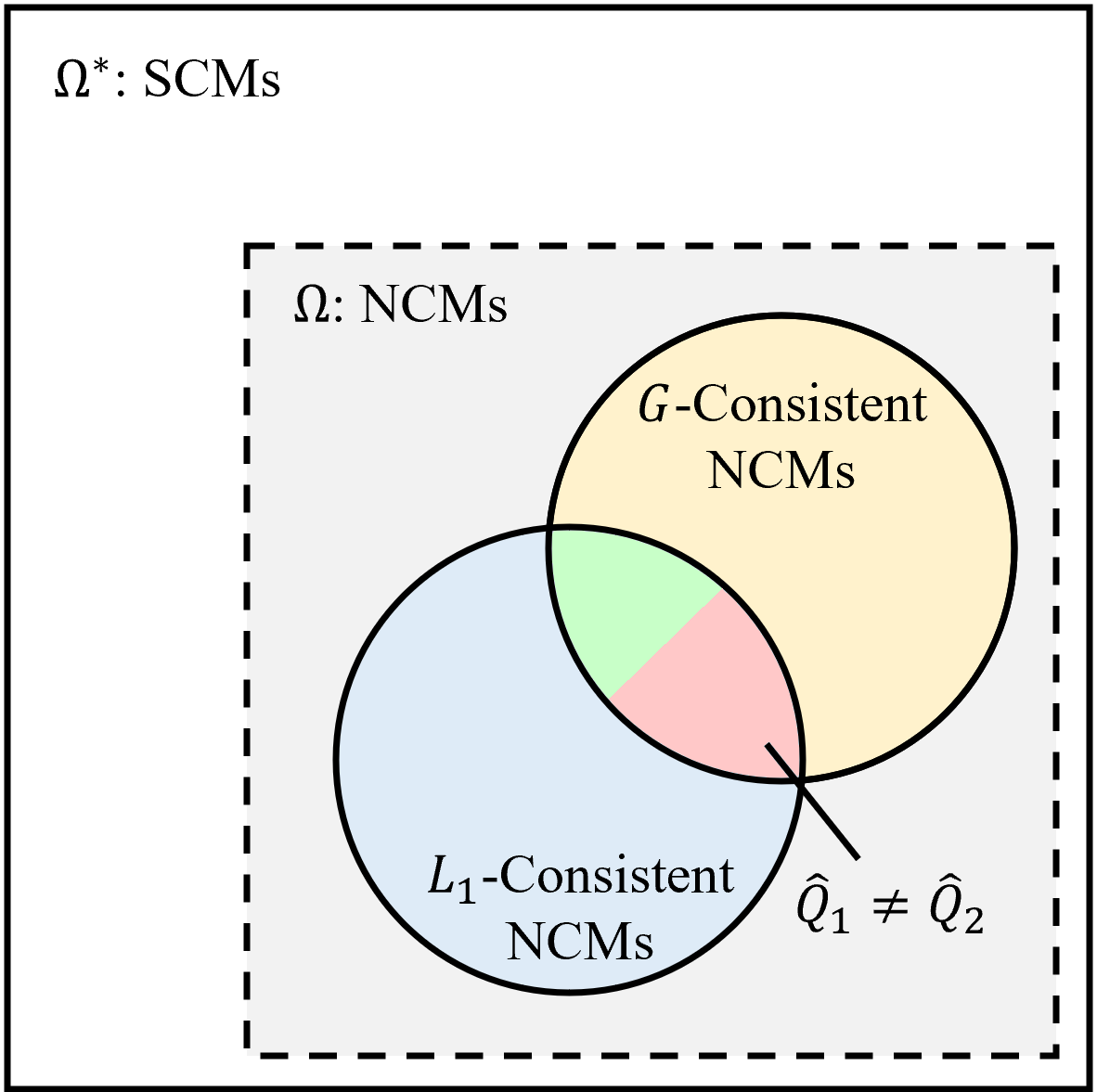}
        \caption{In the non-identifiable case, there could exist two NCMs, $\widehat{M}_1$ and $\widehat{M}_2$, that are both $\cG$-consistent and $L_1$-consistent but still disagree in $Q$.}
        \label{fig:id-venn-diagram-nonidcase}
	\end{subfigure}
	\caption{A visual representation of the ID problem. Here, $Q$ is a query of interest, and $\widehat{Q}_i$ is the answer for that query induced by NCM $\widehat{M}_i$. The goal is to check if all NCMs that are $\cG$-consistent and $L_1$-consistent are also consistent in $Q$.}
	\label{fig:id-venn-diagram}
\end{figure}

The notion of identification requires that a certain effect is identifiable by all (unobserved) SCMs compatible with the corresponding structural constraints. 
This was extended to NCMs through Def. ~\ref{def:nscm-id}; see also Fig.~\ref{fig:id-venn-diagram}. 
In fact, Alg.~\ref{alg:nscm-solve-id} helps to solve the neural identification problem such that if there are two NCMs compatible with the constraints but with different predictions for the causal effect, it will return ``\texttt{FAIL}''. Otherwise, it will returns the estimation of the query in every ID case. 
Note that identification of an effect does not require that the NCM and the true SCM have the same functions, just the effects need to match. 
The following two examples illustrate both positive and negative instances of such  operation.  

\begin{example}
    \label{ex:non-id}
    Consider once again the study of diet on blood pressure introduced in Example \ref{ex:expressiveness-basic}. Suppose we are particularly interested in the $L_2$ expression of $Q = P(B = 1 \mid do(D = 1))$, the causal effect of diet on blood pressure. The true SCM $\cM^*$ is not available, but $L_1$ data ($P(\*v)$) presented in Table \ref{tab:ex-ncht-pv} is (as shown Example \ref{ex:ncht}). Furthermore, we are informed by a doctor that the two variables follow the relation specified by the causal diagram $\cG$ in Fig.~\ref{fig:ex-expr-basic-cg}. Is $Q$ identifiable from $P(\*v)$ and $\Omega(\cG)$, where $\Omega(\cG)$ is the set of all $\cG$-constrained NCMs?
    
    Unfortunately, the answer is no for this particular case. To see why, example \ref{ex:ncht} provides 4 NCMs whose $L_1$ distributions match $P(\*v)$ but disagree on $Q$. Even though $\widehat{M}_2$ can be eliminated as a possible proxy since it is not $\cG$-consistent, the other three models are compatible with $\cG$, so we still cannot pinpoint the correct answer for $Q$. 
    
    Running Alg.~\ref{alg:nscm-solve-id} on this setting would result in the algorithm generating two sets of parameters, $\bm{\theta}^*_{\min}$ and $\bm{\theta}^*_{\max}$, that minimize and maximize $Q$. The behaviors of an NCM with these parameters may reflect the behaviors of $\widehat{M}_3$ and $\widehat{M}_4$ from Example \ref{ex:ncht}, and they do not agree on $Q$.
    	\hfill $\blacksquare$
\end{example}

\begin{example}
    \label{ex:id}
    Suppose now we return to the data collection process and we receive new information about a third variable, sodium intake ($S$) ($S=1$ represents a high sodium diet, $0$ otherwise). 
    
    With the introduction of $S$ in the endogenous set $\*V$, $\cM^*$ can be written as follows:
    \vspace{-0.15in}
    
    \begin{equation}
        \label{eq:ex-id-scm}
        \cM^* := 
        \begin{cases}
            \*U &:= \{U_D, U_{DB}, U_{S}, U_{B}\}, \text{ all binary} \\
            \*V &:= \{D, S, B\} \\
            \cF &:=
            \begin{cases}
                f_D(U_D, U_{DB}) &= \neg U_D \wedge \neg U_{DB} \\
                f_S(D, U_S) &= \neg D \oplus U_{S} \\
                f_B(S, U_{B}) &= (S \vee U_{DB}) \oplus U_{B}
            \end{cases} \\
            P(\*U) &:=
            \begin{cases}
                P(U_D = 1) = P(U_{DB} = 1) = P(U_{S} = 1) = P(U_{B} = 1) = \frac{1}{4}\\
                \text{ all } U \in \*U \text{ are independent}
            \end{cases}
        \end{cases}
    \end{equation}
    
    As reflected in the updated $\cM^*$, diet only affects blood pressure through sodium content, which tends to be higher in low-vegetable diets. Note that the $L_1$, $L_2$, and $L_3$ distributions relating $D$ and $B$ are unchanged with this modification of $\cM^*$. In this case, the query of interest is evaluated as
    \begin{align*}
        P(B = 1 \mid do(D = 1)) &= P\left((S_{D = 1} \vee U_{DB}) \oplus U_{B} = 1\right) \\
        &= P\left(((\neg 1 \oplus U_{S}) \vee U_{DB}) \oplus U_{B} = 1\right) \\
        &= P\left((U_S \vee U_{DB}) \oplus U_{B} = 1\right) \\
        &= \frac{120}{256} = 0.46875, \numberthis \label{eq:ex-id-l2-query}
    \end{align*}
    which matches the result in Eq.~\ref{eq:ex-expr-basic-l2-query}.
    
    \begin{table}[h]
        \centering
        \begin{tabular}{lll|l}
        \hline \hline
        $D$ & $S$ & $B$ & $P(D, S, B)$ \\ \hline \hline
        0   & 0   & 0   & $13/256$     \\ \hline
        0   & 0   & 1   & $15/256$     \\ \hline
        0   & 1   & 0   & $21/256$     \\ \hline
        0   & 1   & 1   & $63/256$     \\ \hline
        1   & 0   & 0   & $81/256$     \\ \hline
        1   & 0   & 1   & $27/256$     \\ \hline
        1   & 1   & 0   & $9/256$      \\ \hline
        1   & 1   & 1   & $27/256$     \\ \hline
        \end{tabular}
        \caption{Updated observational distribution $P(\*V)$ induced by $\cM^*$ for Example \ref{ex:id}.}
        \label{tab:ex-id-pv}
    \end{table}
    
    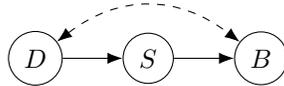
\begin{figure}[h]
        \centering
        \begin{tikzpicture}[xscale=1.5, yscale=2]
    		\node[draw, circle] (D) at (-1, 0) {$D$};
    		\node[draw, circle] (S) at (0, 0) {$S$};
    		\node[draw, circle] (B) at (1, 0) {$B$};
    
    		\path [-{Latex[length=2mm]}] (D) edge (S);
    		\path [-{Latex[length=2mm]}] (S) edge (B);
    		\path [dashed, {Latex[length=2mm]}-{Latex[length=2mm]}, bend left] (D) edge (B);
    	\end{tikzpicture}
    	\caption{Causal diagram $\cG$ of $\cM^*$ with additional $S$ variable for Example \ref{ex:id}}
    	\label{fig:ex-id-cg}
    \end{figure}
    
    This is just the computation of the distribution from Nature's perspective. As the previous examples, we do not have access to the true model ($\cM^*$), just the observational distribution $P(\*v)$, as shown in  Table \ref{tab:ex-id-pv}. We do have access to the updated causal diagram as shown in Fig.~\ref{fig:ex-id-cg}. 
    
    The addition of the new variable and the refinement of the model change the identifiability status of the query $Q = P(B = 1 \mid do(D = 1))$. No matter how Alg.~\ref{alg:nscm-solve-id}
    chooses the parameters for $\bm{\theta}^*_{\min}$ and $\bm{\theta}^*_{\max}$, $\widehat{M}$ will always induce the same result for $Q$ as long as it induces the correct $P(\*v)$. Specifically, through standard do-calculus derivation we would obtain: 
    \begin{align*}
        &P(B = 1 \mid do(D = 1)) &\\
        &= \sum_{s \in \cD_{S}} P(s \mid do(D = 1)) P(B = 1 \mid do(D = 1), s) & \text{Marginalization} \\ 
        &= \sum_{s \in \cD_{S}} P(s \mid do(D = 1)) P(B = 1 \mid do(D = 1), do(s)) & \text{Rule 2} \\ 
        &= \sum_{s \in \cD_{S}} P(s \mid do(D = 1)) P(B = 1 \mid  do(s)) & \text{Rule 3} \\ 
        &= \sum_{s \in \cD_{S}} P(s \mid D = 1) P(B = 1 \mid  do(s)) & \text{Rule 2} \\ 
        &= \sum_{s \in \cD_{S}} P(s \mid D = 1) \sum_{d \in \cD_{D}} P(B = 1 \mid d, do(s)) P(d \mid do(s)) & \text{Marginalization} \\
        &= \sum_{s \in \cD_{S}} P(s \mid D = 1) \sum_{d \in \cD_{D}} P(B = 1 \mid d, do(s)) P(d) & \text{Rule 3} \\
        &= \sum_{s \in \cD_{S}} P(s \mid D = 1) \sum_{d \in \cD_{D}} P(B = 1 \mid d, s) P(d) & \text{Rule 2}\\
        \end{align*}
        Finally, we can replace the corresponding probabilities with the actual values, i.e., 
        \begin{align*}
        P(B = 1 \mid do(D = 1))
        &= P(S = 0 \mid D = 1)P(B = 1 \mid D = 0, S = 0) P(D = 0) \\
        &+ P(S = 0 \mid D = 1)P(B = 1 \mid D = 1, S = 0) P(D = 1) \\
        &+ P(S = 1 \mid D = 1)P(B = 1 \mid D = 0, S = 1) P(D = 0) \\
        &+ P(S = 1 \mid D = 1)P(B = 1 \mid D = 1, S = 1) P(D = 1) \\
        &= \left(\frac{108}{144}\right) \left(\frac{15}{28}\right) \left(\frac{112}{256}\right) + \left(\frac{108}{144}\right) \left(\frac{27}{108}\right) \left(\frac{144}{256}\right) \\
        &+
        \left(\frac{36}{144}\right) \left(\frac{63}{84}\right) \left(\frac{112}{256}\right) + \left(\frac{36}{144}\right) \left(\frac{27}{36}\right) \left(\frac{144}{256}\right) \\
        &= \frac{15}{32} = 0.46875 \numberthis \label{eq:ex-id-docalc}
    \end{align*}
    On the other hand, \textit{any} $\cG$-constrained NCM that induces $P(\*v)$ will produce the same result, matching Eq.~\ref{eq:ex-id-docalc}. For instance, consider the following NCM construction:
    
    \begin{equation}
        \label{eq:ex-id-ncm}
        \widehat{M} := 
        \begin{cases}
            \widehat{\*U} &:= \{\widehat{U}_{DB}, \widehat{U}_S\}, \cD_{\widehat{U}_{DB}} = \cD_{\widehat{U}_S} = [0, 1] \\
            \*V &:= \{D, S, B\} \\
            \widehat{\cF} &:= 
            \begin{cases}
                \hat{f}_D(u_{DB}) &=
                \begin{cases}
                    1 & u_{DB} \geq \frac{112}{256} \\
                    0 & \text{otherwise}
                \end{cases}
                \\
                \hat{f}_S(d, u_S) &=
                \begin{cases}
                    1 & (d = 0) \wedge (u_S \geq \frac{1}{4}) \\
                    1 & (d = 1) \wedge (u_S \geq \frac{3}{4}) \\
                    0 & \text{otherwise}
                \end{cases}
                \\
                \hat{f}_B(b, u_{DB}) &=
                \begin{cases}
                    1 & (s = 0) \wedge \left((\frac{52}{256} \leq u_{DB} < \frac{112}{256}) \vee (u_{DB} \geq \frac{220}{256}) \right) \\
                    1 & (s = 1) \wedge \left((\frac{28}{256} \leq u_{DB} < \frac{112}{256}) \vee (u_{DB} \geq \frac{148}{256}) \right) \\
                    0 & \text{otherwise}
                \end{cases}
            \end{cases}\\
            P(\widehat{\*U}) &:= \widehat{U}_{DB}, \widehat{U}_S \sim \unif(0, 1), \widehat{U}_{DB} \indep \widehat{U}_S
        \end{cases}
    \end{equation}
    
    The corresponding neural network for  $\widehat{\cF}$ can be written as 
    \vspace{-0.15in}
    \begin{align}
        \hat{f}_D(u_{DB}) &= \hat{f}_{\geq 112/256}(u_{DB}) \\
        \hat{f}_S(d, u_S) &= \hat{f}_{\scor}\left(\hat{f}_{\scand} \left(\hat{f}_{\scnot}(d), \hat{f}_{\geq 1/4} (u_S) \right), \hat{f}_{\scand} \left(d, \hat{f}_{\geq 3/4} (u_S)\right)\right)\\
        \hat{f}_B(s, u_{DB}) &= 
        \hat{f}_{\scor}
        \begin{cases}
            \hat{f}_{\scand} \left(\hat{f}_{\scnot}(s), \hat{f}_{\geq 52/256}(u_{DB}), \hat{f}_{\scnot}\left(\hat{f}_{\geq 112/256}(u_{DB}) \right) \right) \\
            \hat{f}_{\scand} \left(\hat{f}_{\scnot}(s), \hat{f}_{\geq 220/256}(u_{DB}) \right) \\
            \hat{f}_{\scand} \left(s, \hat{f}_{\geq 28/256}(u_{DB}), \hat{f}_{\scnot}\left(\hat{f}_{\geq 112/256}(u_{DB}) \right) \right) \\
            \hat{f}_{\scand} \left(s, \hat{f}_{\geq 148/256}(u_{DB}) \right). \\
        \end{cases}
    \end{align}
    
    Note that $\widehat{M}$ follows the format of a $\cG$-constrained NCM defined in Def.~\ref{def:g-cons-nscm}. Additionally, one can verify that $\widehat{M}$ induces the same $L_1$ quantities shown in Table \ref{tab:ex-id-pv}. Applying the mutilation procedure on $\widehat{M}$ to compute the query, 
    \begin{align*}
        &P^{\widehat{M}}(B = 1 \mid do(D = 1)) \\
        &= P(S_{D = 1} = 0)P \left( \left(\frac{52}{256} \leq U_{DB} < \frac{112}{256} \right) \vee \left(U_{DB} \geq \frac{220}{256} \right) \right) \\
        &+ P(S_{D = 1} = 1)P\left( \left(\frac{28}{256} \leq U_{DB} < \frac{112}{256}\right) \vee \left(U_{DB} \geq \frac{148}{256}\right)\right) \\
        &= \left(\frac{96}{256}\right)P(S_{D = 1} = 0) + \left(\frac{192}{256}\right)P(S_{D = 1} = 1) \\
        &= \left(\frac{96}{256}\right)P\left(U_S < \frac{3}{4}\right) + \left(\frac{192}{256}\right)P\left(U_S \geq \frac{3}{4}\right) \\
        &= \frac{72}{256} + \frac{48}{256} = \frac{120}{256} = 0.46875.
    \end{align*}
    This result indeed matches Eq.~\ref{eq:ex-id-l2-query}. We further note $\widehat{M}$ is constructed to fit $P(\*v)$ and not necessarily to match $\cM^*$.
     This is evident when comparing the inputs and outputs of the functions in $\widehat{\cF}$ to those from $\cM^*$. Still, due to the structural constraints encoded in the NCM, it so happens that $\widehat{M}$ also matches in our $L_2$ query of interest even though it is only constructed to match only on layer 1. It is  remarkable that incorporating these family of constraints into the NCM structure allows one to successfully perform cross-layer inferences.
    
    \hfill $\blacksquare$
\end{example}

Thm.~\ref{thm:nscm-id-equivalence} is powerful because it says that performing the identification task in the class of \ncm{}s will yield the correct result, even if the true model can be any SCM. We note that this is not the case for every class of models, even if a model from that class can be both $L_1$-consistent and $\cG$-consistent with the true model. In particular, the expressivity of the \ncm{} allows for this result, and using a less expressive model may produce incorrect results. We illustrate this with the following examples.

\begin{example}
    \label{ex:nonexpr-markovian}
    Suppose we attempt to identify $P(B = 1 \mid do(D = 1))$ from the problem in Example \ref{ex:non-id} using a less expressive model class such as the set of all Markovian models. Recall that the given observational $P(\*V)$ is shown in Table \ref{tab:ex-ncht-pv}, and the given causal diagram $\cG$ is shown in Fig.~\ref{fig:ex-expr-basic-cg}. An example model from the class of Markovian models that achieve $L_1$ and $\cG$ consistency with $\cM^*$ is $\widehat{M}_1$ from Example \ref{ex:ncht}.
    
    Note that due to the lack of unobserved confounding in Markovian models, the induced value for $P(B = 1 \mid do(D = 1))$ for any $L_1$ and $\cG$-consistent Markovian model including $\widehat{M}_1$ is $P(B = 1 \mid do(D = 1)) = P(B = 1 \mid D = 1) = \frac{54}{144} = 0.375$. Since all models from this class (Markovian) agree on the same value for this quantity, the conclusion reached is that it must be identifiable. However, this is clearly not the case, as we know the true value, $P^{\cM^*}(B = 1 \mid do(D = 1)) = 0.46875$. In this case, the reason we reach the wrong result is that the true model is not Markovian, and the set of all Markovian models is not expressive enough to account for all possible SCMs.     \hfill $\blacksquare$
\end{example}

More interestingly, another example using a different class of models follows.

\begin{example}
    \label{ex:nonexpr}
    Consider a special class of SCMs which, given variables $\*V$ and the graph $\cG$, take the following structure.
    \begin{equation}
        \label{eq:ex-nonexpr-scm}
        \widehat{M} =
        \begin{cases}
            \widehat{\*U} &:= \{\widehat{U}_{\*C} : \*C \in C^2(\cG)\}, \cD_{\widehat{U}} = [0, 1] \text{ for all } \widehat{U} \in \widehat{\*U}, \\
            \widehat{\*V} &:= \*V, \\
            \widehat{\cF} &:= \left\{\hat{f}_{V}(\pai{V}, \ui{V}) = \left(\underset{x \in \pai{V}}{\bigoplus} x\right) \oplus \left(\underset{u_i \in \ui{V}}{\bigoplus} \mathbbm{1}(a_{V,i} \leq u_i < b_{V, i})\right) : V \in \*V\right\}, \\
            & a_{V, i}, b_{V, i} \in [0, 1], a_{V, i} \leq b_{V, i}, \\
            & \pai{V} \text{ is the set of parents of } V \text{ in } \cG, \\
            & \ui{V} = \{\widehat{U}_{\*C} : \*C \in C^2(\cG) \text{ such that } V \in \*C\}, \\
            P(\widehat{\*U}) &:= \widehat{U} \sim \unif(0, 1) \text{ for all } \widehat{U} \in \widehat{\*U}.
        \end{cases}
    \end{equation}
    Here, $C^2(\cG)$ denotes the set of all $C^2$-components of $\cG$. The $\oplus$ operator denotes bitwise XOR, with the larger version representing the bitwise XOR of all elements of a set (0 if empty).
    
    \begin{minipage}{\textwidth}
      \begin{minipage}[b]{0.49\textwidth}
        \centering
        \begin{tikzpicture}
    		\node[draw, circle] (X) at (-1, 0) {$X$};
    		\node[draw, circle] (Y) at (1, 0) {$Y$};
    
    		\path [-{Latex[length=2mm]}] (X) edge (Y);
    		\path [dashed, {Latex[length=2mm]}-{Latex[length=2mm]}, bend left] (X) edge (Y);
    	\end{tikzpicture}
    	\captionof{figure}{Causal diagram $\cG$ of $\cM^*$ from Example \ref{ex:nonexpr}}
    	\label{fig:ex-nonexpr-cg}
      \end{minipage}
      \hfill
      \begin{minipage}[b]{0.49\textwidth}
        \centering
        \begin{tabular}{ll|l}
        \hline \hline
        $X$ & $Y$ & $P(X, Y)$ \\ \hline \hline
        0   & 0   & $p_0$  \\ \hline
        0   & 1   & $p_1$  \\ \hline
        1   & 0   & $p_2$  \\ \hline
        1   & 1   & $p_3$  \\ \hline
        \end{tabular}
        \captionof{table}{General observational distribution $P(\*V)$ induced by $\cM^*$ from Example \ref{ex:nonexpr}. We assume positivity (i.e.~$p_i > 0$ for all $i$).}
        \label{tab:ex-nonexpr-pv}
        \end{minipage}
      \end{minipage}
    
We may opt to use a model like this for its simplicity, since having fewer parameters allows for easier optimization. Suppose we try to use this model class to decide if $P(Y \mid do(X))$ is identifiable in the case where the true model, $\cM^*$, induces the graph $\cG$ in Fig.~\ref{fig:ex-nonexpr-cg}. We are given $\cG$ along with $P(\*v)$, which can be represented as shown in Table \ref{tab:ex-nonexpr-pv}. If we construct a model $\widehat{M}$ from Eq.~\ref{eq:ex-nonexpr-scm}, it would have the following form.
    
    \begin{equation*}
        \widehat{M} =
        \begin{cases}
            \widehat{\*U} &:= \{\widehat{U}_{XY}\}, \cD_{\widehat{U}_{XY}} = [0, 1],  \\
            \widehat{\*V} &:= \{X, Y\} \\
            \widehat{\cF} &:=
            \begin{cases}
                \hat{f}_X(\widehat{u}_{XY}) &= \mathbbm{1}(a_{X} \leq \widehat{u}_{XY} < b_{X}), \\
                \hat{f}_Y(x, \widehat{u}_{XY}) &= x \oplus \mathbbm{1}(a_{Y} \leq \widehat{u}_{XY} < b_{Y})
            \end{cases}
            \\
            P(\widehat{\*U}) &:= U_{XY} \sim \unif(0, 1).
        \end{cases}
    \end{equation*}
    
    The parameters we choose to fit $P(\*V)$ are the values of $a_X, b_X, a_Y, b_Y$. However, note that there in order for $\widehat{M}$ to attain $L_1$-consistency with $\cM^*$, we must have:
    
    \begin{enumerate}
        \item $b_X - a_X = p_2 + p_3$.
        \item ($a_X - a_Y = p_1$ and $b_Y - a_X = p_2$) or ($b_Y - b_X = p_1$ and $b_X - a_Y = p_2$)
    \end{enumerate}
    
    2 implies $b_Y - a_Y = p_1 + p_2$, so in all cases, $P^{\widehat{M}}(Y = 1 \mid do(X = 0)) = p_1 + p_2$ and $P^{\widehat{M}}(Y = 1 \mid do(X = 1)) = p_0 + p_3$. In other words, since $P^{\widehat{M}}(Y \mid do(X))$ matches for all models from this class that are $L_1$ and $\cG$-consistent, the conclusion reached is that it must be identifiable.

    However, suppose $\cM^*$ takes the following form:
    \begin{equation*}
        \cM^* =
        \begin{cases}
            \*U &:= \{U_X, U_Y\}, \cD_{U_X} = \cD_{U_Y} = \{0, 1\},  \\
            \*V &:= \{X, Y\} \\
            \cF &:=
            \begin{cases}
                f_X(u_{X}) &= u_X, \\
                f_Y(x, u_{Y}) &= u_Y,
            \end{cases}
            \\
            P(\*U) &:=
            \begin{cases}
                P(U_X = 0, U_Y = 0) &= p_0 \\
                P(U_X = 0, U_Y = 1) &= p_1 \\
                P(U_X = 1, U_Y = 0) &= p_2 \\
                P(U_X = 1, U_Y = 1) &= p_3
            \end{cases}
        \end{cases}
    \end{equation*}
    One can quickly verify that $\cM^*$ does indeed induce the $P(\*V)$ in Table \ref{tab:ex-nonexpr-pv} and $\cG$ from Fig.~\ref{fig:ex-nonexpr-cg}. Note that in this $\cM^*$, $P(Y = 1 \mid do(X = 0)) = P(Y = 1 \mid do(X = 1)) = p_1 + p_3$. In fact, this means there is a complete mismatch from any possible choice of $\widehat{M}$. Like shown in Example \ref{ex:non-id}, $P(Y \mid do(X))$ is actually a non-identifiable query when $\cG$ takes the form in Fig.~\ref{fig:ex-nonexpr-cg}.
    Simply put, $\cM^*$ could be an SCM for which the model class in Eq.~\ref{eq:ex-nonexpr-scm} is not expressive enough to capture.
    
    \hfill $\blacksquare$
\end{example}

\newpage
\subsection{Symbolic versus Optimization-based Approaches for Identification}\label{app:examples-symbolic-neural}

In this section, we discuss the relationship between current approaches to causal identification/estimation versus the new, neural/optimization-based approach proposed in this work. 

The problem of effect identification has been extensively studied in the literature, and \citep{pearl:95a} introduced the \textit{do-calculus} (akin to differential or integral calculus), a set of symbolic rules that can be applied to any expression and evaluate whether a certain invariance across interventional distributions hold given local assumptions encoded in a causal diagram. Multiple applications of the rules can be combined to search for a reduction from a target effect to the distributions in which data is available. 
 There exist  algorithms capable of utilizing the structural constraints encoded in the causal diagram, based on what is known as \textit{c-component factorization} \citep{tian:pea02-general-id}, to find a symbolic expression of the $L_2$ query in terms of the $L_1$ distribution efficiently (for more general tasks, see, e.g.,  \cite{lee:etal19,correa2020gtr}). We call this approach  \textit{symbolic} since it aims to find a closed-form expression of the target effect in terms of the input distributions.

Alternatively, one could take an optimization-based route to tackle the identification problem, such as in Alg.~\ref{alg:nscm-solve-id} discussed in Sec.~\ref{sec:neural-id}. This approach entails a search through the space of possible structural models while trying to maximize/minimize the value of the target query subject to the constraints found in the inputted data. We call this an \textit{optimization-based approach}.  These two approaches are indeed linked as evident from Thm.~\ref{thm:dual-graph-id}, which establishes a duality  saying that the identification status of a target query is shared across symbolic and optimization-based approaches.  

We discuss next some of the possible trade-offs and synergies between these two families of methods. We start with the symbolic approach and re-stating the definition of identification \citep{pearl:2k}, with minor modifications to highlight its model-theoretic perspective regarding the space of SCMs:

\begin{definition}[Causal Effect Identification]
    \label{def:classic-id}
    Let $\Omega^*$ be the space containing all SCMs defined over endogenous variables $\*V$. We say that a causal effect $P(\*y \mid do(\*x))$ is identifiable from the observational distribution $P(\*v)$ and the causal diagram $\cG$ if $P^{\cM_1}(\*y \mid do(\*x)) = P^{\cM_2}(\*y \mid do(\*x))$ for every pair of models $\cM_1, \cM_2 \in \Omega^*$ such that $\cM_1$ and $\cM_2$ both induce $\cG$ and $P^{\cM_1}(\*v) = P^{\cM_2}(\*v)$
    \hfill $\blacksquare$
\end{definition}

\begin{wrapfigure}{r}{0.45\textwidth}
    \begin{tikzpicture}[scale=0.75, every node/.append style={transform shape}]
        \filldraw [fill=gray!50, line width=0.01mm] (-2.05,-0.05) ellipse (2.0 and 1.0);
      	\node [align=center, font=\fontsize{10}{0}\selectfont] at (-2.1,,-0.35) {$\bm{\Omega}^*$};
      	\node [inner sep=0] (omg) at (-1.0,0.85) {};
      	\node [inner sep=0] (omgspace) at (-1.0,0.4) {};
      	
      	\node at (-2.75,-0.15) {$\cM^1$};
      	\draw [fill=black] (-2.75,-0.5) circle (0.1);
      	\node [inner sep=0] (ms) at (-2.75,-0.5) {\ };
      	\node at (-1.4,0.45) {$\cM^2$};
      	\draw [fill=black] (-1.4,0.1) circle (0.1);
      	\node [inner sep=0] (mp) at (-1.4,0.1) {};
      	
      	\draw [densely dotted] (-1.9,0.5) circle (0.1);
      	\draw [densely dotted] (-2.3,0.2) circle (0.1);
      	\draw [densely dotted] (-2.05,-0.15) circle (0.1);
      	\draw [densely dotted] (-0.8,0.15) circle (0.1);
      	\draw [densely dotted] (-1.15,-0.15) circle (0.1);
      	
      	\draw [densely dotted] (-3.0,0.4) circle (0.1);
        \draw [densely dotted] (-3.5,0.25) circle (0.1);
        \draw [densely dotted] (-3.25,0.0) circle (0.1);
        \draw [densely dotted] (-3.8, 0.0) circle (0.1);
        \draw [densely dotted] (-3.4,-0.4) circle (0.1);
        \draw [densely dotted] (-2.2,-0.8) circle (0.1);
        \draw [densely dotted] (-1.7,-0.7) circle (0.1);

      	\draw (-2.05,-2.32) ellipse (1.77 and 0.8);
      	\node [align=center, font=\fontsize{8}{0}\selectfont] (l1) at (-2.95,-3.6) {($\Ll_1$)};
      	\node [align=center, font=\fontsize{8}{0}\selectfont] (l1) at (-1.8,-3.6) {Observational\\ Distributions\\ Data};
      	\node at (-2.05,-2.4) {$P^{\cM_1}(\*V){=}P^{\cM_2}(\*V)$};
      	\draw [fill=black] (-2.05,-2.1) circle (0.05);
      	\node [inner sep=0.2em] (pl1) at (-2.05,-2.1) {\ };
    
      	\draw (2.7,-2.32) ellipse (1.77 and 0.8);
      	\node [align=center, font=\fontsize{8}{0}\selectfont] (l2) at (1.9,-3.6) {($\Ll_2$)};
      	\node [align=center, font=\fontsize{8}{0}\selectfont] (l2t) at (3.0,-3.6) {Interventional\\ Distributions\\ Query};
      	\draw [densely dotted] (2.7,-2.32) ellipse (1.4 and 0.55);
      	\node [align=center] at (2.7,-2.4) {{\small $P^{\cM_1}(\*Y | do(\*x))=$} \\ {\small $P^{\cM_2}(\*Y | do(\*x))$}};
    
      	\draw [fill=black] (2.7,-1.9) circle (0.05);
      	\node [inner sep=0.2em] (pl21) at (2.7,-1.9) {\ };

      	\path [-Latex] (ms) edge (pl1);
      	\path [-Latex] (ms) edge (pl21);
      	\path [-Latex] (mp) edge (pl1);
      	\path [-Latex] (mp) edge (pl21);
      	\path [-Latex] (l1) edge node[above]  {?} (l2);

      	\draw (2.9,-0.05) ellipse (1 and 0.6);
      	\node [align=center, font=\fontsize{10}{0}\selectfont] at (1.1,1.05) {Structural Constraints};
      	\node [align=center, font=\fontsize{9}{0}\selectfont] at (3.0,-0.9) {Causal Diagram};
      	\draw [densely dotted] (2.8,-0.05) ellipse (0.7 and 0.45);
      	\node [align=center] at (3.0,0) {$\cG$};
      	\draw [fill=black] (2.7,0) circle (0.05);
      	\node [inner sep=0.2em] (g1) at (2.7,0) {\ };
      	\path [-Latex] (mp) edge [bend left=25] (g1);
      	\path [-Latex] (ms) edge (g1);
    \end{tikzpicture}
    \caption{$P(\*Y \mid do(\*x))$ is identifiable from $P(\*V)$  and $\cG$ if for all SCM $\cM^1, \cM^2$ (top left) such that $\cM^1, \cM^2$ match in $P(\*V)$ (bottom left) and $\cG$ (top right), then they also match in the target distribution $P(\*Y \mid do(\*x))$ (bottom right).}
    \label{fig:omegastar2}
    \vspace{-0.2in}
\end{wrapfigure}
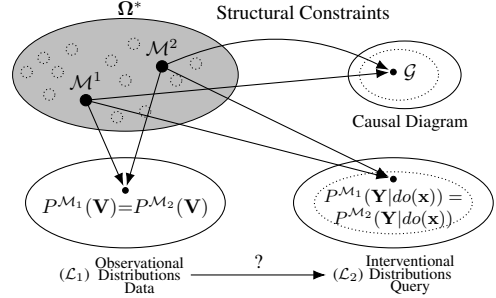

Fig.~\ref{fig:omegastar2} provides an illustration of the many parts involved in this  definition. In words, an interventional distribution $Q = P(\*Y \mid do(\*x))$ is said to be  identifiable if for all  SCMs in $\Omega^*$ (top-left) that share the same causal diagram $\cG$ (top-right), and induce the same probability distribution $P(\*V)$ (bottom-left), they generate the same  distribution to the target query $Q$ (bottom-right). In fact, identifiability can be understood as if the details of the specific form of the true SCM $\cM^*$ -- its functions and probability distribution over the exogenous variables -- are irrelevant,  
 and the constraints encoded in the causal diagram are sufficient to perform the intended cross-layer inference  from the source to the target distributions.

One important observation that follows is that, operationally, symbolic methods do not work directly in the $\Omega^*$ space, but on top of the constraints implied by the true SCM on the causal diagram.
There are a number of reasons for this, but we note that if all that is available about $\cM^*$ are the constraints encoded in $\cG$, it is somewhat expected and reasonable to operate directly on those, instead of considering the elusive, underlying structural causal model. Further, we already know through the CHT that we will never be able to recover $\cM^*$ anyways, so one could see it  as unreasonable to consider $\cM^*$ as the target of the analysis. 
Still, even if we wanted to refer directly to $\cM^*$, 
practically speaking, 
searching through the space $\Omega^*$ is a daunting task since each candidate SCM $\cM = \langle \cF, P(\*U) \rangle \in \Omega^*$ is a complex, infinite dimensional object.  
Note that the very definition of identification (Def.~\ref{def:classic-id}) does not even mention the true SCM $\cM^*$, since it is completely out of reach in most practical situations.  The task here is indeed about whether one can get by without having to know much about the  underlying collection of mechanisms and still answer the query of interest. (In causal terminology, an identification instance is called \textit{non-parametric}
 whenever no constraints are imposed over the functional class $\cF$ or exogenous distribution $P(\*U)$.

The full causal pipeline encompassing both tasks of effect identification and estimation is illustrated in Fig.~\ref{fig:causal-pipeline}. Note that the detachment with respect to the true Nature, alluded to above, is quite prominent in the figure. The very left part contains  $\cM^*$, which is abstracted away in the form of the causal diagram $\cG$ and through its marks imprinted into the observational distribution $P(\*V)$. The symbolic methods takes the pair $\langle \cG, P(\*V) \rangle$ as input (instead of the unobservable $\langle \cF^*, P^*(\*U) \rangle$), and attempt to determine whether the target distribution $Q$ can be expressed in terms of the available input distributions. Formally, the question asked is about whether there exists a mapping $f$ such that
\begin{eqnarray} \label{eq:mapid} 
Q = f_{\cG}(P(\*V)). 
\end{eqnarray}

\begin{figure}[t]
    \begin{center}
        \begin{tikzpicture}[-, scale=0.8, every node/.append style={transform shape}]
            \draw [line width=0.1mm] (-8,3.5) -- (8,3.5);
            \draw [line width=0.3mm, densely dotted] (-5.5,4.5) -- (-5.5,-3.5);
            \draw [line width=0.3mm, densely dotted] (0.4,4.5) -- (0.4,-3.5);
            
            \node [align=center, font=\fontsize{11}{0}\selectfont] (unobservedreality1) at (-6.75, 4.2) {\textbf{Unobserved}};
            \node [align=center, font=\fontsize{11}{0}\selectfont] (unobservedreality2) at (-6.75, 3.8) {\textbf{Reality}};
            \node [align=center, font=\fontsize{14}{0}\selectfont] (causalidentification) at (-2.475, 4) {\textbf{Causal Identification}};
            \node [align=center, font=\fontsize{14}{0}\selectfont] (causalestimation) at (4.2, 4) {\textbf{Causal Estimation}};

            \filldraw [fill=gray!50, line width=0.01mm] (-6.75, 0) ellipse (1.1 and 1.5);
            \filldraw [fill=gray!20, line width=0.01mm] (-6.5, 0.25) ellipse (0.5 and 0.75);
            
            \fill [fill=black] (-6.25, 0) circle (0.1);
            \node [inner sep=0] (mstarright) at (-6.15,0) {};
            \node [align=center, font=\fontsize{8}{0}\selectfont] (mstarlabel) at (-6.3,0.3) {$\cM^*$};
            
            \draw [densely dotted] (-6.9, 1.275) circle (0.1);
            \draw [densely dotted] (-7.4, 0.825) circle (0.1);
            \draw [densely dotted] (-7.2, 0.5) circle (0.1);
            \draw [densely dotted] (-7.6, 0.075) circle (0.1);
            \draw [densely dotted] (-7.2, -0.45) circle (0.1);
            \draw [densely dotted] (-6.8, -0.675) circle (0.1);
            \draw [densely dotted] (-7.1, -0.9) circle (0.1);
            \draw [densely dotted] (-6.3, -0.9) circle (0.1);
            \draw [densely dotted] (-6.7, -1.125) circle (0.1);
            
            \draw [densely dotted] (-6.4, 0.7375) circle (0.1);
            \draw [densely dotted] (-6.7, 0.6625) circle (0.1);
            \draw [densely dotted] (-6.8, 0.3625) circle (0.1);
            \draw [densely dotted] (-6.6, 0.0625) circle (0.1);
            \draw [densely dotted] (-6.4, -0.3125) circle (0.1);
            \draw [densely dotted] (-6.7, -0.2375) circle (0.1);
            
            \node [align=center, font=\fontsize{10}{0}\selectfont] (spaceofall) at (-6.75,-1.8) {Space of All};
            \node [align=center, font=\fontsize{10}{0}\selectfont] (scmsomegastar) at (-6.75,-2.15) {SCMs ($\Omega^*$)};

            \draw (-4.35, 1.5) ellipse (0.8 and 1.05);
            \draw [densely dotted] (-4.35, 1.8) ellipse (0.5 and 0.6);
            \fill [fill=black] (-4.35, 1.7) circle (0.1);
            \node [inner sep=0] (pstarvleft) at (-4.45, 1.7) {};
            \node [inner sep=0] (pstarvright) at (-4.25, 1.7) {};
            \node [align=center, font=\fontsize{8}{0}\selectfont] (pstarvlabel) at (-4.35, 2.0) {$P^*(\*V)$};
            \node [align=center, font=\fontsize{10}{0}\selectfont] (spaceofpv1) at (-4.35, 3.15) {Space of};
            \node [align=center, font=\fontsize{10}{0}\selectfont] (spaceofpv2) at (-3.9, 2.8) {Obs. Distributions};
            
            \draw(-4.35, -1.5) ellipse (0.8 and 1.05);
            \draw [densely dotted] (-4.35, -1.8) ellipse (0.5 and 0.6);
            \fill [fill=black] (-4.35, -1.7) circle (0.1);
            \node [inner sep=0] (gstarleft) at (-4.45, -1.7) {};
            \node [inner sep=0] (gstarright) at (-4.25, -1.7) {};
            \node [align=center, font=\fontsize{8}{0}\selectfont] (gstarlabel) at (-4.35, -2.0) {$\cG^*$};
            \node [align=center, font=\fontsize{10}{0}\selectfont] (spaceofg1) at (-4.35, -2.8) {Space of};
            \node [align=center, font=\fontsize{10}{0}\selectfont] (spaceofg2) at (-4.0, -3.15) {Causal Diagrams};
            
            \draw [rounded corners, double, double distance=1] (-3, 0.65) rectangle (0.2, -0.65);
            \node [inner sep=0] (task1left) at (-3, 0) {};
            \node [inner sep=0] (task1topleft) at (-3, 0.65) {};
            \node [inner sep=0] (task1bottomleft) at (-3, -0.65) {};
            \node [inner sep=0] (task1bottom) at (-1.4, -0.65) {};
            \node [align=center, font=\fontsize{8}{0}\selectfont] (querypydox) at (-1.4, 0.2) {Query $P^*(Y \mid do(X))$};
            \node [align=center, font=\fontsize{8}{0}\selectfont] (idfromgpv) at (-1.4, -0.2) {ID from $\langle \cG^*, P^*(\*V) \rangle?$};
            \node [align=center, font=\fontsize{12}{0}\selectfont] (task1label) at (-1.4, 1) {\textbf{Task 1}};
            \filldraw [fill=black!40!red, rounded corners, line width=0.1mm] (-2.6, -1.3) rectangle (-1.6, -1.9);
            \node [inner sep=0] (notop) at (-2.1, -1.3) {};
            \node [align=center, font=\fontsize{8}{0}\selectfont, text=white] (textno) at (-2.1, -1.6) {\textbf{No}};
            
            \filldraw [fill=black!20!green, rounded corners, line width=0.1mm] (-1.2, -1.3) rectangle (-0.2, -1.9);
            \node [inner sep=0] (yestop) at (-0.7, -1.3) {};
            \node [inner sep=0] (yesright) at (-0.2, -1.6) {};
            \node [align=center, font=\fontsize{8}{0}\selectfont, text=white] (textyes) at (-0.7, -1.6) {\textbf{Yes}};
            
            \node [align=center, font=\fontsize{10}{0}\selectfont] (sampling) at (-2, 2.0) {Sampling};
            
            \path [-Latex, line width=0.3mm] (mstarright) edge (pstarvleft);
            \path [-Latex, line width=0.3mm] (mstarright) edge (gstarleft);
            \path [-Latex, line width=0.3mm] (pstarvright) edge (task1topleft);
            \path [-Latex, line width=0.3mm] (gstarright) edge (task1bottomleft);
            \path [-Latex, line width=0.3mm] (task1bottom) edge (notop);
            \path [-Latex, line width=0.3mm] (task1bottom) edge (yestop);

            \filldraw [fill=gray!30, line width=0.1mm] (1.4, 2.5) rectangle (2.6, 1.3);
            \draw [line width=0.1mm] (1.5, 2.3) -- (2.5, 2.3);
            \draw [line width=0.1mm] (1.5, 2.1) -- (2.5, 2.1);
            \draw [line width=0.1mm] (1.5, 1.9) -- (2.5, 1.9);
            \draw [line width=0.1mm] (1.5, 1.7) -- (2.5, 1.7);
            \draw [line width=0.1mm] (1.5, 1.5) -- (2.5, 1.5);
            \node [inner sep=0] (databottom) at (2, 1.3) {};
            \node [inner sep=0] (dataleft) at (1.4, 1.7) {};
            \node [align=center, font=\fontsize{10}{0}\selectfont] (causalestimand) at (2, 2.8) {Dataset $D^*(\*V)$};
            
            \draw [rounded corners, line width=0.1mm] (0.6, -1.3) rectangle (4.6, -1.9);
            \node [inner sep=0] (estimandleft) at (0.6, -1.6) {};
            \node [inner sep=0] (estimandtop) at (2, -1.3) {};
            \node [align=center, font=\fontsize{8}{0}\selectfont] (pydoxfpv) at (2.6, -1.6) {$P^*(Y \mid do(X)) = f(P^*(\*V))$};
            \node [align=center, font=\fontsize{10}{0}\selectfont] (causalestimand) at (2.6, -2.2) {Causal Estimand};
            
            \draw [rounded corners, double, double distance=1] (2.6, 0.65) rectangle (5.8, -0.65);
            \node [inner sep=0] (task2left) at (2.6, 0) {};
            \node [inner sep=0] (task2topleft) at (2.6, 0.65) {};
            \node [inner sep=0] (task2bottomleft) at (2.6, -0.65) {};
            \node [inner sep=0] (task2right) at (5.8, 0) {};
            \node [align=center, font=\fontsize{10}{0}\selectfont] (fitestimand) at (4.2, 0.2) {Fit Estimand $f$};
            \node [align=center, font=\fontsize{10}{0}\selectfont] (usingdata) at (4.2, -0.2) {using Data $D^*(\*V)$};
            \node [align=center, font=\fontsize{12}{0}\selectfont] (task2label) at (4.2, 1) {\textbf{Task 2}};
            
            \filldraw [fill=gray!20, line width=0.1mm] (6.4, 0.7) rectangle (7.8, -0.7);
            \node [inner sep=0] (empdistleft) at (6.4, 0) {};
            \node [align=center, font=\fontsize{10}{0}\selectfont] (empdist1) at (7.1, 1.35) {Empirical};
            \node [align=center, font=\fontsize{10}{0}\selectfont] (empdist2) at (7.1, 1) {Distribution};
            \draw [line width=0.4mm] (6.6, 0.5) -- (6.6, -0.5);
            \draw [line width=0.4mm] (6.6, -0.5) -- (7.6, -0.5);
            \draw [line width=0.3mm, densely dashed] (7.1, 0.2) -- (7.1, -0.5);
            \draw [line width=0.3mm] (6.6, -0.5) to[out=10, in=190] (7.1, 0.2);
            \draw [line width=0.3mm] (7.6, -0.5) to[out=170, in=-10] (7.1, 0.2);
            
            \path [-Latex, line width=0.3mm] (pstarvright) edge (dataleft);
            \path [-Latex, line width=0.3mm] (yesright) edge (estimandleft);
            \path [-Latex, line width=0.3mm] (databottom) edge (task2topleft);
            \path [-Latex, line width=0.3mm] (estimandtop) edge (task2bottomleft);
            \path [-Latex, line width=0.3mm] (task2right) edge (empdistleft);
        \end{tikzpicture}
    \end{center}
        
    \caption{Causal Pipeline with unobserved SCM $\cM^*$ in the left, generating both $\cG$ and $P(\*v)$, which is taken as input for the identification task (1), which generates input to the estimation task (2). } 
    \label{fig:causal-pipeline}
\end{figure}
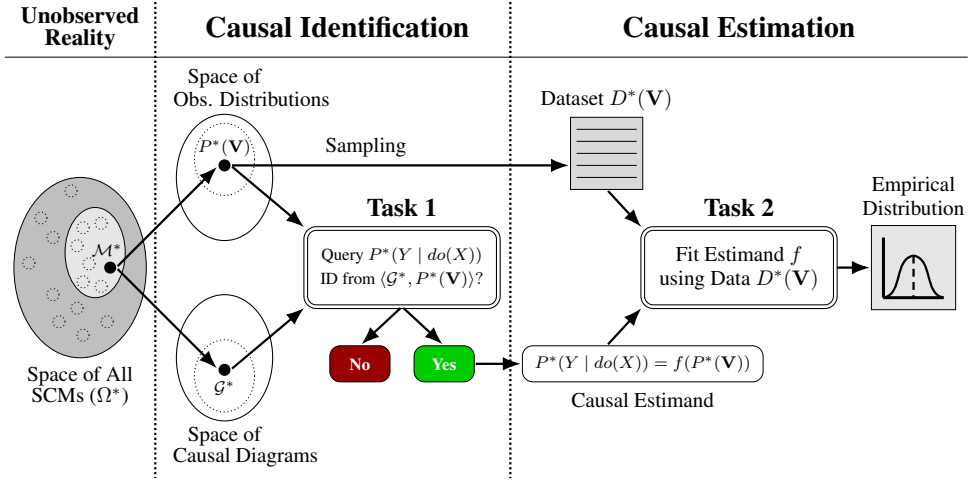

The decision problem regarding the existence of $f(.)$ is certainly within the domain of causal reasoning since it takes an arbitrary causal diagram as input, which encodes causal assumptions about the underlying SCM, and reason through causal axioms on whether this is sufficient to perform the intended cross-layer inference. Causal reasoning is nothing more than the manipulation and subsequent derivation of new causal facts from known causal invariances.
Whenever this step is successfully realized, one can  then ignore the causal invariances (or assumptions) entirely, and simply 
 \begin{wrapfigure}{r}{0.45\textwidth}
    \vspace{-0.15in}
    \begin{center}
        \begin{tikzpicture}[-, scale=0.9, every node/.append style={transform shape}]
            \filldraw [fill=gray!50, line width=0.01mm] (-6.75, 0) ellipse (1.3 and 1.8);
            \filldraw [fill=gray!20, line width=0.01mm] (-6.5, 0.25) ellipse (0.9 and 1.2);
            
            \fill [fill=black] (-6.2, 0.8375) circle (0.1);
            \node [inner sep=0] (mstarright) at (-6.1,0.8375) {};
            
            \draw [densely dotted] (-6.9, 1.5) circle (0.1);
            \draw [densely dotted] (-7.5, 0.925) circle (0.1);
            \draw [densely dotted] (-7.8, 0.075) circle (0.1);
            \draw [densely dotted] (-7.45, -0.45) circle (0.1);
            \draw [densely dotted] (-7.7, -0.775) circle (0.1);
            \draw [densely dotted] (-7.2, -1) circle (0.1);
            \draw [densely dotted] (-6.3, -1.2) circle (0.1);
            \draw [densely dotted] (-6.7, -1.425) circle (0.1);
            
            \draw (-6.7, 0.7625) circle (0.1);
            \draw (-7.0, 0.3625) circle (0.1);
            \draw (-5.9, 0.2) circle (0.1);
            \node [inner sep=0] (mright) at (-5.8,0.2) {};
            \draw (-6.6, 0.0625) circle (0.1);
            \draw (-6.9, -0.2375) circle (0.1);
            \draw (-6, -0.3) circle (0.1);
            \node [inner sep=0, minimum width=0.2cm, minimum height=0.2cm] (mprimeright) at (-6,-0.3) {};
            \draw (-6.4, -0.5125) circle (0.1);
            \node [inner sep=0] (subspacebottom) at (-6.5,-0.8) {};
            
            \draw [rounded corners] (-4.9, 1.2) rectangle (-2.1, 0.8);
            \node [align=center, font=\fontsize{8}{0}\selectfont] (mstarlabel) at (-3.5, 1) {$\cM^* = \langle \cF^*, P^*(\*U) \rangle$};
            \node [align=center, font=\fontsize{8}{0}\selectfont] (mstarlabel2) at (-3.5, 0.6) {(unobserved truth)};
            \node [align=center, font=\fontsize{8}{0}\selectfont] (mlabel) at (-3.5, 0.1) {$\cM = \langle \cF, P(\*U) \rangle$};
            \node [align=center, font=\fontsize{8}{0}\selectfont] (mprimelabel) at (-3.5, -0.5) {$\cM' = \langle \cF', P'(\*U) \rangle$};
            
            \node [align=center, font=\fontsize{10}{0}\selectfont] (omegastarlabel) at (-6.75, 2.1) {Space of All SCMs ($\Omega^*$)};
            
            \node [align=center, font=\fontsize{8}{0}\selectfont] (subspacelabel) at (-3.5, -1.3) {Space of SCMs that};
            \node [align=center, font=\fontsize{8}{0}\selectfont] (subspacelabel2) at (-3.5, -1.6) {induce $\cG^*$ and $P^*(\*V)$};
            \node [align=center, font=\fontsize{8}{0}\selectfont] (subspacelabel3) at (-3.5, -1.9) {($\Omega^*(\cG^*, P^*(\*V))$)};
            
            \path [-, line width=0.2mm] (mstarright.center) edge (mstarlabel.west);
            \path [-, line width=0.2mm] (mright.center) edge (mlabel.west);
            \path [-, line width=0.2mm] (mprimeright.-10) edge (mprimelabel.178);
            \path [-, line width=0.2mm] (subspacebottom) edge (subspacelabel.west);
        \end{tikzpicture}
        \caption{
        In the case where $Q$ is identifiable, any two SCMs in $\Omega^*(\cG^*, P^*(\*V))$ -- the space of SCMs matching $\cM^*$ in $P^*(\*V)$ and $\cG^*$ -- will also match $\cM^*$ in $Q$.
        }
        \label{fig:ncm-id1}
    \end{center}
    \vspace{-0.3in}
\end{wrapfigure}
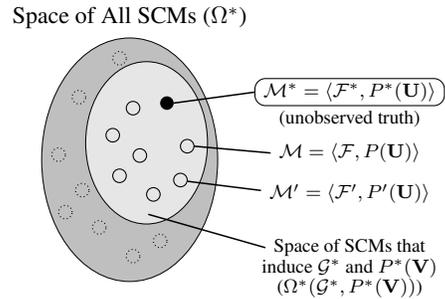
try to evaluate the r.h.s. of Eq.~\ref{eq:mapid} to compute $Q$.

We note that in most practical settings, only a dataset with samples collected from  $P(\*V)$ is available, say $D(\*V)$, as opposed to the distribution itself.
 This entails the second task that asks for a computationally and statistically attractive way of evaluating the r.h.s. of Eq.~\ref{eq:mapid} from the finite  samples contained in the dataset $D(\*V)$. Even though the l.h.s. of Eq.~\ref{eq:mapid} is a causal distribution, the evaluation is complete 
 oblivious to the semantics of such quantity and is entirely about fitting $f$ using the data $D(\*V)$ in the best possible way. 

We now turn our attention to the optimization-based approach as outlined in Alg.~\ref{alg:nscm-solve-id} (Sec.~\ref{sec:neural-id}) and note that it will have a very different interpretation of the identifiability definition.  
In fact, instead of avoiding the SCM altogether and focusing on the
causal diagrams' constraints, as done by the symbolic approach, it will put the SCM at the front and center of the analysis.
Fig.~\ref{fig:ncm-id1} illustrates this point by showing the space of all SCMs called $\Omega^*$ (in dark gray). The true, unknown SCM $\cM^*$ is shown as a black dot and generates the pair $\langle \cG^*, P^*(\*V)$, the latter is taken as the input of the identification analysis. The approach will then focus on the set of SCMs that have the same interventional constraints as $\cG^*$ (Def.~\ref{def:g-consistency}) and that also have the capability of generating the same observational distribution $P^*(\*V)$ (Def.~\ref{def:li-consistency}). This subspace is called $\Omega^*(\cG^*, P^*(\*V))$ and marked in light gray. Since $\cM^*$ is almost never inferrable, the optimization-approach will search for two SCMs $\cM, \cM' \in \Omega^*(\cG^*, P^*(\*V))$ such that the former will try to maximize the target query ($Q_{max}$), while the latter minimizes it ($Q_{min}$). Two candidate SCMs for this job are shown in the figure. Whenever the search ends and two of such SCMs are discovered, they will predict exactly the same interventional distribution (i.e., $Q_{min} =Q_{max}$) if $Q$ is identifiable. This is because by the definition of identifiability (Def.~\ref{def:classic-id}), all SCMs in the light gray area will necessarily exhibited the same interventional behavior. 

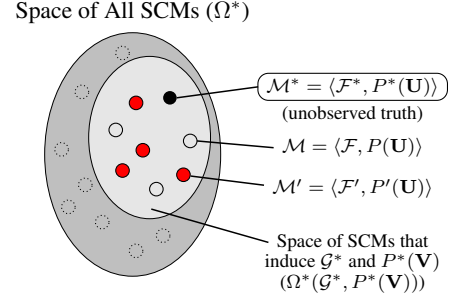
\begin{wrapfigure}{r}{0.45\textwidth}
    \vspace{-0.15in}
    \begin{center}
        \begin{tikzpicture}[-, scale=0.9, every node/.append style={transform shape}]
            \filldraw [fill=gray!50, line width=0.01mm] (-6.75, 0) ellipse (1.3 and 1.8);
            \filldraw [fill=gray!20, line width=0.01mm] (-6.5, 0.25) ellipse (0.9 and 1.2);
            
            \fill [fill=black] (-6.2, 0.8375) circle (0.1);
            \node [inner sep=0] (mstarright) at (-6.1,0.8375) {};
            
            \draw [densely dotted] (-6.9, 1.5) circle (0.1);
            \draw [densely dotted] (-7.5, 0.925) circle (0.1);
            \draw [densely dotted] (-7.8, 0.075) circle (0.1);
            \draw [densely dotted] (-7.45, -0.45) circle (0.1);
            \draw [densely dotted] (-7.7, -0.775) circle (0.1);
            \draw [densely dotted] (-7.2, -1) circle (0.1);
            \draw [densely dotted] (-6.3, -1.2) circle (0.1);
            \draw [densely dotted] (-6.7, -1.425) circle (0.1);
            
            \filldraw [fill=red](-6.7, 0.7625) circle (0.1);
            \draw (-7.0, 0.3625) circle (0.1);
            \draw (-5.9, 0.2) circle (0.1);
            \node [inner sep=0] (mright) at (-5.8,0.2) {};
            \filldraw [fill=red] (-6.6, 0.0625) circle (0.1);
            \filldraw [fill=red] (-6.9, -0.2375) circle (0.1);
            \filldraw [fill=red] (-6, -0.3) circle (0.1);
            \node [inner sep=0, minimum width=0.2cm, minimum height=0.2cm] (mprimeright) at (-6,-0.3) {};
            \draw (-6.4, -0.5125) circle (0.1);
            \node [inner sep=0] (subspacebottom) at (-6.5,-0.8) {};
            
            \draw [rounded corners] (-4.9, 1.2) rectangle (-2.1, 0.8);
            \node [align=center, font=\fontsize{8}{0}\selectfont] (mstarlabel) at (-3.5, 1) {$\cM^* = \langle \cF^*, P^*(\*U) \rangle$};
            \node [align=center, font=\fontsize{8}{0}\selectfont] (mstarlabel2) at (-3.5, 0.6) {(unobserved truth)};
            \node [align=center, font=\fontsize{8}{0}\selectfont] (mlabel) at (-3.5, 0.1) {$\cM = \langle \cF, P(\*U) \rangle$};
            \node [align=center, font=\fontsize{8}{0}\selectfont] (mprimelabel) at (-3.5, -0.5) {$\cM' = \langle \cF', P'(\*U) \rangle$};
            
            \node [align=center, font=\fontsize{10}{0}\selectfont] (omegastarlabel) at (-6.75, 2.1) {Space of All SCMs ($\Omega^*$)};
            
            \node [align=center, font=\fontsize{8}{0}\selectfont] (subspacelabel) at (-3.5, -1.3) {Space of SCMs that};
            \node [align=center, font=\fontsize{8}{0}\selectfont] (subspacelabel2) at (-3.5, -1.6) {induce $\cG^*$ and $P^*(\*V)$};
            \node [align=center, font=\fontsize{8}{0}\selectfont] (subspacelabel3) at (-3.5, -1.9) {($\Omega^*(\cG^*, P^*(\*V))$)};
            
            \path [-, line width=0.2mm] (mstarright.center) edge (mstarlabel.west);
            \path [-, line width=0.2mm] (mright.center) edge (mlabel.west);
            \path [-, line width=0.2mm] (mprimeright.-10) edge (mprimelabel.178);
            \path [-, line width=0.2mm] (subspacebottom) edge (subspacelabel.west);
        \end{tikzpicture}
        \caption{
        In the case where $Q$ is non-identifiable, there exist two SCMs in $\Omega^*(\cG^*, P^*(\*V))$ (the space of SCMs matching $\cM^*$ in $P^*(\*V)$ and $\cG^*$) that do not match $\cM^*$ in $Q$.
        }
        \label{fig:ncm-id2}
    \end{center}
\end{wrapfigure}
On the other hand, the situation is qualitatively different when considering non-identifiable effects. To understand how, the first observation comes from the contrapositive of Def.~\ref{def:classic-id}, which says that non-identifiability implies there exists (at least) two SCMs $\cM, \cM'$ within the $\Omega^*(\cG^*, P^*(\*V))$ subspace that share the constraints as in $\cG^*$ and generate the same observational distribution ($P(\*V) = P'(\*V)$), but induce different interventional distributions($P(Y | do(X)) \neq P'(Y | do(X))$). The optimization-based approach will try to exploit precisely this fact, searching for non-identifiability witnesses, possibly different than the true $\cM^*$. Those are illustrated as red dots and hollow circles in Fig.~\ref{fig:ncm-id2}. Still, this is all that is needed to determine whether an effect is not identifiable. 
In practice, each distribution $\cQ, \cQ'$ will only be approximations and it may be hard to detect non-identifiability depending on how close they are from each other. This is indeed what leads to the probabilistic nature of such identifiability statements when using optimization-based methods (as discussed in Appendix \ref{app:experiments}), as opposed to the deterministic ones entailed by the do-calculus and symbolic family. 

By and large, both approaches take the causal assumptions and try to infer something about the unknown, target quantity $P^*(Y | do(X=x))$ that is defined by the unobserved $\cM^*$. In the case of the symbolic approach, the structural assumptions encoded in the causal diagram $\cG$ are used, while an  optimization-based approach will use constraints encoded through the scope of the functions and independence among the exogenous variables $\*U$. Evaluating $P^*(Y \mid do(X = x))$ is possible in identifiable cases, which means that our predictions will match what the true $\cM^*$ would say, while the assumptions would be too weak in others cases, and non-identifiability will take place, which means we are unable to make any statement about $\cM^*$ without  strengthening the assumptions.

After having understood the different takes of both approaches to the problem of identification, we consider again the entire pipeline shown in Fig.~\ref{fig:causal-pipeline}. We note that the two tasks, identification and estimation, are both necessary in any causal analysis, but they are usually studied separately in the symbolic literature. Symbolic methods returns the identifiability status of a query and include the mapping $f$ whenever the effect is identifiable. The target quantity can then be evaluated by standard statistical methods (plug-in) or more refined learning procedures, e.g., multi-propensity score/inverse probability weighting \cite{jung2020estimating}, empirical risk minimization \cite{jung2020werm}, and double machine learning  \cite{jung2021dml}.  

On the other hand, the optimization-based framework proposed here is capable of identifying and estimating the target quantities in an integrated manner. 
As shown through simulations (Sec.~\ref{sec:experiments}), given the sampling nature of the optimization procedure, the performance of the method relies not only on the causal assumptions, but also on accurate optimization, which may require a large amount of samples and computation for training.

In practice, one may consider a hybrid approach where a symbolic algorithm for the identification step is ran first, since it is deterministic and always returns the right answer, and then the estimation step is performed through an NCM.  With this use case in mind, we define Alg.~\ref{alg:nscm-solve-id-symbolic} in Fig.~\ref{fig:causal-statistical-inference} (left) as this hybrid alternative to the more pure Alg.~\ref{alg:nscm-solve-id}. In some sense, this approach combines the best of both worlds -- it relies on an ideal ID algorithm, which is correct with probability one, and the powerful computational properties of neural networks, which may perform well and scale to complex settings, for estimation. We illustrate this using the two figures in Fig.~\ref{fig:causal-statistical-inference} (right).

\begin{figure}[t]
    \begin{minipage}{.54\textwidth}
    \IncMargin{1em}
    \hspace{1em}
    \begin{algorithm}[H]
        \DontPrintSemicolon
        \SetKwData{ncmdata}{$\widehat{M}$}
        \SetKwData{graphdata}{$\cG$}
        \SetKwData{variabledata}{$\*V$}
        \SetKwData{pvdata}{$P(\*v)$}
        \SetKwData{thetastar}{$\bm{\theta}^*$}
        \SetKwFunction{ncmfunc}{NCM}
        \SetKwFunction{symbolicid}{symbolicID}
        \SetKwInOut{Input}{Input}
        \SetKwInOut{Output}{Output}
        
        \Input{ causal query $Q = P(\*y \mid do(\*x))$, $L_1$ data \pvdata, and causal diagram \graphdata}
        \Output{ $P^{\cM^*}(\*y \mid do(\*x))$ if identifiable, \texttt{FAIL} otherwise}
        \BlankLine
        \eIf{$\symbolicid{Q}$}{
            $\ncmdata \gets \ncmfunc{\variabledata, \graphdata}$ \tcp*{from Def.\ \ref{def:g-cons-nscm}}
            $\thetastar \gets \arg \min_{\bm{\theta}} D(P^{\ncmdata(\bm{\theta})}(\*v), \pvdata)$ \tcp*{for some divergence $D$}
            \Return $P^{\ncmdata(\thetastar)}(\*y \mid do(\*x))$\;
        }{
            \Return \texttt{FAIL}
        }
        \caption{Identifying queries with a symbolic ID procedure and then estimating with NCMs.}
        \label{alg:nscm-solve-id-symbolic}
    \end{algorithm}
    \DecMargin{1em}
    \end{minipage}
    \hfill
    \begin{minipage}{.42\textwidth}
        \begin{center}
            \begin{subfigure}{\textwidth}
                \centering
                \includegraphics[width=.6\linewidth]{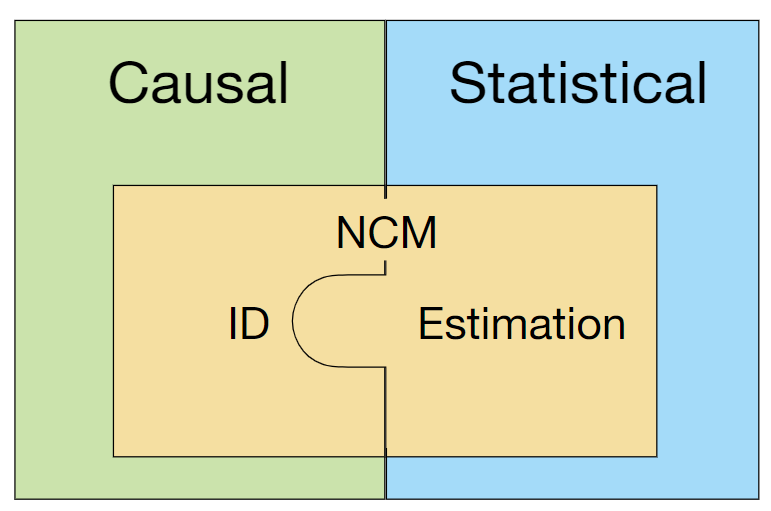}
                \caption{Neural ID + Neural Estimation (Alg.~\ref{alg:nscm-solve-id})}
                \label{fig:neural-id-neural-est}
            \end{subfigure}
            
            \begin{subfigure}{\textwidth}
                \centering
                \includegraphics[width=.6\linewidth]{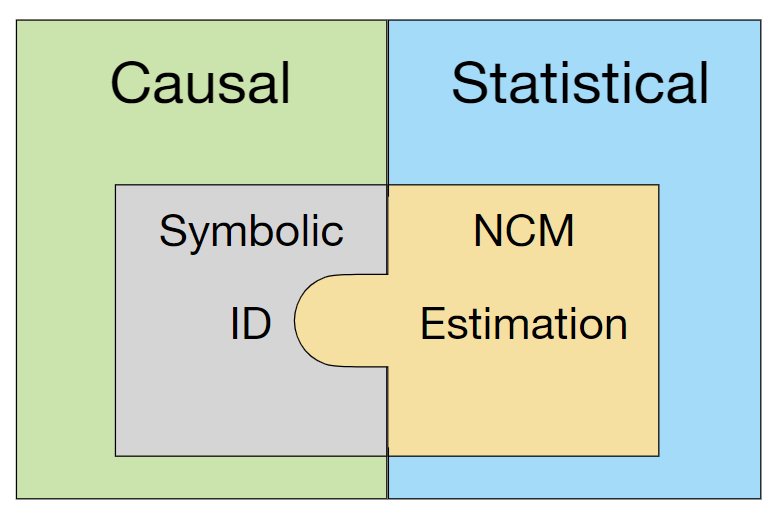}
                \caption{Symbolic ID + Neural Estimation (Alg.~\ref{alg:nscm-solve-id-symbolic})}
                \label{fig:symbolic-id-neural-est}
            \end{subfigure}
        \end{center}
    \end{minipage}
    \caption{(Left panel) Algorithm for solving identification problem with symbolic solvers and estimation with NCMs. (Right) Schematic illustrating differences between Alg.~\ref{alg:nscm-solve-id} (a) and Alg.~\ref{alg:nscm-solve-id-symbolic} (b).}
    \label{fig:causal-statistical-inference}
\end{figure}

After all, the NCM approach provides a cohesive framework that unifies both the causal and the statistical components involved in evaluating interventional distributions. In the case of the hybrid approach, we note that the causal reasoning is shifted to the symbolic identification methods. It is remarkable but somewhat unsurprising that neural methods, which are the state-of-the-art for function approximations, are capable of performing the statistical step required for causal estimation. More interestingly, we demonstrate in this work that neural nets (and proxy SCMs, in general) are also capable of performing causal reasoning, such as in the inferences required to solve the ID problem. Many other works have studied using neural networks as an estimation tool, as listed in the introduction, but to the best of our knowledge, our work is the first one that utilizes neural networks to provide a complete solution within the neural framework to the identification problem, and causal reasoning more broadly.

In addition to showing that neural nets are also capable of performing causal reasoning, there are other implications for using a proxy SCM in place of the true SCM, as opposed to abstracting the true SCM entirely. Firstly, the generative capabilities of proxy SCMs, like the \ncm{}, can be useful if the user desires a source of infinite data, which is a quite common use case found in the literature. Secondly, it provides quick estimation results for multiple queries via the mutilation procedure without the need for retraining. Whenever we apply a symbolic approach we need to derive (and then train) a specialized estimand, which can be time consuming. Thirdly, working in the space of SCMs tends to provide straightforward interpretation of the causal problems using its semantics, which implies that it has the potential to be more easily extensible to other related problems. One may be interested in generalizations of the ID problem such as identifying other types of queries or identifying and fusing from different sources and experimental conditions \citep{bareinboim:etal12,bareinboim:etal15,lee:etal19}, such as in reinforcement learning. In general, extending the line of reasoning of Alg.~\ref{alg:nscm-solve-id} to other identification cases is much simpler than deriving a new symbolic solution under different constraints.

\newpage
\section{Generalizations} \label{app:generalizations}

\subsection{NCMs with other Functions and Noise Distributions  (Proofs)}\label{sec:nn-general}

For the sake of concreteness and simplicity of exposition, the NCM from Def.~\ref{def:nscm} specifically uses $\unif(0, 1)$ noise for the variables in $\*U$ and MLPs as functions. This leads to easy implementation and concrete examples for discussion (as in Appendix \ref{app:examples}).

We note that these are not meant to be limitations for the NCM framework. The NCM object can be extended to use other function or noise types, provided they exhibit certain properties. We will henceforth refer to specific versions of the NCM with their particular implementation of the noise and functions (e.g., Def.~\ref{def:nscm} will be called NCM-FF-Unif). 

Consider a more general version of the NCM below:

\begin{definition}[\genncm{} (General Case)]
    \label{def:gen-ncm}
    A Neural Causal Model (for short, \genncm{}) $\widehat{M}(\bm{\theta})$ over variables $\*V$ with parameters $\bm{\theta} = \{\theta_{V_i} : V_i \in \*V\}$ is an SCM $\langle \widehat{\*U}, \*V, \widehat{\cF}, \widehat{P}(\widehat{\*U}) \rangle$ such that 
    \begin{itemize}
    \setlength\itemsep{-0.5em}
        \item $\widehat{\*U} \subseteq \{\widehat{U}_{\*C} : \*C \subseteq \*V\}$, where each $\widehat{U}$ is associated with some subset of variables $\*C \subseteq \*V$;
        
        \item $\widehat{\cF} = \{\hat{f}_{V_i} : V_i \in \*V\}$, where each $\hat{f}_{V_i}$, parameterized by $\theta_{V_i} \in \bm{\theta}$, maps values of $\Ui{V_i} \cup \Pai{V_i}$ to values of $V_i$ for some $\Pai{V_i} \subseteq \*V$ and $\Ui{V_i} = \{\widehat{U}_{\*C} : \widehat{U}_{\*C} \in \widehat{\*U}, V_i \in \*C\}$;
    \end{itemize}
    Furthermore, a \genncm{} $\widehat{M}(\bm{\theta})$ should have the following properties:
    \\
    (P1) For all $\widehat{U} \in \widehat{\*U}$, $\widehat{U}$ has a well-defined probability density function $P(\widehat{U})$ (i.e. its cumulative density function is absolutely continuous).
    \\  
    (P2) For all functions $\hat{f}_{V_i} \in \widehat{\cF}$ and all functions $g: \bbR \rightarrow \bbS$, where $\bbS \subset \bbR$ is a countable set, there exists parameterization $\theta^*_{V_i}$ such that $\hat{f}_{V_i}(\cdot; \theta^*_{V_i}) = g$ (universal representation).
    \hfill $\blacksquare$
\end{definition}

We can use these properties to prove a more general version of Thm.~\ref{thm:lrepr}. We first state the following result to aid the proof.

\begin{fact}[Probability Integral Transform {\citep[Thm.~2.1.10]{CaseBerg:01}}]
    \label{lem:pdf-to-unif}
    For any random variable $X$ with a well-defined probability density function $P(X)$, there exists a function $f: \cD_{X} \rightarrow [0, 1]$ such that $f(X)$ is a $\unif(0, 1)$ random variable.
    \hfill $\blacksquare$
\end{fact}

The generalized proof follows.

\begin{theorem}[\genncm{} Expressiveness]
    For any SCM $\cM^* = \langle \*U, \*V, \cF, P(\*U) \rangle$, there exists a \genncm{} $\widehat{M}(\bm{\theta}) = \langle \widehat{\*U}, \*V, \widehat{\cF}, \widehat{P}(\widehat{\*U}) \rangle$ s.t. $\widehat{M}$ is $L_3$-consistent w.r.t. $\cM^*$.
    \hfill $\blacksquare$
\end{theorem}

\begin{proof}
    Lemma \ref{lem:scm-to-ctm} guarantees that there exists a canonical SCM $\cM_{\ctm} = \langle \*U_{\ctm}, \*V, \cF_{\ctm}, P(\*U_{\ctm})\rangle$ that is $L_3$-consistent with $\cM^*$. Hence, to construct $\widehat{M}$, it suffices to show how to construct $\cM_{\ctm}$ using the architecture of a \genncm{}.
    
    Following Def.~\ref{def:gen-ncm}, we choose $\widehat{\*U} = \{\widehat{U}_{\*V}\}$. By property P1 and Fact \ref{lem:pdf-to-unif}, there exists a function $g^{U}$ such that $g^{U}(\widehat{U}_{\*V})$ is a $\unif(0, 1)$ random variable.
    
    Using the construction in Lem.~\ref{lem:unif-to-pmf}, there must also exist a function $g^{R}$ such that
    \begin{equation}
        \label{eq:genncm-match-ctm-pu}
        P(g^{R}(g^{U}(\widehat{U_{\*V}}))) = P(\*U_{\ctm})
    \end{equation}
    
    To choose the functions of $\widehat{\cF}$, consider each $\hat{f}_{V_i} \in \widehat{\cF}$. Combining the above results, there must exist
    \begin{equation}
        \label{eq:genncm-match-ctm-f}
        g_{V_i} = f_{V_i}^{\ctm}\left(\pai{V_i}, g^{R}\left(g^{U}(\widehat{U}_{\*V}) \right)\right)
    \end{equation}
    By property P2, we can choose parameterization $\theta^*_{V_i}$ such that $\hat{f}_{V_i}(\cdot; \theta^*_{V_i}) = g_{V_i}$.
    
    By Eqs.~\ref{eq:genncm-match-ctm-pu} and \ref{eq:genncm-match-ctm-f}, the \genncm{} $\widehat{M}$ is constructed to match $\cM_{\ctm}$ on all outputs. Hence, for any counterfactual query $\bm{\upvarphi}$, we have
    $$\cM_{\ctm} \models \bm{\upvarphi} \Leftrightarrow \widehat{M} \models \bm{\upvarphi}$$
    and therefore
    $$\cM^* \models \bm{\upvarphi} \Leftrightarrow \widehat{M} \models \bm{\upvarphi}.$$
\end{proof}

The general $\cG$-constrained version of the \genncm{} follows the same procedure as Def.~\ref{def:g-cons-nscm}.

\begin{definition}[$\cG$-Constrained \genncm{} (General Case)]
    \label{def:g-cons-gen-ncm}
    Let $\cG$ be the causal diagram induced by SCM $\cM^*$. Construct \genncm{} $\widehat{M}$ as follows. \textbf{(1)} Choose $\widehat{\*U}$ s.t. $\widehat{U}_{\*C} \in \widehat{\*U}$ if and only if $\*C$ is a $C^2$-component in $\cG$. Each $\widehat{U}_{\*C}$ has its own distribution $P(\widehat{U}_{\*C})$ independent of other variables in $\widehat{\*U}$, but it is shared as input to the functions of all variables in $\*C$. \textbf{(2)} For each variable $V_i \in \*V$, choose $\Pai{V_i} \subseteq \*V$ s.t. for every $V_j \in \*V$, $V_j \in \Pai{V_i}$ if and only if there is a directed edge from $V_j$ to $V_i$ in $\cG$.
    Any \genncm{} in this family is said to be $\cG$-constrained.
     $\blacksquare$
\end{definition}

We prove a general version of Thm.~\ref{thm:nscm-g-uat} similarly.

\begin{theorem}[\genncm{} $L_2$-$\cG$ Representation]
    For any SCM $\cM^*$ that induces causal diagram $\cG$, there exists a $\cG$-constrained \genncm{} $\widehat{M}(\bm{\theta}) = \langle \widehat{\*U}, \*V, \widehat{\cF}, \widehat{P}(\widehat{\*U}) \rangle$ that is $L_2$-consistent w.r.t. $\cM^*$.
    \hfill $\blacksquare$
\end{theorem}

\begin{proof}
    Fact \ref{fact:scm-to-discrete} states that there exists a discrete SCM $\cM' = \langle \*U', \*V, \cF', P(\*U')\rangle$ with finite domains of $\*U'$ that is $\cG$-consistent and is $L_2$-consistent with $\cM^*$. Hence, to construct $\widehat{M}$, we can simply show how to construct $\cM'$ using the architecture of a \genncm{}.
    
    Following Def.~\ref{def:g-cons-gen-ncm}, we choose $\widehat{\*U} = \{\widehat{U}_{\*C} : \*C \in \bbC(\cG)\}$, where $\bbC(\cG)$ is the set of all $C^2$-components of $\cG$. Denote $\widehat{\*U}_{V} = \{\widehat{U}_{\*C}: \*C \in \bbC(\cG) \text{ s.t. } V \in \*C\}$.
    
    By property P1 and Fact \ref{lem:pdf-to-unif}, there exists a function $g_{V}^{U}$ such that $g_{V}^{U}(\widehat{\*U}_{V})$ is a $\unif(0, 1)$ random variable for each $V \in \*V$.
    
    Using the construction in Lem.~\ref{lem:unif-to-pmf}, there must also exist a function $g_V^{R}$ such that
    \begin{equation}
        \label{eq:g-genncm-match-ctm-pu}
        P(g_V^{R}(g_{V}^{U}(\widehat{\*U}_{V}))) = P(\*U'_V)
    \end{equation}
    for each $V \in \*V$.
    
    To choose the functions of $\widehat{\cF}$, consider each $\hat{f}_{V_i} \in \widehat{\cF}$. Combining the above results, there must exist
    \begin{equation}
        \label{eq:g-genncm-match-ctm-f}
        g_{V_i} = f_{V_i}' \left(\pai{V_i}, g_{V_i}^{R}\left(g_{V_i}^{U}(\widehat{\*U}_{V}) \right)\right)
    \end{equation}
    By property P2, we can choose parameterization $\theta^*_{V_i}$ such that $\hat{f}_{V_i}(\cdot; \theta^*_{V_i}) = g_{V_i}$.
    
    By Eqs.~\ref{eq:g-genncm-match-ctm-pu} and \ref{eq:g-genncm-match-ctm-f}, the \genncm{} $\widehat{M}$ is constructed to match $\cM'$ on all outputs. Hence, $\widehat{M}$ must be $L_2$-consistent with $\cM^*$.
\end{proof}

The remaining results, including Corol.~\ref{cor:ncht}, Thm.~\ref{thm:nscm-g-cons}, Thm.~\ref{thm:nscm-id-equivalence}, Corol.~\ref{thm:dual-graph-id}, Corol.~\ref{thm:markovid}, and Corol.~\ref{thm:nscm-id-correctness} can be proven for \genncm{}s with minimal changes to the proofs for \ncm{}s.

\subsection{Pearl's Causal Hierarchy and Other Classes of Models} \label{sec:pch-other-models}

As developed in the paper, we note that the expressiveness power of \ncms{} is one desirable features of this class as models, which makes them comparable to the SCM class when considering the PCH's capabilities. Broadly speaking, neural networks are well understood to be maximally expressive due to the Universal Approximation Theorem \citep{Cybenko89}. We note that Theorem \ref{thm:lrepr} follows the same spirit when considering causal inferences, i.e., \ncms{} is a data structure built from neural networks that are maximally expressive on
 \emph{every} distribution in \emph{every} layer of the PCH generated by the underlying SCM $\cM^*$, and which has the potential of acting as a proxy under some conditions. This capability is not shared across all neural architectures, and to understand how this discussion extends to these other classes, we define expressivity more generally (Fig.~\ref{fig:expressiveness}): 

\begin{definition}[$L_i$-Expressiveness]
    Let $\Omega^*$ be the set of SCMs and $\Omega$ be a model class of interest. We say that $\Omega$ is $L_i$ expressive if for all $\cM^* \in \Omega^*$, there exists $\cM \in \Omega$ s.t. $L_i(\cM^*) = L_i(\cM)$.
    \hfill $\blacksquare$
\end{definition}

It should be emphasized that expressiveness on layers higher than Layer 1 is a highly nontrivial property (e.g., Example \ref{ex:expressiveness-basic} in Appendix \ref{app:examples-expr} shows the complications involved in constructing an NCM to replicate another SCM on all 3 layers). Even a model class that is $L_2$-expressive can represent far more settings than model classes that are only $L_1$-expressive, since there are many more ways in which SCMs can differ on $L_2$ than on $L_1$. Theorem \ref{thm:lrepr} states that \ncm{}s are $L_3$-expressive, which leads to the question: what model classes are not $L_3$-expressive?

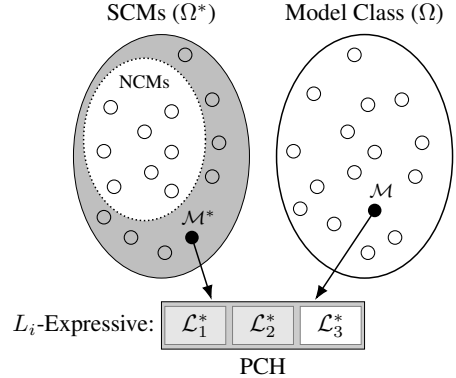
\begin{wrapfigure}{r}{0.45\textwidth}
    \vspace{-0.15in}
    \begin{center}
        \begin{tikzpicture}[-, scale=0.9, every node/.append style={transform shape}]
        
            \filldraw [fill=gray!50, line width=0.01mm] (-6.75, 0) ellipse (1.3 and 1.8);
            \fill [fill=gray!0] (-7, 0.25) ellipse (0.9 and 1.2);
            \filldraw [fill=white, line width=0.2mm, densely dotted] (-7, 0.25) ellipse (0.9 and 1.2);
            
            \node [align=center, font=\fontsize{10}{0}\selectfont] (omegastarlabel) at (-6.75, 2.1) {SCMs ($\Omega^*$)};
            \node [align=center, font=\fontsize{8}{0}\selectfont, fill=white, inner sep=0.7mm] (omegastarlabel) at (-7, 1.1) {NCMs};
            \node [align=center, font=\fontsize{10}{0}\selectfont] (omegastarlabel) at (-3.75, 2.1) {Model Class ($\Omega$)};
            
            \draw (-6, 0.8375) circle (0.1);
            
            \draw (-6.4, 1.5) circle (0.1);
            \filldraw [fill=white] (-7.5, 0.725) circle (0.1);
            \filldraw [fill=white] (-7.6, 0.075) circle (0.1);
            \filldraw [fill=white] (-7.45, -0.45) circle (0.1);
            \draw (-7.6, -0.9) circle (0.1);
            \draw (-7.2, -1.2) circle (0.1);
            \fill [fill=black] (-6.3, -1.2) circle (0.1);
            \node [inner sep=0] (mstarbottom) at (-6.3,-1.2) {};
            \node [align=center, font=\fontsize{8}{0}\selectfont] (mstarlabel) at (-6.2, -0.9) {$\cM^*$};
            \draw (-6.7, -1.425) circle (0.1);
            
            \filldraw [fill=white] (-6.6, 0.6625) circle (0.1);
            \filldraw [fill=white] (-7.0, 0.3625) circle (0.1);
            \draw (-5.9, 0.2) circle (0.1);
            \filldraw [fill=white] (-6.6, 0.0625) circle (0.1);
            \filldraw [fill=white] (-6.9, -0.2375) circle (0.1);
            \draw (-6, -0.3) circle (0.1);
            \filldraw [fill=white] (-6.6, -0.5125) circle (0.1);
            
            \draw [line width=0.2mm] (-3.75, 0) ellipse (1.3 and 1.8);
            
            \draw(-3.2, 0.8375) circle (0.1);
            
            \draw (-3.9, 1.5) circle (0.1);
            \draw (-4.5, 0.925) circle (0.1);
            \draw (-4.8, 0.075) circle (0.1);
            \draw (-4.45, -0.45) circle (0.1);
            \draw (-4.7, -0.775) circle (0.1);
            \draw (-4.2, -1) circle (0.1);
            \draw (-3.3, -1.2) circle (0.1);
            \draw (-3.7, -1.425) circle (0.1);
            
            \draw (-3.7, 0.7625) circle (0.1);
            \draw (-4.0, 0.3625) circle (0.1);
            \draw (-2.9, 0.2) circle (0.1);
            \draw (-3.6, 0.0625) circle (0.1);
            \draw (-3.9, -0.2375) circle (0.1);
            \draw (-3, -0.3) circle (0.1);
            \fill[fill=black] (-3.6, -0.8) circle (0.1);
            \node [inner sep=0] (mbottom) at (-3.6,-0.8) {};
            \node [align=center, font=\fontsize{8}{0}\selectfont] (mlabel) at (-3.5, -0.5) {$\cM$};

            \filldraw [fill=gray!40,line width=0.01mm] (-6.75, -2.15) rectangle (-3.75, -2.85);
            
            \node (liexpressive) at (-7.9, -2.5) {$L_i$-Expressive:};
            \node (pch) at (-5.25, -3.1) {PCH};
        
           	\node [fill=gray!20, draw=gray, line width=0.08mm, align=center, minimum width=0.9cm] (L1Pv) at (-6.25,-2.5) {$\Ll_1^*$};
           	\node [fill=gray!20, draw=gray, line width=0.08mm, align=center, minimum width=0.9cm] (L2Pv) at (-5.25,-2.5) {$\Ll_2^*$};
           	\node [fill=gray!0, draw=gray, line width=0.08mm, align=center, minimum width=0.9cm] (L3Pv) at (-4.25,-2.5) {$\Ll_3^*$};
           	
           	\node [inner sep=0] (layerstop1) at (-6,-2.15) {};
           	\node [inner sep=0] (layerstop2) at (-4.5,-2.15) {};
           	
           	\path [-Latex, line width=0.2mm] (mstarbottom) edge (layerstop1);
           	\path [-Latex, line width=0.2mm] (mbottom) edge (layerstop2);
        \end{tikzpicture}
        \caption{The SCM class $\Omega^*$ is shown in the l.h.s. and another class (non-SCM) $\Omega$ class is shown in the r.h.s.. If we consider $M^* \in \Omega^*$ that generates a specific PCH, $\Omega$ is $L_i$ representative if there is a model $M \in \Omega^*$ that exhibit the same behavior under $L_i$. 
        }
        \label{fig:expressiveness}
    \end{center}
\end{wrapfigure}

Many neural model classes are $L_1$-expressive, which can be seen as the property of modeling an unique probability distribution over a set of observed variables. Additionally, this model should be generative, i.e., one should be able to draw samples from the model's fitted distribution. Model classes that have this property include variational models with normalizing flows \citep{pmlr-v37-rezende15}, which are generative models and are proven to be able to fit any data distribution. It is also believed that popular generative models like generative adversarial networks \citep{NIPS2014_5ca3e9b1} have this property as well.

Interesting enough, being generative does not imply being causal or counterfactual. For instance, these model classes are not well defined for valuating any causal distributions. A model from one of these classes that is fitted on $P(\*V)$ can only provide samples from $P(\*V)$. On the other hand, distributions like $P(\*V_{\*x})$, the distribution of $\*V$ under the intervention $do(\*X = \*x)$ for some $\*X \subseteq \*V$, are entirely different distributions from Layer 2 of the PCH. A model fitted on $P(\*V)$ cannot output samples from $P(\*V_{\*x})$, and it certainly cannot output samples from \emph{every} Layer 2 or Layer 3 distributions. These models, therefore, are not even $L_2$-expressive, let alone $L_3$-expressive.

It is somewhat natural, and perhaps obvious, that model classes that are not defined on higher layers cannot be expressive on higher layers. Still, we can also consider this property on model classes proposed by other deep learning works which include a causal component. Works such as \citep{yoon2018ganite} and \citep{shi2019adapting} have had great success estimating higher layer quantities like average treatment effects under backdoor/conditional ignorability conditions \cite[Sec.~3.3.1]{pearl:2k}. Perhaps these models can always succeed at modeling Layer 2 expressions as long as the ground truth can be modeled by an SCM that also meets those same backdoor conditions. However, since there are SCMs that do not fit the backdoor setting, we cannot say that these models are $L_2$ or $L_3$-expressive. We illustrate this point by providing an example (Example \ref{ex:id}) in Appendix \ref{app:examples-id}, where the \ncm{} estimates a causal quantity in a setting that does not meet the backdoor condition.

In some way, these works and most of the literature are concerned with estimating causal effects that are identifiable from observational data $P(V)$, usually in the form of the backdoor formula, which means that the ``causal reasoning'' already happened outside the network itself. In other words, a causal distribution $Q = P(\*Y | do(\*X=\*x))$ being identifiable means that there exists a function $f$ such that $Q = f(P(V))$. The causal ``inference'' means determining from causal assumptions (for example the causal diagram $G$) whether $Q$ is identifiable, or whether $f$ exists. Whenever this is the case, and $f$ has a closed-expression, the inference is done, and the task that remains is to estimate $f$ efficiently from finite samples obtained from $P(V)$. 

Even though many existing works do not have the $L_2$ or $L_3$-expressiveness properties, this is not to say that the \ncm{} is the only model class that is expressive enough to represent all the three layers of the PCH. Consider the following example of a model class artificially constructed to be $L_2$-expressive:

\begin{definition}[$L_2$-Expression Wrapper]
    \label{def:l2-expression-wrapper}
    Let $M^{L_1}$ be an $L_1$-expressive model class such as a normalizing flow model. Denote $M^{L_1}_{P}$ as a model from this class fitted on distribution $P$. For a collection of $L_2$ distributions $\*P^*$, define $L_2$-expression wrapper $M^{L_2} := \{M^{L_1}_P : P \in \*P^*\}$.
    \hfill $\blacksquare$
\end{definition}
It is fairly straightforward to show that the $L_2$-expression wrapper is $L_2$-expressive. For any SCM $\cM^*$, we can simply construct $L_2$-expression wrapper $M^{L_2}$ with $\*P^* = L_2(\cM^*)$. Then, for any $L_2$-query $Q$, $M^{L_2}$ answers $Q$ by providing the distribution of $M^{L_1}_Q$. In other words, we could simply train a generative model for \textit{every} $L_2$ distribution of an SCM, and collectively, these models induce all $L_2$ distributions of the SCM.

Although expressive, this model is highly impractical for most practical uses. First, there are typically too many distributions in Layer 2 to effectively learn all of them. A new distribution would be created for every possible intervention of every subset of variables. In fact, this amount grows exponentially with the number of variables, and in continuous cases, it is infinite.

Second, unlike the NCM data structure (Def.~\ref{def:g-cons-nscm}), there is no clear way to incorporate the constraints of a CBN in the expression wrapper. Without these constraints, it is obvious that the expression wrapper suffers from the same limitations from the CHT as NCMs (Corol.~\ref{cor:ncht}). Suppose we are given $P(\*V)$ from true SCM $\cM^*$, but we are interested in the distribution $P(\*V_{\*x})$. We could guarantee an expression wrapper $M^{L_2}$ such that $M^{L_1}_{P(\*V)} = P^{\cM^*}(\*V)$, but $M^{L_1}_{P(\*V_{\*x})}$ could be any distribution over $\*V$ that is consistent with $\*x$, and we would not be able to guarantee that it matches $P^{\cM^*}(\*V_{\*x})$.

Attempting to incorporate the constraints of a CBN in an expression wrapper would scale poorly and be difficult to realize in practice. Consider the following example.

\begin{example}
    \label{ex:G-expression-wrapper}
    \begin{figure}[h]
        \centering
        \begin{tikzpicture}[xscale=1.5, yscale=2]
    		\node[draw, circle] (D) at (-1, 0) {$D$};
    		\node[draw, circle] (S) at (0, 0) {$S$};
    		\node[draw, circle] (B) at (1, 0) {$B$};
    
    		\path [-{Latex[length=2mm]}] (D) edge (S);
    		\path [-{Latex[length=2mm]}] (S) edge (B);
    		\path [dashed, {Latex[length=2mm]}-{Latex[length=2mm]}, bend left] (D) edge (B);
    	\end{tikzpicture}
    	\caption{Causal diagram $\cG$ of $\cM^*$ from Example \ref{ex:id}.}
    	\label{fig:wrapper-cg}
    \end{figure}
    Consider the same setting from Example \ref{ex:id}, where we would like to evaluate a query like $Q = P(B = 1 \mid do(D = 1))$ given $P(\*V)$ and the graph $\cG$ (displayed again in Fig.~\ref{fig:wrapper-cg} for convenience).
    If we were using the expression wrapper, we would need an estimator for each of the $L_1$ and $L_2$ distributions such $P(\*V), P(\*V_{D = 1}), P(\*V_{S = 1})$, and so on.
    
    However, unlike the NCM, the expression wrapper does not automatically incorporate the constraints implied by $\cG$. Here is a list of some  constraints used in the do-calculus derivation in Example \ref{ex:id}:
    \begin{enumerate}
        \item $P(B \mid do(D), S) = P(B \mid do(D, S))$
        \item $P(B \mid do(D, S)) = P(B \mid do(S))$
        \item $P(S \mid do(D)) = P(S \mid D)$
        \item $P(D \mid do(S)) = P(D)$
        \item $P(B \mid do(S), D) = P(B \mid S, D)$
    \end{enumerate}
    
    This list is not exhaustive, and there are more constraints induced by $\cG$. This list is also compressed, since intervening on different values for the same variables result in different distributions. Note that even with just 3 binary variables, the number of constraints is quite long when enumerated.
    
    Maintaining all of these constraints while fitting all of the desired distributions is the key difficulty of using the $L_2$-expression wrapper. While fitting $P(\*V)$, the expression wrapper would have to also generate some other distributions from layer 2 like $P(\*V_{D = 1})$, keeping constraints in mind. One cannot make such a choice arbitrarily. For instance, setting $P(\*V_{D = 1})$ to $P(\*V)$, but with all values of $D$ set to 1, would still violate constraint 3 from the above list. This would be infeasible for any tractable optimization procedure.
    \hfill $\blacksquare$
\end{example}

This illustrates the benefits of the properties of the NCM. Not only is the NCM $L_3$-expressive, a non-trivial property to achieve, but it also is easily able to incorporate structural constraints. 
\newpage
\section{Frequently Asked Questions}

\begin{enumerate}[label=Q\arabic*.]
    \item Why are the sets of noise variables for each function not necessarily independent in the SCM and NCM?
    \vspace{+0.03in}
    \\ \textbf{Answer.} 
Each structural causal model includes a set of observed (endogenous) and unobserved (exogenous) variables, $\*V$ and $\*U$, respectively. The set $\*U$  accounts for the unmodeled sources of variation that generate randomness in the system. On the one hand, an exogenous variable $U \in \*U$ could affect all endogenous variables in the system. On the other hand, each $U_i \in \*U$ could affect only one observable $V_i \in \*V$ and be independent to all other variables in $\*U \setminus \{U_i\}$.  These two represent opposite sides of a spectrum, where the former means zero assumptions on the presence of unobserved confounding, while the latter represents the strongest assumption known as Markovianity. We operate in between, allowing for all possible cases, which is represented in the causal diagram through bidirected edges. 

In the context of identification, Markovianity is usually known as the lack of unobserved confounding. As shown in Corol.~\ref{thm:markovid}, the identification problem becomes trivial and causation is always derivable from association in this case, i.e., \emph{all} layer 2 queries are computable from layer 1 data.
In practice, identification becomes somewhat more involved in non-Markovian settings. Example \ref{ex:non-id} in Appendix \ref{app:examples} shows a case where a causal query is not identifiable even in a 2 variable case. More discussion and examples can be found in \citep[Sec.~1.4]{bareinboim:etal20}.
    \\
    
    \item How come the true SCM cannot be learned from observational data? If the functions are universal approximators, this seems to be easily accomplished. 
    \vspace{+0.05in}
    \\ \textbf{Answer.} 
    This is not possible in general, as highlighted by Corol.~\ref{cor:ncht}. The issue is that the learned model, $\widehat{M}$, could generate the distribution of observational data perfectly yet still differ in interventional distributions from the true SCM, $\cM^*$. Searching for (or randomly choosing) some SCM that matches the observational distribution will almost surely fail to match the interventional distributions. Since the true SCM, $\cM^*$, is unobserved, one would not be able to determine whether the inferences made on the learned model $\widehat{M}$ preserve the validity of the causal claims. See Example \ref{ex:ncht} in Appendix \ref{app:examples} for a more detailed discussion.
    \\
    
    \item Why are feedforward neural networks and uniform distributions used  in Def.~\ref{def:nscm}? Would these results not hold for other function approximators and noise distributions?
    \vspace{+0.05in}
    \\ \textbf{Answer.} 
    We implemented in the body of the paper NCMs with feedforward neural networks and uniform distributions for the sake of clarity of exposition and presentation. The results of the NCM are not limited to this particular architecture choice, and other options may offer different benefits in practice. A general definition of the NCM is introduced in Def.~\ref{def:gen-ncm}, and the corresponding results  in this generalized setting are provided in Appendix \ref{sec:nn-general}.
    \\
    
    \item What is a causal diagram? What is a CBN? Is there any difference between these objects?
    \vspace{+0.05in}
    \\ \textbf{Answer. } 
    A causal diagram $\cG$ is a  non-parametric  representation of an SCM in the form of a graph, which follows Def.~\ref{def:scm-cg}, such that nodes represent variables, a directed edge from $X$ to $Y$ indicates that $X$ is an input to $f_Y$, and a bidirected edge from $X$ to $Y$ indicates that there is unobserved confounding between $X$ and $Y$. The causal diagram $\cG$ is strictly weaker than the generating SCM $\cM^*$. Talking about the diagram itself is usually ill-defined, since it is relevant in the context of the constraints it imposes over the space of distributions generated by $\cM^*$ (both observed and unobserved). 
    
    A Causal Bayesian Network (CBN, Def.~\ref{def:cbn}) is a pair consisting of a causal diagram $\cG$ and a collection of layer 2 (interventional) distributions $\*P^*$, where the distributions follow a set of constraints represented in the diagram (i.e., they are ``compatible''). Remarkably,  the causal diagram induced by any SCM, along with its induced layer 2 distributions, together form a valid CBN (Fact \ref{fact:scm-cbn-l2}). Causal diagrams and the constraints encoded in CBNs are often the fuel to causal inference methods (e.g., do-calculus) and used to bridge the gap between layer 1 and layer 2 quantities. We refer readers to Appendix \ref{app:examples-cg} for a further discussion on these inductive biases and their motivation, and \citep[Sec.~1.3]{bareinboim:etal20} for further elaborations. 
    \\
    
    \item Based on the answer to Q4, it is not entirely clear what an NCM really represents. Is an NCM a causal diagram? Or is it an SCM? Or should I think about it in a different way? 
    \vspace{+0.05in}
    \\ \textbf{Answer} 
    An NCM is not a causal diagram, and although an NCM is technically an SCM, by definition, it may be helpful to think of it as a separate data structure. As discussed above, SCMs contain strictly more information than causal diagrams. Still, NCMs without $\cG$-consistency (Def.~\ref{def:g-consistency}) do not even ascertain the constraints represented in $\cG$ about layer 2 distributions. So, even though an NCM seems to appear like an SCM, it is an ``empty'' data structure on its own. Whenever we impose $\cG$-consistency, NCMs may have the capability of generating interventional distributions, contingent on $L_1$-consistency and the identifiability status of the query. 
    
 More broadly, we think it is useful to think about an NCM as a candidate for a proxy model for the true SCM. 
    Unlike the SCM, the NCM's concrete definitions for functions and noise allow the NCM to be optimized (via neural optimization procedures) for learning in practical settings.
    One additional difference between the NCM and SCM is their purpose in causal inference. The SCM data structure is typically used to represent reality, so in some sense, the SCM has all information of interest from all layers of the PCH. Of course, this reality is unobserved, so the SCM is typically unavailable for use. The NCM is then trained as a proxy for the true SCM, so it has all of the functionality of an SCM while being readily available. However, unlike the SCM, the ``quality'' of the information in the NCM is dependent on the amount of input information. If an NCM is only trained on data from layer 1, the Neural CHT (Corol.~\ref{cor:ncht}) says that distributions induced by the NCM from higher layers will, in general, be incorrect. For that reason, while the NCM has the same interface as the SCM for the PCH, the NCM's higher layers may be considered ``underdetermined'', in the sense that they do not contain any meaningful information about reality.
    \\
    
    \item Why do you assume the causal diagram exists instead of learning it? What if the causal diagram is not available? 
    \vspace{+0.05in}
    \\ \textbf{Answer. } 
    As implied by the N-CHT (Corol.~\ref{cor:ncht}), cross-layer inferences cannot be accomplished in the general case without additional information. 
    In other words, the causal diagram assumption is out of necessity. In general, one cannot hope to learn the whole causal diagram from observational data either.
    However, it is certainly not the case that the causal diagram will always be available in real world settings. In such cases, other assumptions may be necessary which may help structural learning, or perhaps the causal diagram can be partially specified to a degree that allows answering the query of interest (see also Footnotes \ref{footnote:structural-learning} and \ref{footnote:other-ncms}).
    \\
    
    \item Are NCMs assumed to be acyclic?
    \vspace{+0.05in}
    \\ \textbf{Answer.}
    We assume for this work that the true SCM $\cM^*$ is \textit{recursive}, meaning that the variables can be ordered in a way such that if $V_i < V_j$, then $V_j$ is guaranteed to not be an input to $f_{V_i}$. This implies that the causal diagram induced by $\cM^*$ is acyclic, so $\cG$-constrained NCMs (from Def.~\ref{def:g-cons-nscm}) must also be recursive. That said, not all works make the recursive assumption, so the general definition of NCMs does not need to be constrained a priori. 
    \\
    
    \item  Why is $\cM^*$ assumed to be an SCM instead of an NCM in the definition for neural identifiability (Def.~\ref{def:nscm-id})?
    \vspace{+0.05in}
    \\ \textbf{Answer.}
    $\cM^*$ indicates the ground truth model of reality, and we assume that reality is modeled by an SCM, without any constraints in the form of the functions or probability over the exogenous variables. The NCM is a data structure that we constructed to solve causal inference problems, but the inferences of interest are those from the true model. For this reason, several results were needed to connect the properties of NCMs to those from SCMs, including Thm.~\ref{thm:nscm-g-cons} (showing that NCMs can encode the same constraints as SCMs' causal diagrams) and Thm.~\ref{thm:nscm-g-uat} (showing that NCMs can represent any $\cG$-consistent SCM on layer 2). These results are required to show the validity (or meaning) of the inferences made from the NCM, which otherwise would have no connection to the true model, $\cM^*$.
    
    Naturally, checking for identifiability within the space of NCMs defeats the purpose of the work since the model of interest is not (necessarily) an NCM. Conveniently, however, the expressiveness of the NCM (as shown in Thm.~\ref{thm:nscm-g-uat}) shows that if two SCMs agree on $P(\*V)$ and $\cG$ but do not agree on $P(\*y \mid do(\*x))$, then there must also be two NCMs that behave similarly. This means that attempting to check for identifiability in the space of NCMs will produce the correct result, which is why Alg.~\ref{alg:nscm-solve-id} can be used. Still, interestingly, this may no longer be the case within a less expressive model class (see Example \ref{ex:nonexpr} in Appendix \ref{app:examples-id}).
    \\
    
    \item What is the purpose of Thm.~\ref{thm:nscm-g-uat} if we already have Thm.~\ref{thm:lrepr}?
    \vspace{+0.05in}
    \\ \textbf{Answer.}
    While Thm.~\ref{thm:lrepr} shows that NCMs (from Def.~\ref{def:nscm}) can express any SCM on all three layers of the PCH, it does not make any claims about $\cG$-constrained NCMs (Def.~\ref{def:g-cons-nscm}). In fact, a $\cG$-constrained NCM is obviously incapable of expressing any SCM that is not $\cG$-consistent. Thm.~\ref{thm:nscm-g-uat} emphasizes that even when encoding the structural constraints in $\cG$ (CBN-like), NCMs are still capable of expressing any SCM that is also compatible with $\cG$. If this property did not hold, it would not always be possible to use the NCM as a proxy for the SCM, since there may be some SCMs that are $\cG$-consistent, but there does not exist an NCM that matches the corresponding distributions.
    
    One important subtlety of Thm.~\ref{thm:nscm-g-uat} is that the expressiveness is described after applying the constraints of $\cG$. It may be relatively simple to use a universal approximator to construct a model class that can express any SCM on layer 2, such as the $L_2$-expression wrapper (Def.~\ref{def:l2-expression-wrapper} from Appendix \ref{sec:pch-other-models}). However, the $L_2$-expression wrapper does not naturally enforce the constraints of $\cG$. Incorporating such constraints is nontrivial, and even if the constraints were applied, the model class may no longer retain its expressiveness. Thm.~\ref{thm:nscm-g-uat} shows that the NCM is the maximally expressive model class that still respects the constraints of $\cG$.
    \\
    
    \item The objective function in Eq.~\ref{eq:ncm-total-loss} is not entirely clear. What role does $\lambda$ play?
    \vspace{+0.05in}
    \\ \textbf{Answer.} 
    The purpose of Alg.~\ref{alg:ncm-learn-pv} is to solve the optimization problem in lines 2 and 3 of Alg.~\ref{alg:nscm-solve-id}. The goal can be described in two parts: maximize/minimize a query of interest, like $P(\*y \mid do(\*x))$, while simultaneously learning the observational distribution $P(\*V)$. This is a nontrivial optimization problem, and there are many different ways to approach it.
    The proposed method penalizes both of these quantities in the objective. The first term of Eq.~\ref{eq:ncm-total-loss} attempts to maximize log-likelihood, while the second term attempts to maximize/minimize $P(\*y \mid do(\*x))$. Certainly, we do not want the optimization of the second term to affect the optimization of the first term, since we need accurate learning of $P(\*V)$ to make any deductions on the identifiability status of $P(\*y \mid do(\*x))$ and subsequent estimation of it.
    
    Further, the hyperparameter $\lambda$ is used to balance the amount of attention placed on maximizing/minimizing $P(\*y \mid do(\*x))$. If $\lambda$ is too small, then the term would be ignored, and the optimization of $P(\*y \mid do(\*x))$ could fail even in non-ID cases. On the other hand, if $\lambda$ is too large, the optimization procedure might attempt to maximize/minimize $P(\*y \mid do(\*x))$ at the expense of learning $P(\*V)$ properly, which may result in incorrect inferences in ID cases. Our solution starts $\lambda$ at a large value to allow the optimization to quickly find parameters that maximize/minimize $P(\*y \mid do(\*x))$,  and then it lowers $\lambda$ as training progresses so that $P(\*V)$ is learned accurately by the end of training.
    \\
    
    \item Why would one want, or prefer, to solve identifiability using NCMs? Would it not be easier to use existing results (e.g., do-calculus) to determine the identifiability of a query, then use the NCM for estimation?
    \vspace{+0.05in}
    \\ \textbf{Answer.} 
     We provide a more detailed discussion on this topic in Appendix \ref{app:examples-symbolic-neural}. 
    In summary, when solving causal problems like identification or estimation, we introduced a new framework using NCMs where one can construct a proxy model to mimic the true SCM via neural optimization. Meanwhile, existing works focus on solving the problems at a higher-level of abstraction with the causal diagram, ignoring the true SCM. We call these works ``symbolic'' since they directly use the constraints of the causal diagram to algorithmically derive a solution. We are not deciding for the causal analyst which method to use when comparing symbolic and  optimization-based methods. For instance, we developed Alg.~\ref{alg:nscm-solve-id-symbolic} as an alternative to Alg.~\ref{alg:nscm-solve-id} for researchers who would prefer to use a symbolic approach for identification while using an NCM for estimation.

    We note that the identification and estimation tasks typically appear together in practice, and it is  remarkable that the proxy-model framework with NCMs is capable of solving the entire causal pipeline. In some sense, the identification task provides the causal reasoning required to perform estimation. Once a query is determined to be identifiable, the estimation task is only a matter of fitting an expression. It may be somewhat unsurprising that the expressiveness of neural networks along with its strong optimization results allow it to fit any expression, but this work shows, for the first time, that neural nets are also capable of performing the causal reasoning required to determine identifiability. We note that Alg.~\ref{alg:nscm-solve-id} describes a procedure for solving the ID problem that is unique compared to symbolic approaches. Our goal with Alg.~\ref{alg:nscm-solve-id} is to provide insights to alternative approaches to causal problems under the proxy model framework compatible with optimization methods.
    \\
    
    \item Why would one want to learn an entire SCM to solve the estimation problem for a specific query? Shouldn't the efforts be focused on estimating just the query of interest?
    \vspace{+0.03in}
 \\ \textbf{Answer. } 
    Indeed, existing works typically focus on a single interventional distribution and also abstract out the SCM, as shown in Fig.~\ref{fig:causal-pipeline}.  In particular,  solutions to the identification problem typically derive an expression for the query of interest using the causal diagram, and solutions to the estimation problem directly evaluate the derived expression.
    
    One of the novelties of this work is the introduction of a method of solving these tasks using a proxy model of the SCM. Compared to existing methods, there are a number of reasons one may want to have a proxy model. For instance, one may want to have generative capabilities of both layer 1 and identifiable layer 2 distributions rather than simply obtaining a probability value. We refer readers to Appendix \ref{app:examples-symbolic-neural} for a more detailed discussion.
    \\
    
    \item It appears that the theory and results only hold if learning is perfect. What happens if there is error in training?
    \vspace{+0.05in}
    \\ \textbf{Answer.}  
    This work does not make any formal claims regarding cases with error in training. Still, one of our goals with the experiments of Sec.~\ref{sec:experiments} is to show empirically that results tend to be fairly accurate if training error is minimized to an acceptable degree (i.e., $L_1$-consistency nearly holds). A more refined understanding and detailed theoretical analysis of robustness to training error is an interesting direction for future work.
    \\
    
\end{enumerate}

\end{document}